\documentclass[12pt]{article}
\usepackage{setspace}

\usepackage{amsmath}
\usepackage{times}
\usepackage{graphicx}
\usepackage{epstopdf} 

\usepackage[dvipsnames,table,xcdraw]{xcolor}
\usepackage{multirow}
\usepackage{hyperref}
\hypersetup{
colorlinks=true,    
linkcolor=blue,     
citecolor=DarkGreen,    
breaklinks=true,    
pdfborder={0 0 0},  
}
\usepackage{xcolor}
\definecolor{customgreen}{rgb}{0.0, 0.5, 0.0}  
\hypersetup{
    citecolor=customgreen  
}

\allowdisplaybreaks

\usepackage{xr}
\usepackage{pifont}
\usepackage{apacite}
\usepackage{etoolbox}
\makeatletter

\usepackage{rotating}
\usepackage{bbm}
\usepackage{latexsym}

\renewcommand{\thesubsubsection}{\arabic{section}.\arabic{subsubsection}}

\makeatletter  
\long\def\@makecaption#1#2{%
  \vskip\abovecaptionskip
  \sbox\@tempboxa{{\captionfonts #1: #2}}%
  \ifdim \wd\@tempboxa >\hsize
    {\captionfonts #1: #2\par}
  \else
    \hbox to\hsize{\hfil\box\@tempboxa\hfil}%
  \fi
  \vskip\belowcaptionskip}
\makeatother   
\usepackage{algorithm,algpseudocode}

\usepackage{etoolbox} 

\makeatletter
\patchcmd{\endalgorithmic}
  {\ALG@printbox}
  {}{}{}
\makeatother

\usepackage{array}
\newcolumntype{P}[1]{>{\centering\arraybackslash}p{#1}}
\newcolumntype{M}[1]{>{\centering\arraybackslash}m{#1}}

\makeatletter
\newcommand\thefontsize{The current font size is: \f@size pt}
\makeatother
\usepackage{setspace}
\usepackage{titlesec}
\usepackage{appendix}
\usepackage{caption}

\newcommand{\captionfonts}{\normalsize}

\usepackage{float}
\usepackage{placeins}

\usepackage{afterpage}
\floatstyle{plaintop}

\restylefloat{table}

\makeatletter
\newcommand{\multiline}[1]{%
  \begin{tabularx}{\dimexpr\linewidth-\ALG@thistlm}[t]{@{}X@{}}
    #1
  \end{tabularx}
}

\usepackage{fancyhdr}
\usepackage{bm}
\usepackage{amsmath,tabularx}
\usepackage{amsfonts}
\usepackage{subcaption}
\usepackage{graphicx}

\newtheorem{definition}{Definition}[section]

\newenvironment{proof}{{\bf \emph{Proof.} }}{\hfill $\Box$}

\newtheorem{theorem}{Theorem}[section]
\newtheorem{corollary}[theorem]{Corollary}
\newtheorem{proposition}[theorem]{Proposition}
\newtheorem{lemma}[theorem]{Lemma}

\newtheorem{assumption}[theorem]{Assumption}
\usepackage{graphicx}

\renewcommand*\thesubsection{\arabic{section}.\arabic{subsection}}
\renewcommand*\thesubsubsection{%
  \arabic{section}.\arabic{subsection}.\arabic{subsubsection}%
}

\usepackage{cleveref}
\usepackage{stfloats}
\usepackage{chngcntr}
\def\mS{\mathcal{S}}
\def\mO{\mathcal{O}}
\def\mX{\mathcal{X}}
\def\mA{\mathcal{A}}
\def\mB{\mathcal{B}}

\def\mR{\mathcal{R}}

\def\k{_{k}}
\def\t{_{t}}
\def\nk{_{k+1}}
\def\nt{_{t+1}}
\def\pt{_{t-1}}

\def\s{{s}}

\def\pk{_{k-1}}

\def\l{^{l}}

\def\ta{_{\tau}}

\usepackage{filecontents}

\begin{document}

\hspace{13.9cm}1

\ \vspace{20mm}\\

{ \bf \large Active Inference and Reinforcement Learning: A unified inference on continuous state and action spaces under partial
observability}

\ \\
{\bf \large Parvin Malekzadeh}\\
{p.malekzadeh@mail.utoronto.ca}\\
\ \\
{\bf \large Konstantinos N. Plataniotis}\\
{The Edward S. Rogers Sr. Department of Electrical and Computer Engineering, University of Toronto, Toronto, ON, M5S 3G8, Canada.}


\captionsetup[figure]{name={Fig.},labelsep=period}
\thispagestyle{empty}
\markboth{}{NC instructions}
\ \vspace{-0mm}\\
\thispagestyle{fancy}
\begin{center} {\bf Abstract} \end{center}
Reinforcement learning (RL) has garnered significant attention for developing decision-making agents that aim to maximize rewards, specified by an external supervisor, within fully observable environments. However, many real-world problems involve partial or noisy observations, where agents cannot access complete and accurate information about the environment. These problems are commonly formulated as partially observable Markov decision processes (POMDPs). Previous studies have tackled RL in POMDPs by either incorporating the memory of past actions and observations or by inferring the true state of the environment from observed data. Nevertheless, aggregating observations and actions over time becomes impractical in problems with large decision-making time horizons and high-dimensional spaces. Furthermore, inference-based RL approaches often require many environmental samples to perform well, as they focus solely on reward maximization and neglect uncertainty in the inferred state.
Active inference (AIF) is a framework naturally formulated in POMDPs and directs agents to select actions by minimizing a function called expected free energy (EFE). This supplies reward-maximizing (or exploitative) behaviour, as in RL, with information-seeking (or exploratory) behaviour. Despite this exploratory behaviour of AIF, its usage is limited to problems with small time horizons and discrete spaces, due to the computational challenges associated with EFE.
 In this paper, we propose a unified principle that establishes a theoretical connection between AIF and RL, enabling seamless integration of these two approaches and overcoming their aforementioned limitations in continuous space POMDP settings. We substantiate our findings with rigorous theoretical analysis, providing novel perspectives for utilizing AIF in designing and implementing artificial agents. 
Experimental results demonstrate the superior learning capabilities of our method compared to other alternative RL approaches in solving partially observable tasks with continuous spaces. Notably, our approach harnesses information-seeking exploration, enabling it to effectively solve reward-free problems and rendering explicit task reward design by an external supervisor optional.
\section{Introduction }
Decision-making is the process of evaluating and selecting a course of action from various alternatives based on specific criteria or goals. This process occurs within an environment where an agent interacts and influences its state through actions. When making decisions,\footnote{We use the terms decision-making and action selection interchangeably throughout the work. } the agent carefully assesses available actions and their potential outcomes. Once a decision is made, the agent translates it into action by executing a specific course of action. The effectiveness of a decision is determined by examining its outcomes, which can involve achieving desired results or mitigating undesirable consequences~\cite{puterman2014markov}.

Reinforcement learning (RL)~\cite{sutton2018reinforcement}  is a framework that models agents to interact with an environment typically represented as a Markov decision process (MDP), where the agent has complete and accurate observation of the true state of the environment. In RL, outcomes are often associated with rewards designed by an external supervisor and received by the agent based on its actions. The main objective of RL is to learn optimal policies that define an agent's decision-making strategy to maximize a value function, which represents the expected long-term reward.
Recent advancements in using deep neural networks (DNNs) as parametric function approximators enabled RL algorithms to successfully solve tasks with continuous state and action spaces~\cite{haarnoja2018soft, schulman2017proximal, dai2022diversity}.   Among the popular RL methods designed for continuous state and action spaces, actor-critic methods~\cite{haarnoja2018soft, mnih2016asynchronous, schulman2017proximal} have gained significant attention. These methods employ a policy iteration algorithm to learn the optimal policy.
However, in many real-world scenarios, the true and complete state of the environment is often inaccessible, leading to a situation known as learning under uncertainty or partial observability. Partial observability can arise from various sources, such as temporary information like a way point sign in a navigation task, sensor limitations, noise, and a limited view of the environment. Consequently, the agent must select optimal actions based on incomplete or noisy information about the environment's true (hidden or latent) states. Such environments can be formulated as partially observable Markov decision processes (POMDPs). Unfortunately, identifying the objectives on which an optimal policy is based for POMDPs is generally undecidable~\cite{madani1999undecidability} because partial and noisy observations lack the necessary information for decision-making.

RL algorithms typically assume complete and accurate observation of the environment's states, which limits their performance on partially observable tasks. To address this, various approaches have been proposed to extend RL methods for solving partially observable tasks. One popular approach is the use of memory-based methods, where recurrent neural networks (RNNs) are employed to remember past observations and actions~\cite{zhu2017improving, nian2020dcrac, haklidir2021guided, ni2022recurrent}. However, these methods may become impractical when dealing with large dimensions of actions and/or observations, or when the decision-making time horizon is extensive, due to the demands of maintaining long-term memory.  Consequently, these approaches are mostly applicable to environments with finite and countable action and observation spaces. Additionally, training an RNN is more challenging than training a feed-forward neural network since RNNs are relatively more sensitive to the hyperparameters and structure of the network~\cite{pascanu2013difficulty}. To address these issues, some recent work~\cite{ramicic2021uncertainty, igl2018deep, lee2020stochastic, han2020variational, Hafner2020Dream} has proposed inferring the belief state, which is a representation of the hidden states given the past observations and actions. From the inferred belief state, an optimal policy can be derived to maximize the expected long-term reward. However, due to the partial observability and limitations of the available observations, the agent cannot accurately determine the exact hidden state with certainty. As a result, the inferred belief state is often represented as a probability distribution, indicating the likelihood of different hidden states given past observations and actions.
\\
Given the uncertainty surrounding the inferred state, it is crucial for the agent not to rely solely on its existing knowledge (exploitation). Instead, in a partially observable environment, the agent should actively engage in an information-directed interaction with the environment, aiming to minimize uncertainty and maximize information about the true state of the environment, given only partial or noisy observations of states~\cite{mavrin2019distributional, dong2021variance, likmeta2022directed, maddox2019simple, malekzadeh2022akf, yin2021sequential, malekzadeh2020mm}.

Active inference (AIF)~\cite{friston2017active,friston2010action} is a framework derived from the free energy principle that models and explains the decision-making behaviour of agents in environments modelled as POMDPs.  AIF optimizes two complementary objective functions: variational free energy (VFE) and expected free energy (EFE). The VFE objective is minimized with respect to past interactions with the environment, enabling the agent to learn a generative model of the environment and infer the belief state used for action selection. 
This process of learning the generative model and inferring the belief state is called perceptual learning and inference. 
Action selection involves minimizing the EFE objective with respect to the future, aiming to find an optimal plan—sequences of actions.\footnote{In this paper, we will use the term `plan' to denote a sequence of actions, distinguishing it from a (state-action) policy typically defined in RL.  } By minimizing the EFE, AIF simultaneously maximizes the agent's information gain about the hidden states (exploration) and optimizing its expected long-term reward (exploitation). The information gain represents the enhancement in the agent's knowledge of the environment, allowing it to make more informed decisions as it interacts with the world.  Notably, AIF agents select actions that maximize the expected long-term reward, akin to classical RL, while maximizing information about the environment's hidden states. \\
While AIF is naturally modelled in POMDPs and provides intrinsic information-directed exploration, its applications are limited to environments with discrete state, observation, and action spaces due to computational issues with the EFE~\cite{millidge2020deep, da2020relationship, lanillos2021active,sajid2021exploration}. This limitation arises because AIF finds the optimal plan by computing the EFE for each possible plan and subsequently selecting the plan that minimizes the EFE~\cite{tschantz2020scaling}. Recent studies have proposed approaches to calculate the EFE in a more tractable manner. These approaches include limiting the future decision-making time horizon~\cite{tschantz2020scaling}, employing bootstrapping estimation techniques~\cite{millidge2020deep, Hafner2022}, and utilizing Monte Carlo tree search (MCTS) methods~\cite{maisto2021active, fountas2020deep}. However, it's important to note that these methods are (partially) heuristic and typically applicable only to finite time horizons or discrete action spaces.

Considering the recent advancements in scaling RL for complex, fully observable environments with continuous state and action spaces, and acknowledging the capability of AIF to perform belief state inference and information-seeking exploration in partially observable environments, a compelling question emerges: Is there a relationship between AIF and RL that enables the development of integrated solutions, leveraging ideas from both fields and surpassing the limitations of individual AIF or RL approaches? In this paper, we propose "unified inference"—a convergence of AIF and RL on a common theoretical ground—to harness the mutual advantages of both paradigms and address the challenges posed by realistic POMDP settings with high-dimensional continuous spaces.
\\ The key contributions of our proposed unified inference framework are summarized as follows:
\begin{itemize}
\item \textbf{Extension of EFE to stochastic belief state-action policy:}  We extend the EFE, originally defined for plans~\cite{friston2017active} in AIF, to accommodate the learning of a stochastic policy in an infinite time horizon POMDP setting with continuous state, action, and observation spaces. This extension of the EFE to a belief state-action policy allows actions to be chosen at each time step based on the inferred belief state, eliminating the need to enumerate every possible plan into the future. This extension integrates information-seeking exploration and reward-maximization under partial observability learning, making it a practical and effective objective function for action selection in both AIF and RL frameworks within POMDP settings with continuous spaces. We hence refer to this extension as the unified objective function. Our experiments demonstrate that introducing stochasticity to the policy significantly improves the stability and robustness of our algorithm in tasks with continuous spaces, where challenges such as belief state inference and hyperparameter tuning pose significant obstacles.
\item \textbf{Unified policy iteration in continuous space POMDPs:}  To optimize the proposed unified objective function and find the optimal policy, we introduce a computationally efficient algorithm called unified policy iteration. This algorithm generalizes the policy iteration guarantees in MDPs~\cite{haarnoja2018soft} to POMDPs and provides theoretical proof of its convergence to the optimal policy. Within the unified policy iteration framework, the extended EFE can be treated as a negative value function from an RL perspective. This theoretical connection demonstrates three important aspects: \textit{(i)} AIF can be analyzed within the framework of RL algorithms, enabling insights from scaling RL to infinite time horizons and continuous space MDPs to be directly applied to scaling AIF for use in infinite time horizons and continuous space POMDPs. This provides deep learning practitioners with a starting point to leverage RL findings and further advance AIF in challenging tasks. \textit{(ii)} We can generalize a range of RL approaches designed for continuous state and action MDPs to continuous state, action, and observation space POMDPs while incorporating an inherent information-seeking exploratory term. This bridges the gap between MDPs and POMDPs and also opens up new possibilities for applying state-of-the-art RL techniques to challenging decision-making problems. \textit{(iii)} We can extend a range of reward-dependent RL methods to a setting where the reward function is not determined by an external supervisor. This aspect of the proposed policy iteration is significant as it eliminates the challenges associated with designing reward functions.  For instance, a poorly designed reward function may lead to slow learning or even convergence to sub-optimal policies, highlighting the importance of mitigating the necessity of defining task rewards through unified policy iteration.
\item \textbf{Unified actor-critic for continuous space POMDPs:}  Building upon the parametric function approximations in the unified policy iteration, we present a novel unified actor-critic algorithm that unifies actor-critic methods for both MDPs and POMDPs. Our approach stands out as one of the few computationally feasible methods for addressing POMDPs and opens new possibilities for enhancing the performance of the popular reward-maximizing actor-critic RL algorithms, such as soft actor-critic (SAC)~\cite{haarnoja2018soft} and Dreamer~\cite{Hafner2020Dream} when applied to real-world scenarios that inherently involve partial observability. To evaluate the effectiveness of our unified actor-critic algorithm in addressing high-dimensional continuous space POMDPs, we conduct experiments on modified versions of Roboschool tasks~\cite{brockman2016openai}, where agents have access to partially observed and noisy states. The experimental results demonstrate that our proposed unified actor-critic method achieves superior sample efficiency and asymptotic performance compared to existing frameworks in the literature.
\end{itemize}
{A comparative overview of foundational elements and decision-making strategies within RL, AIF, and our proposed unified inference approach, with a particular focus on their application in continuous decision spaces, is presented in Table~\ref{table:summary}. This summary encapsulates the methodologies employed by these frameworks, highlighting their distinctive approaches to learning and decision-making in environments characterized by continuous decision spaces.
}
%

\begin{table}[!t]
\centering
\begin{tabular}{|M{3.1cm}|M{2.6cm}|M{3.45cm}|M{3.45cm}|}
\hline
\textbf{Aspect} & \textbf{RL} & \textbf{AIF} & \textbf{Unified Inference} \\ \hline
\textbf{Decision Space} & Continuous MDP & Continuous POMDP & Continuous POMDP \\ \hline
\textbf{Foundation for Decision-Making} & N/A & Perceptual inference \& learning via VFE & Perceptual inference \& learning via VFE \\ \hline
\textbf{Decision-Making Strategy} &  State-action policy learning & Plan learning & Belief state-action policy learning \\ \hline
\textbf{Decision-Making Objective} & Value function (reward maximization) & EFE (reward + information gain maximization) & Unified objective (reward + information gain + policy entropy maximization) \\ \hline
\textbf{Optimization Mechanism} & Policy iteration & Enumeration & Unified policy iteration \\ \hline
\textbf{Learning Approach} & Actor-critic & Bootstrapping, MCTS & Unified actor-critic \\ \hline
\end{tabular}
\caption{{ Comparative overview of foundational and decision-making aspects in RL, AIF, and unified inference under continuous decision spaces. }}
\label{table:summary}
\end{table}
\subsection{Overview}
In Section~\ref{Sec:problem setting}, we present a comprehensive introduction to the problem's nature, which serves as the foundation for our study.
Section~\ref{Sec:problem} offers an introductory tutorial on POMDPs and covers essential concepts like policies, generative models, and inference.
We review the RL paradigm for MDPs and the AIF paradigm for POMDPs in Section~\ref{sec:overview}.
Section~\ref{sec:unified_inference} introduces the proposed unified inference framework and establishes its convergence to the optimal policy.
In Section~\ref{sec:design}, we delve into the implementation and modeling aspects of the proposed unified inference, demonstrating how our formulation extends various existing MDP-based actor-critic methods directly to their respective POMDP cases. Section~\ref{sec:related} discusses existing works related, and Section~\ref{sec:results} presents the experiments utilized to evaluate the performance of our method.
Finally, we conclude the paper and outline potential avenues for future research in Section~\ref{sec:con}.
\section{Problem characteristics} \label{Sec:problem setting}
This section outlines our assumptions concerning various key elements of the decision-making problem addressed in this paper. 
\begin{enumerate}
\item {  \textit{Finite or infinite horizon:}  We assume agents interact with environments over an infinite horizon, appropriate in scenarios where decisions have long-term impacts. Infinite horizon problems, requiring analysis of endless action sequences, are inherently more complex than their finite counterparts.}
\item {  \textit{Fully or partially observable:} We tackle sequential decision-making in partially observable environments. Partial observability is prevalent in many real-world tasks and challenges decision-making compared to fully observable settings~\cite{igl2018deep}.
}
\item {  \textit{Discrete or continuous state, observation, and action spaces:}  We consider environments with continuous state, observation, and action spaces. These continuous spaces are common in complex applications, as they enhance the representation of the environment and actions, and better handle complexity.}
\item { \textit{Stochastic or deterministic environment:}  We consider stochastic environments introducing randomness in action outcomes.  Unlike deterministic environments, stochastic environments require strategies that account for outcome variability, enhancing decision robustness and capturing real-world complexity.}
%
\item {  \textit{Stationary or non-stationary environment:} We assume stationary environments with consistent statistical properties, a common simplification to manage the computational complexity of non-stationary environments~\cite{sutton2018reinforcement, puterman2014markov}. While focused on stationary settings, the methodologies proposed in this paper also apply to non-stationary environments.}
\item { \textit{Markovian or non-Markovian policy:} 
We focus on Markovian policies that rely solely on the current belief state, streamlining computation and memory compared to non-Markovian (history-dependent) policies.}
\item { \textit{Stochastic or deterministic policy:} 
We adopt stochastic policies for their enhanced ability to explore environments and handle uncertainties, facilitating adaptive responses. These features promote better generalization and task-specific fine-tuning~\cite{ramicic2021uncertainty, haarnoja2018soft}.
}
\end{enumerate}

Table~\ref{Table:setting} provides a comprehensive comparison between the key characteristics of our problem and the problems addressed by most RL and AIF paradigms. This comparison covers aspects such as the observability of environment states, decision-making horizon, stochastic properties of the environment, and policy class. By analyzing these elements in comparison to typical RL and AIF problems, we gain insights into the distinct nature and unique challenges of our problem setting.
\begin{table*}[!t]
\caption{Comparison of our problem setting with common problem settings addressed by RL and AIF algorithms. 
}\label{Table:setting}
\centering
\begin{tabular}{|M{1.2cm}|M{1.25cm}|M{1.45cm}|M{2.102cm}|M{1.95cm}|M{1.75cm}|M{2.3cm}|  }
 \hline
 \centering { \textbf{Method}} &  \centering {\textbf{Horizon}} &   \textbf{ Observability}  & {\textbf{State \&  observation space}} &  {\textbf{Action space}}  &    \textbf{ Environment} & {\textbf{Policy}} 
\\ \hline
 \centering {RL} &   \centering { Finite/ Infinite}  & {Full}  &  { Discrete/ Continuous} & {  Discrete/ Continuous}  &  {Stochastic \& Stationary} & {Markovian \& Stochastic} 
\\ \hline
 \centering  AIF &   \centering { Finite} & {Partial}  &  { Discrete} & { Discrete}   &  {Stochastic \& Stationary} & {Non-Markovian \& Deterministic}  
\\ \hline
 \centering  Ours &   \centering { Infinite} & {Partial}  &  { Continuous} & { Continuous}   &  {Stochastic \& Stationary} & {Markovian \& Stochastic} 
\\ \hline
\end{tabular}
\end{table*}
\section{Preliminaries and Problem modeling} \label{Sec:problem}
In the previous section, we outlined the characteristics of the problem addressed in this paper, involving an agent making decisions within a stochastic environment with continuous state, observation, and action spaces. This section delves into modeling these characteristics, beginning with an overview of relevant concepts including POMDPs, policies, generative models, and belief state inference, and then detailing the specific models applied to our problem.
\\
We assume the agent receives observations and makes decisions at discrete time steps $t$, continuing indefinitely, where $t \in \{0, 1, 2, \ldots\}$. For notation, $x_t$ represents variable $x$ at time step $t$, and $x_{f:h}$ includes all elements from $t=f$ to $t=h$, i.e., $x_{f:h} = (x_f, x_{f+1}, ..., x_h)$. { Additionally, $\Delta(\mX)$ denotes the set of all probability distributions\footnote{We use the probability distribution function for both the probability mass function (PMF) and probability density function (PDF). } over set $\mX$, and $|\mX|$ represents its cardinality.} For further clarification, a comprehensive list of all notations used is provided in Table~\ref{Table:1} in the appendix.
\subsection{Partially observed Markov decision processes (POMDPs)} \label{subsec:POMDP}
A POMDP is fully specified by $\mathcal{M}=(\mS, \mO, \mA,  d_0, \Omega, \mR, U, \Theta)$, where $\mathcal{S} \in \mathbb{R}^{D_{\mS}}$ represents the state space encompassing all possible (hidden or latent) states of the environment. Here, $\mathbb{R}$ denotes the set of real numbers, and $D_{\mS}$ is the dimension of the state space.
$\mathcal{O} \in \mathbb{R}^{D_{\mO}}$ denotes the observation space, the set of all possible observations the agent can receive, with $D_{\mO}$ as the dimension of the observation space.
$\mathcal{A} \in \mathbb{R}^{D_{\mA}}$ is the action space, indicating all possible actions the agent can take, where $D_{\mA}$ is the dimension of action space.
These spaces can be either discrete or continuous.
 At time step $t$, $s_t$, $o_t$, and $a_t$ denote elements from $\mathcal{S}$, $\mathcal{O}$, and $\mathcal{A}$, respectively.
 \\
 $d_0: \mS \rightarrow \Delta(\mS)$ denotes the probability distribution of the initial latent state, with $d_0({s}_0)$ specifying the probability of the environment starting in state ${s}_0 \in \mS$.  
$\Omega:\mS \times \mA \rightarrow \Delta(\mS)$ is the (forward) transition function, such that $\Omega({s}\nt|{s}\t,a\t)$ specifies the probability of transitioning to ${s}\nt \in \mS$ from $s\t$ after acting $a\t$. \\
$\mR$ is the set of possible rewards and is called the reward space. We will use $r\t$ to denote an element of the reward set $\mR$ at time $t$. Adhering to RL and AIF standards~\cite{friston2017active, sutton2018reinforcement}, we assume that $\mR \in \mathbb{R}$, although more general vector-valued reward functions are also possible.
$U: \mS \times \mA \rightarrow \Delta(\mR)$ is the reward function, such that $U(r_t | s_t, a_t)$ indicates the probability of a reward value $r_t$ given state $s_t$ and action $a_t$. The reward function $U$ is also known as the external (or extrinsic) reward function since its value is determined by an external supervisor or designer.\footnote{In this work, we interchange the terms reward function and extrinsic reward function when referring to $U$. } 
\\
$\Theta:\mS \rightarrow \Delta(\mO)$ is the observation function, such that $\Theta(o\t|s\t)$ specifies the probability of observing $o\t$ given state $s\t$. This distinguishes POMDPs from MDPs, where the agent indirectly observes $s_t$ through $o_t$. In MDPs, observations directly reflect the true state, making $o_t = s_t$.
A POMDP is deemed deterministic if $\Omega$, $U$, and $\Theta$ are all deterministic. Otherwise, it is considered stochastic. Furthermore, the POMDP  is termed stationary if these functions remain constant over time and non-stationary if they vary.

Following assumptions 2-5 in Section~\ref{Sec:problem setting}, we represent our environment as a stationary stochastic POMDP with continuous state, observation, and action spaces.

\subsection{Policy} \label{sec:policy}
A policy is a sequence of decision rules denoted by $\pi_{0:\infty}=(\pi_0, \pi_1, \pi_2,  ..., \pi\t, ...)$. Each decision rule $\pi_t$ models how an agent selects action $a\t \in \mA$  at time step $t$. To identify an optimal policy, a specific optimality criterion is defined.  \\
A Markovian decision rule $\pi_t(a_t|o_t)$ operates solely on the current observation $o_t$, whereas a history-dependent decision rule $\pi_t(a_t|a_{0:t-1}, o_{0:t})$ utilize the complete history of actions $a_{0:t-1}$ and observations $o_{0:t-1}$ until time $t$. Decision rules can be deterministic, selecting actions with certainty, or stochastic, introducing probability into the selection process~\cite{puterman2014markov}. 

The distinctions between Markovian, history-dependent, deterministic, and stochastic decision rules give rise to corresponding classes of policies.  
\\
Moreover, a policy $\pi_{0:\infty}$ is considered stationary if the decision rule $\pi_t$ remains constant for all time instants $t \in \{0, 1, 2, \dots\}$. Otherwise, it is called non-stationary. In the case of a stationary POMDP with an infinite time horizon, it is common to assume a stationary policy~\cite{puterman2014markov, russell2002artificial}. This assumption arises from the understanding
that at each time step, an infinite number of future time steps exist, and the statistical properties of the environment remain unchanged. Adopting a stationary policy can simplify the agent’s decision-making process.

{ In an MDP where $o_t = s_t$ and states exhibit the Markovian property, all policies are Markovian \cite{puterman2014markov}. A Markovian policy within an MDP selects action $a_t$ based solely on the current state, denoted as $\pi(a_t|s_t)$, and is thus referred to as a state-action policy.
Conversely, in POMDPs where observations lack the Markovian property, history-dependent policies may outperform Markovian ones by enabling more informed decisions \cite{igl2018deep, montufar2015geometry, puterman2014markov}. However, each time step in POMDPs introduces one new action and one new observation to the history, adding $D_{\mathcal{O}} + D_{\mathcal{A}}$-dimensional data to the existing history. Furthermore, the addition of one observation and one action to the history at each time step leads to an exponential increase in the number of possible histories. Specifically, the number of possible histories expands at each step by a factor of $(|\Omega| \times |\mathcal{A}|)$, reflecting the agent's $|\mathcal{A}|$ possible actions and $|\Omega|$ possible observations. 
Consequently, the size of the policy space, which represents the number of possible policies derived from these histories, also undergoes the same exponential growth.
\\
To tackle these challenges, certain methods \cite{igl2018deep, kochenderfer2015decision, montufar2015geometry} employ belief state-action policies, where decisions at time $t$ are based on a belief state $b_t \in \mathcal{B}$, i.e., $\pi_t(a_t|b_t)$. Here, the belief state $b_t$ represents a probability distribution over potential states of the environment at time $t$, and $\mathcal{B} \in \mathbb{R}^{|\mathcal{S}|}$ encompasses all possible belief states. Derived using Bayes' rule through a process known as inference, the belief state consolidates all pertinent historical information necessary for action selection, effectively rendering belief state-action policies Markovian.
\\
Unlike history-dependent policies, belief states in belief state-action policies maintain a constant size, $|\mathcal{S}|$, at each time step. Consequently, the memory and computation required to store and evaluate each belief state and its corresponding policy does not inherently increase over time, providing significant efficiency advantages over history-dependent approaches, particularly in environments with infinite horizons or high-dimensional action and observation spaces. However, inferring and storing belief states can be computationally and memory-intensive, particularly in large or continuous state spaces. Nevertheless, this memory and computational requirement generally falls below that of history-dependent policies~\cite{yang2021recurrent}.
Further details comparing the complexity of history-based versus belief state-action policies are provided in Appendix~\ref{APP:belief_history}.}
\subsection{Generative model} \label{Sec:GM}
 With policy $\pi$ and initial state distribution $d_0$, interactions between an agent and a stationary POMDP $\mathcal{M}$ gives the following sequence $(s_0,o_0,a_0, s_1, o_{1}, a_{1}, \ldots )$ with the probability distribution:
\begin{eqnarray}
 \!\!\!\!\!\!\!\!\!\!\!\!\!\!\! p(s_0,o_0,a_0, s_1, o_{1}, a_{1}, ...|\pi) \footnotemark
  = p(s_0,o_0, s_1, o_{1}, ... |a_{0:\infty})  \prod_{k=1}^{\infty} p(a\pk|\pi), \label{Eq:GP}
\end{eqnarray}
\footnotetext{We abuse notation by writing $p(.)$ for both the PMF and PDF. Moreover, we use $p(x\t)$ as shorthand for $p(X\t = x\t)$, where $X\t$ is a random variable at time $t$ taking on values $x\t \in \mathcal{X}$. }
where  $p( a\pk| \pi)$ denotes the probability that policy $\pi$ chooses action $a\pk$.
Given $\Theta$ and $\Omega$, $p(s_0,o_0,s_1, o_{1},  ...|\pi)$ in Eq.~\eqref{Eq:GP} simplifies to:
\begin{eqnarray}
 \!\!\!\!\!\! \!\!\!\!\!\! p( s_0,o_0, s_1, o_{1}, ...|a_{0:\infty})  &\!\!\!=\!\!\!&  d_0(s_0)  p(o_0|s_0)  \prod_{k=1}^{\infty} p(o\k|{s}\k) p({s}\k|{s}\pk,a\pk) \nonumber \\
 &\!\!\!=\!\!\!&  d_0(s_0)  \Theta(o_0|s_0)  \prod_{k=1}^{\infty} \Theta(o\k|{s}\k) \Omega({s}\k|{s}\pk,a\pk). \label{Eq:HMM}
\end{eqnarray}

As explained in Sub-section~\ref{sec:policy}, the belief state is a variable estimated by the agent. To infer the belief state using Bayesian inference (as detailed in Sub-section~\ref{sub-sec:inference}), the agent relies on $p(s_{0:\infty}, o_{0:\infty}|a_{0:\infty})= p(s_0,o_0, s_1, o_{1}, \ldots|a_{0:\infty})$.
However, the agent lacks direct access to $\Omega$ and $\Theta$; thus, it constructs a generative model of the POMDP denoted as $P(s_{0:\infty}, o_{0:\infty}|a_{0:\infty})=P(s_0,o_0, s_1, o_{1}, \ldots|a_{0:\infty})$, which is decomposed as follows \cite{han2020variational, friston2017active, lee2020stochastic}: \footnote{We use $P$ to represent the agent's probabilistic model of the POMDP, which is a learned approximation, while $p$ denotes the true probabilistic components of the POMDP. Since the agent is unaware of the true probabilistic components, we solely refer to the agent's generative model from now on.}
\begin{eqnarray}
 P(s_{0:\infty}, o_{0:\infty}|a_{0:\infty}) = d_0(\s_0)  P(o_0|\s_0)  \prod_{k=1}^{\infty} P(o\k|{\s}\k) P({s}\k|{\s}\pk,a\pk), \label{Eq:final_GM} 
\end{eqnarray}
where $P(o_t | s_t)$ and $P(s_t|s_{t-1},a_{t-1})$  are the agent's models of observation function ${\Theta}$ and the transition function $\Omega$, respectively. These are referred to as the likelihood function or the observation model, and the transition model, respectively.
 
The generative model of an MDP is derived by setting $P(o_t|s_t)=1$ in the POMDP's model. Fig.~\ref{Fig:Relationship} shows the dependency relationships in the generative models for POMDPs and MDPs, illustrating how a POMDP's model encompasses that of an MDP.
%
%
\begin{figure}[!t]
\centering
\includegraphics[scale=1.]{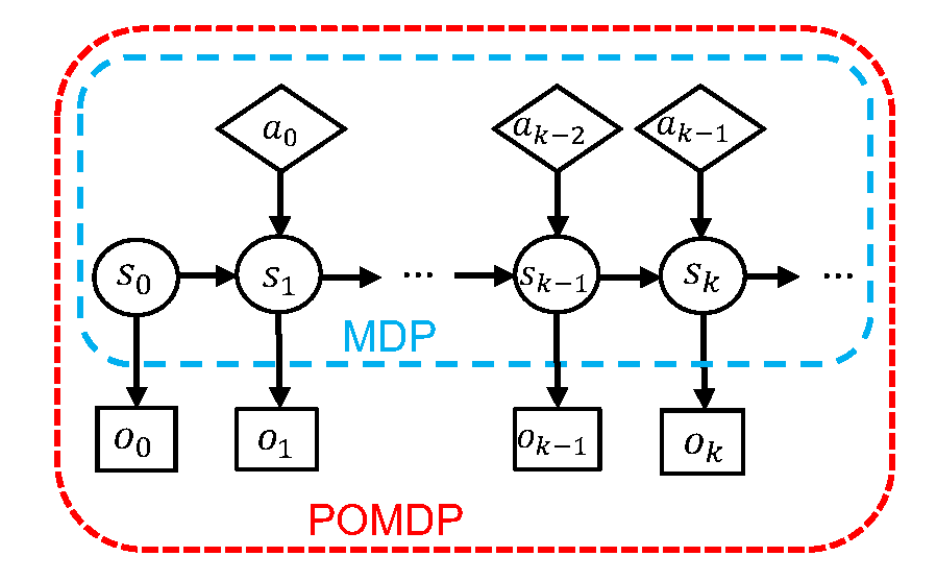}
\caption{ Relationship between generative models of  MDPs and POMDPs. Arrows indicate dependence.  }\label{Fig:Relationship}
\end{figure}
\subsection{Inference} \label{sub-sec:inference}
The belief state $b_t$ utilized by a Markovian belief state-action policy $\pi(a_t|b_t)$ signifies a probability distribution over the state space $\mathcal{S}$ at time $t$, where $b_t(s_t) = P(s_t | o_{0:t}, a_{0:t-1})$. Through a generative model for the environment, the agent estimates $b_t$, a process known as inference~\cite{igl2018deep}.
Inference is performed using Bayes' rule:
\begin{eqnarray}
b_t = \underbrace{ \frac{\int{b\pt(s\pt) P(o\t|\s\t)  P(\s\t|\s\pt,a\pt)}  ds\pt } {\int \int{b\pt(s\pt) P(o\t|\s\t) P(\s\t|\s\pt,a\pt)} ds\t ds\pt}}_{\text{BeliefUpdate} (b\pt, s\t, o\t, a\pt)}. \label{Eq:beleif-state}
\end{eqnarray}
Eq.~\eqref{Eq:beleif-state} presents Bayes' rule for inference in POMDPs with continuous state and observation spaces. For discrete spaces, summations would be used instead. Given the belief state $b_{t-1}$, an action $a_{t-1}$, an observation $o_{t}$, and the latent state $s_{t}$, the belief state $b_{t}$ is fully determined via the recurrent update function $\text{BeliefUpdate}$. Consequently, the belief state forward distribution $P(b_{t}| s_{t}, o_{t}, b_{t-1}, a_{t-1})$, representing the distribution over the subsequent belief state, is $\delta(b_t- \text{BeliefUpdate}(b_{t-1}, s_{t}, o_{t}, a_{t-1}))$, where $\delta$ denotes the Dirac delta function. Therefore, the belief state $b_{t}$ depends solely on the previous belief state $b_{t-1}$, exhibiting the Markovian property.
 
In practice, exact inference often becomes computationally intractable in continuous spaces due to the integrals involved. To address this, we use variational inference, a common approximation technique in the RL and AIF literature~\cite{Hafner2020Dream, lee2020stochastic, millidge2020deep, mazzaglia2021contrastive}. Details on variational inference are provided in Sub-section~\ref{sub:AIF_inference}.
%
\section{Review of the State-of-The-Art algorithms} \label{sec:overview}
RL and AIF are the primary frameworks for decision-making problems. RL algorithms typically assume full observability and model environments as MDPs, whereas AIF assumes partial observability and models them as POMDPs. This section overviews their fundamental concepts and underscores their key distinctions.

\textbf{Note 4.1:} In the rest of this paper, we fix the current time $t$ and assume that the agent focuses on optimizing future actions for time steps $\tau \in \{t, t+1, ...\}$.
\subsection{Reinforcement learning (RL)}
As RL algorithms are formulated within MDPs, it is sufficient to consider a Markovian state-action policy $\pi(a_\tau | s_\tau)$~\cite{sutton2018reinforcement}. The goal of RL is to find an optimal state-action policy $\pi^*$ maximizing the expected long-term reward $J^\pi$, expressed as $J^{\pi} = \mathbb{E}_{  \prod_{\tau=t}^{\infty} P({\s}_{\tau}|{\s}_{\tau-1},a_{\tau-1}) P(r_{\tau}|{\s}_{\tau},a_{\tau}) \pi(a\ta|\s\ta)} \left[  \sum_{\tau=t}^{\infty}  \gamma^{\tau-t}  r\ta \right]$, where  $\gamma \in [0,1)$ is the discount factor used to ensure the sum is bounded. 
RL employs various methods to find the optimal policy $\pi^*$. Three common categories relevant to our context are as follows: 
\begin{itemize}
\item \textit{Policy-based approaches:} These algorithms directly find $\pi^*$ by maximizing $J^{\pi}$ using  gradient ascent methods. Policy-based approaches have shown success in dealing with continuous state and action spaces~\cite{ma2016online, schulman2015trust}, but they often suffer from high variance in gradient estimates~\cite{tucker2018mirage}.
\item \textit{Value-based methods:} Value-based algorithms rely on learning either the state value function $V^{(\pi)}(s_t)$ or the state-action value function $Q^{(\pi)}(s_t, a_t)$, which are obtained by conditioning $J^\pi$ on the state $S_t = s_t$ and the state-action pair $(S_t = s_t, A_t = a_t)$, respectively:
 \begin{eqnarray}
  \!\!\!\!\!\!\!\!\! V^{(\pi)}(\s\t) &\!\!=\!\!& \mathbb{E}_{ \prod_{\tau=t}^{\infty} P({\s}_{\tau+1}|{\s}_{\tau},a_{\tau})P(r_{\tau}|{\s}_{\tau},a_{\tau}) \pi(a\ta|\s\ta)} \bigg[ \sum_{\tau=t}^{\infty}  \gamma^{\tau-t} r\ta| \s\t \bigg], \\
\!\!\!\!\!\!\!\!\!\!\!\! Q^{(\pi)}(\s\t,a\t) &\!\!=\!\!& \mathbb{E}_{ \prod_{\tau=t}^{\infty} P({\s}_{\tau+1}|{\s}_{\tau},a_{\tau}) P(r_{\tau}|{\s}_{\tau},a_{\tau}) \pi(a\ta|\s\ta)} \bigg[ \sum_{\tau=t}^{\infty}  \gamma^{\tau-t} r\ta | \s\t, a\t \bigg], 
\end{eqnarray}
Value-based algorithms calculate $V^{(\pi)}(s_t)$ and $Q^{(\pi)}(s_t, a_t)$ recursively using Bellman equations:
%
\begin{eqnarray}
\!\!\!\!\!\!\!\!\!\!\!\! V^{(\pi)}(\s\t) &\!\!=\!\!& \mathbb{E}_{\pi(a\t|\s\t) P(r\t|{\s}\t,a\t)  } \left[  {r\t} +  \gamma \mathbb{E}_{P(\s\nt|\s\t,a\t) }[V^{\pi}(\s\nt) ] \right], \label{Eq:Bellman}  \\
\!\!\!\!\!\!\!\!\!\!\!\! Q^{\pi}(\s\t, a\t) &\!\!=\!\!& \mathbb{E}_{P(r\t|{\s}\t,a\t)} [r\t] + \mathbb{E}_{P(\s\nt|\s\t,a\t)\pi( a\nt|\s\nt)  }[Q^{\pi}(\s\nt,a\nt) ]. \label{Eq:Bellman_Q} 
\end{eqnarray}
%
Using the Bellman equations, RL iteratively updates these value functions using the Bellman operators~\cite{sutton2018reinforcement}. The Bellman operator for state-action value function is denoted as $T^{RL}_{\pi}$ and is defined as:
\begin{eqnarray}
\!\!\!\!\!\!\!\!\!\!\!\! \!\!\!\! T^{RL}_{\pi} Q(\s\t, a\t) &\!\!\!\!\!: =\!\!\!\!\!& \mathbb{E}_{P(r\t|{\s}\t,a\t)} [r\t] + \mathbb{E}_{P(\s\nt|\s\t,a\t)\pi( a\nt|\s\nt)  }[Q(\s\nt,a\nt) ]. 
\end{eqnarray}
The optimal policy is then found $\pi^*=\arg\max_{\pi} V^{(\pi)}(\s\t)= \arg\max_{\pi} Q^{(\pi)}(\s\t,a\t)$. 
\\
It has been shown that the following Bellman optimality equations on the optimal state value function $V^*(\s\nt):= V^{(\pi^*)}(\s\t)$ and the optimal state-action value function $Q^*(\s\t,a\t):= Q^{(\pi^*)}(\s\t,a\t)$ hold~\cite{sutton2018reinforcement}:
\begin{eqnarray}
\!\!\!\!\!\!\!\!\!\!\!\!\!\!\!\! V^{*}(\s\t) &\!\!=\!\!& \max_{a \in \mA} \bigg\{\mathbb{E}_{P(r\t|\s\t,a)  }[ r\t] +  \gamma \mathbb{E}_{P(\s\nt|\s\t,a)  }[V^{*}(\s\nt) ] \bigg\}, \label{Eq:Bellman_opt} \\
\!\!\!\!\!\!\!\!\!\!\!\!\!\!\!\! Q^{*}(\s\t,a\t) &\!\!\!\!\!=\!\!\!\!\!&    \mathbb{E}_{P(r\t|\s\t,a)  }[ r\t] +  \gamma  \mathbb{E}_{P(\s\nt|\s\t,a\t)  }[\max_{a' \in \mA}  Q^{*}(\s\nt,a') ]. \label{Eq:Bellman_opt_Q}
\end{eqnarray}
Thus, value-based algorithms can calculate the optimal state value function \(V^*(s_t)\) or state-action value function \(Q^*(s_t, a_t)\), and derive the optimal action \(a^*\) as \(a^* = \arg\max_{a} Q^*(s_t,a)\) or \(a^* = \arg\max_{a} V^*(s_t)\). However, when dealing with large action spaces, using the max operator to select the best action can become computationally expensive~\cite{ okuyama2018autonomous}.
%
 %
\item \textit{{Actor-critic methods:}} These hybrid approaches blend policy-based and value-based techniques. The policy, or actor, selects actions, while the critic estimates the state or state-action value function, evaluating the actor's decisions. Using policy gradients, the actor improves its policy, while the critic assesses this enhanced policy by estimating its corresponding value function. This iterative process, known as policy iteration, guarantees convergence to the optimal policy under certain conditions with sufficient iterations.  Additionally, policy iteration is computationally efficient since it updates the policy at each iteration, making it suitable for problems with continuous state and action spaces. 
\end{itemize}
RL algorithms, including value-based, policy-based, and actor-critic methods, can be either model-free or model-based, depending on how they learn the optimal policy. Model-based algorithms learn a generative model of the MDP and a reward model, offering improved sample efficiency~\cite{Hafner2020Dream}. 
Conversely, model-free methods, while not relying on explicit models, often require more environment interactions for optimization. 
Recent approaches, like CVRL~\cite{ma2021contrastive}, integrate model-free and model-based components to leverage their respective strengths.

Despite RL's advancements in continuous spaces, it primarily operates in fully observable environments. As such, its applicability in partially observable scenarios remains limited~\cite{da2020relationship, han2020variational, igl2018deep}.
\subsection{Active inference (AIF)} 
AIF has gained attention as a framework unifying inference and action selection under the free energy principle in POMDPs \cite{friston2017active, friston2011action}. In AIF, the agent engages in inferring and learning a generative model of the POMDP and then utilizes the inferred belief state and generative model to seek an optimal plan, a sequence of future actions, that minimizes the EFE.
 %
 \subsubsection{Perceptual inference and learning } \label{sub:AIF_inference}
AIF assumes that the agent has access to past observations and actions and models the instant generative model as $P(\s_t, o_t | a_{t-1}, \s_{t-1})$~\cite{friston2017active, millidge2021whence}. Upon receiving the observation $o_t$, AIF learns $P(\s_t, o_t | a_{t-1}, \s_{t-1})$ and utilizes variational inference to approximate the belief state $b_t$ using the variational posterior $q(\s_t | o_t)$.\footnote{Throughout this paper, the notation $q$ is used to represent variational probability distributions.}
This is achieved by minimizing the VFE at time $t$, denoted as $F_t$. The VFE is an upper bound on the current Bayesian surprise, defined as $-\log p(o_t)$, where $\log$ represents the natural logarithm function. In machine learning, the VFE corresponds to the negative Evidence Lower Bound (ELBO) within the framework of Variational AutoEncoder (VAE)~\cite{kingma2013auto}. $F_t$ is defined as follows:
\begin{eqnarray}
\!\!\!\!\!\!\!\!  F\t &\!\!\!=\!\!\!& \mathbb{E}_{q(\s\t|o\t)} \left[ \text{log} q(\s\t|o\t) - \text{log} P(\s\t, o\t|a\pt,\s\pt) \right]. 
\end{eqnarray}
By factoring $P(\!\s_t,o_t|a_{t-1},\s_{t-1}\!)$ into $P(\!o_t|\s_t\!)$ and $P(\!\s_t|\s_{t-1}, a_{t-1}\!)$, $F_t$ can be rewritten as
\begin{eqnarray}
F\t  &=& -\mathbb{E}_{q(\s\t|o\t)} \left[\text{log} p(o\t|\s\t) \right] + D_{KL} [q(\s\t|o\t),P(\s\t|\s\pt,a\pt)],  \label{Eq:VFE}
\end{eqnarray}
where $D_{KL}$ is Kullback Leibler (KL)-divergence. By minimizing $F_t$, the agent ensures that its generative model aligns with the received observations and the inferred states.
The process of inferring $q(\s_t | o_t)$ by minimizing $F_t$ is called perceptual inference, and minimizing $F_t$ with respect to the transition model $P(\s_t | \s_{t-1}, a_{t-1})$ and the likelihood function $P(o_t | \s_t)$ is known as perceptual learning in AIF~\cite{friston2017active}. 
 \subsubsection{Plan selection}
AIF assumes that the agent has a desired distribution $\tilde{P}(o_{\tau+1})$ over future observations $o_\tau$ for $\tau \in \{ t, t+1, ..., T\}$, where $T$ represents a finite time horizon. This desired distribution, known as the prior preference, encodes the agent's goals.
\\
Looking ahead to future time steps, the agent in AIF aims to minimize its expected future Bayesian surprise based on its prior preference $\tilde{P}(o_{\tau+1})$. This is achieved by minimizing the EFE given the current observation $o_t$ over all possible plans $\tilde{a}:=a_{t:T-1}$~\cite{millidge2021whence}. The EFE is denoted as $G_{\text{AIF}}(o\t)$ and is defined as follows:
\begin{eqnarray}
 {G}_{\text{AIF}}(o\t) &=&  \mathbb{E}_{q({\s}_{t:T}, {o}_{t+1:T}, \tilde{a}|o\t )} \bigg[  \text{log} \frac{q(\s_{t+1:T},\tilde{a}|\s\t )}{{P}({\s}_{t+1:T}, o_{t+1:T}|\s\t,\tilde{a}) } \bigg], \label{Eq:EFE_def_prod}
\end{eqnarray}
where the variational distribution $q(\s_{t+1:T}, \tilde{a})$ infers future states and actions and the variational distribution $q({\s}_{t:T}, {o}_{t+1:T}, \tilde{a}|o\t )$ considers expectations for future observations. AIF factors $q(\s_{t+1:T},\tilde{a}|\s\t )$ as
$q(\s_{t+1:T},\tilde{a}|\s\t )=\prod_{\tau =t}^{T-1} q(\tilde{a} ) P(\s_{\tau+1}|\s\ta,a\ta)$ and approximates the generative model ${P}({\s}_{t+1:T}, o_{t+1:T}|s\t, \tilde{a})$ in Eq.~\eqref{Eq:EFE_def_prod} with a biased generative model $\tilde{P}({\s}_{t+1:T}, o_{t+1:T}|s\t, \tilde{a})$, defined as follows:
\begin{definition} {(Biased generative model)} \label{def:AIF_biased_GM} \\
Given the prior preference $\tilde{P}(o_{\tau+1})$, the generative model ${P}({\s}_{t+1:T}, o_{t+1:T}|s\t, \tilde{a})$ in Eq.~\eqref{Eq:EFE_def_prod} is approximated by a biased generative model $\tilde{P}({\s}_{t+1:T}, o_{t+1:T}|s\t, \tilde{a})$, which factors as: 
\begin{eqnarray}
\!\!\!\!\!\!\!\!\! \tilde{P}({\s}_{t+1:T}, o_{t+1:T}|s\t, \tilde{a}) = \prod_{\tau =t}^{T-1}  \tilde{P}(o_{\tau+1}|s\ta, a\ta) q({\s_{\tau+1}}|{o_{\tau+1}}). \label{Eq:AIF_GM} 
\end{eqnarray}
\end{definition}

By  replacing the generative model ${P}({\s}_{t+1:T}, o_{t+1:T}|s\t, \tilde{a})$ with the biased generative model $\tilde{P}({\s}_{t+1:T}, o_{t+1:T}|s\t, \tilde{a})$, $G_{\text{AIF}}(o\t)$ in Eq.~\eqref{Eq:EFE_def_prod} can be rewritten as:
\begin{eqnarray}
 {G}_{\text{AIF}}(o\t)  &=& \mathbb{E}_{q({\s}_{t:T}, {o}_{t+1:T}, \tilde{a}|o\t)} \bigg[ \sum_{\tau =t}^{T-1}  \text{log} \frac{q(\tilde{a}) P(\s_{\tau+1}|\s\ta,a\ta)}{\tilde{P}({o}_{\tau+1}|s\ta, a\ta)q({\s}_{\tau+1}|o_{\tau+1}) } \bigg]. \label{Eq:EFE_def}
\end{eqnarray}
Focusing on deriving the optimal variational distribution $q^*(\tilde{a})$~\cite{millidge2021whence,mazzaglia2022free} that minimizes ${G}_{\text{AIF}}(o\t)$, it has been shown that $q^*(\tilde{a})=\sigma \left( -{G}_{\text{AIF}}^{(\tilde{a})} (o\t) \right)$, where $\sigma$ is the Softmax function and ${G}_{\text{AIF}}^{(\tilde{a}\t)}(o\t)$ is the EFE for a fixed plan $\tilde{a}$, defined as:
\begin{eqnarray}
\!\!\!\!\!\!\!\!\!\!\!\!\! {G}_{\text{AIF}}^{(\tilde{a})}(o\t)  &\!\!\!\!=\!\!\!\!&   \mathbb{E}_{  q({\s}_{t:T}, {o}_{t+1:T}|o\t,\tilde{a} )}  \bigg[ \sum_{\tau =t}^{T-1} \text{log} \frac{ P(\s_{\tau+1}|\s\ta,a\ta)}{ q({\s}_{\tau+1}|o_{\tau+1})} - \text{log} \tilde{P}({o}_{\tau+1}|s\ta, a\ta)  \bigg].  \label{Eq:EFE_simple}
\end{eqnarray}
As $q^*(\tilde{a})=\sigma \left( -{G}_{\text{AIF}}^{(\tilde{a})} (o\t) \right)$, the most probable plan, denoted as $\tilde{a}^*$, corresponds to the one that minimizes ${G}_{\text{AIF}}^{(\tilde{a})}(o\t)$. Thus, $\tilde{a}^*$ is referred to as the optimal plan and can be found as $\tilde{a}^*=\arg\min_{\tilde{a}}{G}_{\text{AIF}}^{(\tilde{a})}(o\t)$. AIF achieves this optimal plan by running its biased generative model forward from the current time $t$ to the time horizon $T$, generating hypothetical future states and observations for all possible plans $\tilde{a}$. It then selects the plan $\tilde{a}^*$ that results in the minimum value of ${G}_{\text{AIF}}^{(\tilde{a})}(o\t)$.

\textbf{Note 4.2:} It should be noted that AIF~\cite{friston2017active} chooses its actions according to the optimal plan $\tilde{a}^*$. This is in contrast to the optimal state-action policy $\pi^*$ in RL, which maps states to probability distributions over the action space.

According to the complete class theorem in AIF~\cite{friston2012active}, $\mathbb{E}_{P(r\ta|\s\ta,a\ta)} [r\ta]$ can be encoded as $\text{log} \tilde{P}(o_{\tau+1}|s\ta, a\ta) $. 
Therefore, ${G}_{\text{AIF}}^{(\tilde{a})}(o\t)$ in Eq.~\eqref{Eq:EFE_simple} can be rewritten as 
%
\begin{eqnarray}
\!\!\!\!\!\!\!\!\!\!\!\!\!\!\!\! {G}_{\text{AIF}}^{(\tilde{a})}(o\t)  &\!\!\!\!\!=\!\!\!\!\!&   - \underbrace{\mathbb{E}_{\prod_{\tau =t}^{T-1} q({s}\ta|o\ta) \tilde{P}({o_{\tau+1}}|s\ta, a\ta) P(r\ta|s\ta,a\ta)  } \left[ \sum_{\tau =t}^{T-1}  r\ta \right]}_{\text{Expected long-term  reward}}     \nonumber \\
&&  \!\!\!\!\!\!\!\! - \underbrace{  \mathbb{E}_{ \prod_{\tau =t}^{T-1} q(\s\ta|o\ta)  \tilde{P}({o_{\tau+1}}| s\ta, a\ta)} \left[ \sum_{\tau =t}^{T-1} D_{KL} [q(.|o_{\tau+1}), P(.|\s\ta,{a}\ta) \right] .}_{\text{Expected information gain}} \label{Eq:epistemic}  
\end{eqnarray}
The first term in Eq.~\eqref{Eq:epistemic} represents the expected long-term (extrinsic) reward, which is known as the goal-directed term in AIF~\cite{millidge2021whence}. The second term, KL-divergence between the variation posterior $q(.|o\ta)$ and the transition model $P(.|\s\ta,a\ta)$, represents the expected information gain, referred to as the epistemic (intrinsic) value, which measures the amount of information acquired by visiting a particular state. By following the optimal plan $\tilde{a}^*$ that minimizes ${G}_{\text{AIF}}^{(\tilde{a})}(o\t)$, the agent indeed maximizes the expected long-term extrinsic reward (exploitation) while reducing its uncertainty and maximizing the expected information gain (exploration) about the hidden states.
%

In this section, we have explored the strengths and weaknesses of RL and AIF algorithms. RL excels in fully observable problems with high-dimensional continuous spaces, utilizing policy iteration and function approximations like DNNs. However, it struggles with partial observability and relies heavily on well-defined extrinsic reward signals. As a result, RL can face difficulties when the extrinsic reward function is absent or provides sparse or zero rewards from the environment.
In contrast, AIF addresses partial observability by integrating belief state inference and information-seeking exploratory behavior, which is often lacking in traditional RL algorithms~\cite{haarnoja2018soft, okuyama2018autonomous}. The information gain term allows AIF agents to learn even without explicitly defined extrinsic reward functions, as any plan inherently possesses intrinsic value. This aspect of AIF is particularly valuable when designing appropriate reward functions is challenging. However, the computational burden of considering all possible plans limits many AIF methods to discrete spaces or finite horizon POMDPs.
Combining RL and AIF can leverage their respective strengths, offering more effective decision-making in POMDPs. The following section introduces a unified framework that merges RL and AIF, addressing partially observable problems with infinite time horizons and continuous spaces.
\section{Unified inference integrating AIF and RL in continuous space POMDPs} \label{sec:unified_inference}
 In this section, we present a unified inference framework that combines the strengths of AIF and RL. This framework enables an agent to learn an optimal belief state-action policy in an infinite time horizon POMDP with continuous spaces, while promoting exploration through information-seeking.
The proposed unified inference approach introduces a unified objective function that formulates the action selection problem and establishes the optimality criteria for the policy. This unified objective function combines the objective functions of both the AIF and RL frameworks within the context of a POMDP with continuous state and action spaces.
Next, we show that the unified objective function obeys the so-called unified Bellman equation and the unified Bellman optimality equation. These equations generalize the standard Bellman equation and Bellman optimality equation from MDPs (i.e., Eqs.~\eqref{Eq:Bellman}, \eqref{Eq:Bellman_Q}, \eqref{Eq:Bellman_opt}, and \eqref{Eq:Bellman_opt_Q})  to the broader class of POMDPs. By utilizing these unified Bellman equations, we derive a unified policy iteration framework that iteratively optimizes the proposed unified objective function.
\\
The proposed unified policy iteration serves two main purposes. Firstly, it extends the guarantees of policy iteration from MDPs~\cite{puterman2014markov} to POMDPs, allowing us to apply insights from MDP-based RL algorithms, such as actor-critic algorithms, in the context of POMDP-based AIF. This extension enables the generalization of these MDP-based algorithms to POMDPs and facilitates computationally feasible action selection through AIF in problems with continuous spaces and an infinite time horizon.
Secondly, this approach enables us to leverage recent breakthroughs in RL methods that rely on extrinsic rewards and apply them to tasks without pre-specified reward functions. By incorporating this capability, we can address a broader range of problem domains beyond those limited to explicit reward-based learning.

Prior to detailing the unified objective function, we must outline assumptions regarding our decision-making model. 
To establish the foundational formulation of the performance criterion in a POMDP with continuous state, action, and observation spaces, it is necessary to consider specific regularity assumptions outlined by~\citeA{puterman2014markov}, known as the POMDP regularity. 
\begin{assumption} {(POMDP regularity)} \label{ASS:regula}
\begin{itemize}
\item [(A1)] The state space $\mS$ is a compact set in $\mathbb{R}^{D_{\mS}}$.
\item [(A2)] The action space $\mA$ is a compact set in $\mathbb{R} ^{D_{\mA}}$.
\item [(A3)] The observation space $\mO$ is a compact set in $\mathbb{R}^{D_{\mO}}$.
\item [(A4)]  The extrinsic reward space $\mR$ is a compact set in $\mathbb{R}$
\item [(A5)]  The transition function $\Omega$, observation function $\Theta$, and reward function $U$ are Lipschitz continuous.
 \end{itemize}
 \end{assumption}
{ Appendix~\ref{app:regularity} provides detailed explanations of each regularity assumption.} These conditions enable us to focus on deriving an objective function as an optimality criterion for an optimal policy. However, as discussed in Sub-section~\ref{sec:policy}, to mitigate the high memory requirement associated with history-dependent policies, we adopt a belief state-action policy $\pi(a_t|b_t)$, requiring the agent to infer the belief state $b_t$ before selecting an action at time step $t$.
In line with common practices in RL and AIF approaches applied to POMDPs~\cite{von2001concerning, Hafner2020Dream, millidge2021whence, han2020variational}, we assume that the agent performs perceptual inference and learning prior to action selection at time step $t$ as follows:
\begin{assumption} {(Perceptual inference and learning)} \label{ASS:percep}\\
Prior to action selection at time step $t$, given the current observation $o_t$, the agent performs inference by approximating its belief state $b_t$ with a variational posterior distribution.  This perceptual inference is performed concurrently with the perpetual learning of the generative model through the minimization of the VFE (ELBO in VAEs~\cite{kingma2013auto}). 
 \end{assumption}
Appendix~\ref{sub-sec:condionined_perc} offers further details on the learning process of variational inference and the generative model.
\subsection{Problem formulation: Unified objective function for AIF and RL in POMDPs}  \label{sec:problem}
Given the regularity assumption in Assumption~\ref{ASS:regula} and the Perceptual Inference and Learning assumption in Assumption~\ref{ASS:percep}, our attention now turns to defining a performance criterion, or objective function, for learning the optimal belief state-action policy.

In POMDPs, where agents lack complete knowledge of true environmental states, devising a performance criterion presents a challenge~\cite{nian2020dcrac, krishnamurthy2015structural, chatterjee2016decidable}. The RL literature concerning POMDPs often focuses on maximizing the expected long-term extrinsic reward~\cite{haklidir2021guided, montufar2015geometry} as the optimality criterion. The EFE in AIF, as expressed in Eq.~\eqref{Eq:epistemic},  inherently balances information-seeking exploration and reward maximization in POMDPs. Hence, if we can extend the EFE to a stochastic belief state-action policy $\pi$ and reduce the computational requirements for learning the optimal policy, the extended EFE can serve as a suitable objective function for our infinite time horizon POMDP with continuous state, action, and observation spaces.
\\
By extending the EFE to a stochastic belief state-action policy $\pi$, action $a_t$ is sampled from policy $\pi(a_t|b_t)$\footnote{In the context of our proposed unified inference framework, when we refer to the term "policy" denoted as $\pi$, we mean a Markovian belief state-action policy.} instead of being selected from a predetermined plan $\tilde{a}$. The main idea is to learn a mapping from belief states to actions without explicitly searching through all possible plans. Thus, our goal becomes finding an optimal stochastic belief state-action policy $\pi^*$ that minimizes ${G}^{(\pi)}_{\text{Unified}} (b_t)$, which represents the EFE corresponding to $\pi(a_t|b_t)$, i.e., $\pi^* = \arg \min_{\pi} {G}^{(\pi)}_{\text{Unified}} (b_t)$.
Learning the optimal belief state-action policy is computationally more efficient than finding the optimal plan in continuous spaces, facilitating quicker decision-making. Moreover, by selecting actions individually from the policy at each time step instead of a pre-determined plan, the agent can update its policy iteratively based on new observations over time, allowing it to adapt to environmental dynamics, avoid repeating errors, and generalize better to new scenarios.
\\
Theorem~\ref{the:EFE_stoch} states the extended EFE for the belief state-action policy $\pi$.
\begin{theorem} \label{the:EFE_stoch}
Let $\pi$ be a stochastic belief state-action policy selecting action $a\ta$ according to $\pi(a\ta|b\ta)$ for $\tau=\{t, t+1, t+2, ... \}$ in a POMDP satisfying Assumption~\ref{ASS:regula}. ${G}^{(\pi)}_{\text{Unified}}(b\t)$, the EFE corresponding to the policy $\pi$, can be achieved as
\begin{eqnarray}
\!\!\!\!\!\!\!\!\! {G}^{(\pi)}_{\text{Unified}} (b\t)&\!\!=\!\!& \mathbb{E}_{ \prod_{\tau =t}^{\infty} P(b_{\tau+1}| o_{\tau+1}, b_{\tau}, a_{\tau}) \tilde{P}(o_{\tau+1}|a\ta, b\ta) \pi(a\ta|b\ta) q(s_{\tau+1}| o_{\tau+1}, b_{\tau}) } \big[ \sum_{\tau =t}^{\infty} \text{log} \pi(a_{\tau}| {b_{\tau}})    \nonumber  \\
\!\!\!\!\!\!\!\!\!  &\!\!-\!\!& \text{log} \tilde{P}({o}_{\tau+1}|a\ta, b\ta) + \text{log} \frac {P({\s_{\tau+1}}|a\ta, b_{\tau}) }{q({\s_{\tau+1}}|o_{\tau+1}, b_{\tau})}  \big]. \label{Eq:Beleif_EFE_beleif}  
\end{eqnarray} 
\end{theorem}
\begin{proof}
See Appendix~\ref{APP:EFE_stoch}.
\end{proof} \vspace{.1in} \\
The expression $P(b_{\tau+1}|s_{\tau+1}, o_{\tau+1}, b_{\tau}, a_{\tau})$ in Eq.~\eqref{Eq:Beleif_EFE_beleif} represents the agent's model of the belief state forward distribution, as defined in Sub-section~\ref{sub-sec:inference}. ${P}(s_{\tau+1}|a_{\tau}, b_{\tau})$ represents the belief state-conditioned transition model $b_{\tau}$.
Furthermore, $q({\s_{\tau+1}}|o_{\tau+1}, b_{\tau})$ and $\tilde{P}(o_{\tau+1}|b\ta, a\ta)$ denotes the variational posterior $q({\s_{\tau+1}}|o_{\tau+1})$ and the prior preference $\tilde{P}(o_{\tau+1}|s\ta, a\ta)$ in AIF conditioned on the belief state $b_{\tau}$.\footnote{ As the belief state $b\ta$ encapsulates the agent's knowledge about the underlying state $s\ta$ based on all available observations and actions, the direct conditioning on $s_{\tau}$ can be omitted when conditioning on $b_{\tau}$}
Thus, Eq.~\eqref{Eq:Beleif_EFE_beleif} computes the EFE, where the agent possesses a biased generative model over the sequence $(s\nt,  ..., b\ta,s_{\tau+1}, o_{\tau+1},  b_{\tau+1}, ...)$ given $b\t$. This sequence involves $o_{\tau+1} \sim \tilde{P}(o_{\tau+1}|a\ta,b\ta)$, $s_{\tau+1} \sim q({\s_{\tau+1}}|o_{\tau+1}, b_{\tau})$, and $b_{\tau+1} \sim P(b_{\tau+1}|  s_{\tau+1},o_{\tau+1}, b_{\tau}, a_{\tau})$. We can regard this biased generative model over this sequence as an extension of the biased generative model of AIF defined in Eq.~\eqref{Eq:AIF_GM} to encompass a stochastic belief state-action policy $\pi$.  Hence, we denote this biased generative model as $\tilde{P}({s}_{t+1:\infty}, o_{t+1:\infty}, b_{t+1:\infty}|b\t, {a}_{t:\infty} )$ and refer to it as the biased belief state generative model, which is formally defined as follows:
\begin{definition}{(Biased belief state generative model)} \label{def:Beleif_biased_GM}\\
Given the prior preference $\tilde{P}(o_{\tau+1}|b\ta,a\ta)$ for $\tau=\{t, t+1, \ldots\}$, the biased belief state generative model is $\tilde{P}({s}_{t+1:\infty}, o_{t+1:\infty}, b_{t+1:\infty}|b\t, {a}_{t:\infty})$, which factorizes as:
\begin{eqnarray}
 \lefteqn{ \!\!\!\!\!\!\!\!\!\!\!\!\!\! \tilde{P}({s}_{t+1:\infty}, o_{t+1:\infty}, b_{t+1:\infty}|b\t, {a}_{t:\infty}  ) =  \prod_{\tau =t}^{\infty} \tilde{P}({s}_{\tau+1}, o_{\tau+1}, b_{\tau+1}|b\ta, {a}_{\tau}  ) } \nonumber \\
&& \,\,\,\,\,\,\,\,\,\,\,\,\,\,\,\,\,\, =  \prod_{\tau =t}^{\infty} \tilde{P}(o_{\tau+1}|b\ta, a\ta) q({\s_{\tau+1}}|o_{\tau+1}, b_{\tau}) P(b_{\tau+1}| s_{\tau+1}, o_{\tau+1}, b_{\tau}, a_{\tau}).  \label{Eq:AIF_beleif_GM}
\end{eqnarray} 
\end{definition} 
The biased belief state generative model approximates ${P}({s}_{t+1:\infty}, o_{t+1:\infty}, b_{t+1:\infty}|b\t, {a}_{t:\infty})$, which extends the generative model of a POMDP defined in Eq.~\eqref{Eq:final_GM} to account for a stochastic belief state-action policy. Hence, we refer to ${P}({s}_{t+1:\infty}, o_{t+1:\infty}, b_{t+1:\infty}|b\t, {a}_{t:\infty})$ as the belief state generative model, which can be factorized as:
\begin{align}
 \lefteqn{ \!\!\!\!\!\!\!\!\!\!\!\!\!\!\!\!\!\!  {P}({s}_{t+1:\infty}, o_{t+1:\infty}, b_{t+1:\infty}|b\t, {a}_{t:\infty} \! )= \prod_{\tau =t}^{\infty} {P}({s}_{\tau+1}, o_{\tau+1}, b_{\tau+1}|b\ta, {a}_{\tau}  )  } \nonumber \\
&& =  \prod_{\tau =t}^{\infty}\! {P}(s_{\tau+1}|a_{\tau}, \! b_{\tau}\!) {P}(o_{\tau+1}|s_{\tau+1},\! b_{\tau}\!) P(b_{\tau+1}| s_{\tau+1}, o_{\tau+1}, b_{\tau}, a_{\tau}\!),
\label{Eq:beleif_GM}
\end{align} 
where ${P}(o_{\tau+1}|s_{\tau+1},b_{\tau})$ represents the belief state-conditioned likelihood model.

To accommodate the reward-maximizing RL objective, following the complete class theorem in AIF~\cite{friston2012active}, we employ the reparameterization $\text{log} \tilde{P}(o_{\tau+1}|a\ta,b\ta)=\mathbb{E}_{P(r\ta| b\ta,a\ta)} [r\ta]$, where $P(r\ta| b\ta,a\ta)$ represents the belief state-conditioned reward model.
Additionally, to ensure that the value of $G^{(\pi)}_{\text{Unified}}(b\t)$ in Eq.~\eqref{Eq:Beleif_EFE_beleif} remains bounded, we use the discount factor $\gamma \in [0,1)$, following the convention in RL literature. We denote the modified version of $G^{(\pi)}_{\text{Unified}}(b\t)$ as ${G}^{(\pi)} (b\t)$, which can be expressed as follows:
 \begin{eqnarray}
\!\!\!\! {G}^{(\pi)} ( b\t ) &\!\!\!\!\!=\!\!\!\!& \mathbb{E}_{ \prod_{\tau =t}^{\infty} P(b_{\tau+1}|\s_{\tau+1}, o_{\tau+1}, b_{\tau}, a_{\tau}) \tilde{P}(o_{\tau+1}|a\ta, b\ta) \pi(a\ta|b\ta) q(s_{\tau+1}| o_{\tau+1}, b_{\tau})  } \bigg[ \! \sum_{\tau =t}^{\infty} \! \gamma^{\tau-t} \! \bigg(\! \text{log} \pi(a_{\tau}| {b_{\tau}}\!)   \nonumber     \\
\!\!\!\!\!\!\!\!\!  &\!\!\!\!-\!\!\!\!& \text{log} \tilde{P}({o}_{\tau+1}|b\ta, a\ta) + \text{log} \frac {P({\s_{\tau+1}}|a\ta,b\ta) }{q({\s_{\tau+1}}|o_{\tau+1},b\ta)}  \bigg) \bigg]  \label{Eq:EFE_reward} \\
\!\!\!\!\!\!\!\!\! &\!\!\!\!=\!\!\!\!&  \underbrace{ \mathbb{E}_{ \prod_{\tau =t}^{\infty} {P}(o_{\tau+1}, b_{\tau+1}, s_{\tau+1}| b\ta, a\ta)} \big[ \sum_{\tau =t}^{\infty} - \gamma^{\tau-t} H(\pi(.| {b_{\tau}})) \big]}_{\text{Expected entropy}}   \nonumber  \\
\!\!\!\!\!\!\!\!\!  &\!\!\!\!-\!\!\!\!& \underbrace{ \mathbb{E}_{  \prod_{\tau =t}^{\infty} \pi(a\ta|b\ta) {P}(s_{\tau+1},  o_{\tau+1}, b_{\tau+1}|  b\ta, a\ta) P(r\ta|b\ta, a\ta) } \big[ \sum_{\tau =t}^{\infty} \gamma^{\tau-t} r\ta \big]}_{\text{Expected long-term  extrinsic reward} } \nonumber   \\
\!\!\!\!\!\!\!\!\!  &\!\!\!\!-\!\!\!\!& \underbrace{ \mathbb{E}_{\prod_{\tau =t}^{\infty } \pi(a\ta|b\ta) {P}( o_{\tau+1}, b_{\tau+1}|  b\ta, a\ta) } \big[\sum_{\tau =t}^{\infty} \gamma^{\tau-t} D_{KL} [q(.|o_{\tau+1}, b\ta),P(.|a\ta,b\ta)] \big] }_{\text{Expected information gain}}, \label{Eq:Beleif_EFE_reward2}
\end{eqnarray} 
where $H(\pi(.| {b_{\tau}}))$ is the Shannon entropy of $\pi(.| {b_{\tau}})$ and is calculated as $H(\pi(.| {b_{\tau}}))=-\mathbb{E}_{\pi(a\ta|b\ta)} [\text{log}\pi(a\ta|b\ta)]$. This entropy term provides a bonus for random exploration, encouraging the agent to distribute the probability across all possible actions as evenly as possible. This helps prevent the phenomenon of policy collapse~\cite{millidge2020deep}, where the policy quickly converges to a degenerate distribution (i.e., a deterministic policy).
The second term in Eq.~\eqref{Eq:Beleif_EFE_reward2} corresponds to the agent's objective of maximizing the expected long-term extrinsic reward, emphasizing exploitation. Including this term encourages the agent to focus on immediate rewards and exploit its current knowledge.
The third term represents the expected information gain. By incorporating this term into the objective function, the agent is motivated to actively seek out information to reduce uncertainty about the hidden state of the environment.
By combining both reward maximization and information-seeking, ${G}^{(\pi)} (b\t)$ serves as a unified objective function for both RL and AIF in continuous space POMDPs. The problem is thus formulated as finding the optimal belief state-action policy $\pi^*$ that minimizes ${G}^{(\pi)} (b\t)$.

Since ${G}^{(\pi)} (b_t)$ is a composite objective function that comprises multiple terms, including the expected long-term extrinsic reward, the entropy of the policy, and the information gain, we introduce the scaling factors $0 \leq \alpha < \infty$, $0 \leq \beta < \infty$, and $0 \leq \zeta < \infty$ in ${G}^{(\pi)} (b_t)$ to adjust the relative weights of these components in deriving the optimal policy. This flexibility allows us to prioritize specific objectives according to the task and desired behaviour. 
By incorporating these scaling factors, we rewrite ${G}^{(\pi)} (b_t)$ in Eq.~\eqref{Eq:EFE_reward} as follows:
\begin{eqnarray}
\!\!\!\!\!\! G^{(\pi)}(b\t) &=&  \mathbb{E}_{ \prod_{\tau =t}^{\infty} \pi(a_{\tau}|b_{\tau}) {P}(s_{\tau+1},  o_{\tau+1}, b_{\tau+1}|  b\ta, a\ta) P(r\ta|b\ta, a\ta)  } \big[ \sum_{\tau =t}^{\infty}  \gamma^{\tau-t} \big( \beta \, \text{log}\pi(a_{\tau}|b_{\tau})   \nonumber \\
 \!\!\!\!\!\! &-& \alpha \, r\ta   +  \zeta \, \text{log} \frac {P({\s_{\tau+1}}|b\ta,a\ta) }{q({\s_{\tau+1}}|o_{\tau+1}, b\ta)}  \big]. \label{Eq:Objective}
\end{eqnarray}

\textbf{Note 5.1:} Throughout this paper, when we refer to $G^{(\pi)}(b\t)$, we specifically mean the scaled version expressed in Eq.~\eqref{Eq:Objective}. Furthermore, as $G^{(\pi)}(b\t)$ represents the EFE as a function of a given belief state $b\t$, we refer to ${G}^{(\pi)}(b\t)$ as the belief state EFE to distinguish it from the EFE ${G}^{(\tilde{a})}(o\t)$ defined in AIF (Eq.~\eqref{Eq:EFE_simple}).

Given the the belief state EFE ${G}^{(\pi)}(b\t)$, the problem can now be formulated as 
\begin{eqnarray}
 \pi^* \in \arg\min_{\pi} G^{(\pi)}(b\t). \label{Eq:objec}
\end{eqnarray}
The belief state EFE corresponding to the optimal policy $\pi^*$ is referred to as the optimal belief state EFE and is denoted as $G^{*}(b\t)$. In other words, we have $G^{*}(b\t) := G^{(\pi^*)}(b\t) = \min_{\pi} G^{(\pi)}(b\t)$. It should be noted that $\pi^* \in \arg\min_{\pi} G^{(\pi)}(b\t)$ and not $\pi^* = \arg\min_{\pi} G^{(\pi)}(b\t)$ because the optimization problem in Eq.~\eqref{Eq:objec} may admit more than one optimal policy.
\\
Considering that the action space $\mathcal{A}$ in our problem is continuous, the space of possible policies $\pi$ representing probability distribution functions over the continuous action space $\mathcal{A}$ is vast and continuous.  Therefore, solving the optimization problem stated in Eq.~\eqref{Eq:objec} is a challenging and non-trivial task. In the next sub-section, we will demonstrate that $G^{(\pi)}(b\t)$ exhibits a recursive relationship, which offers a pathway for solving the optimization problem in Eq.~\eqref{Eq:objec}.
\subsection{Unified Bellman equation}
In this sub-section, we introduce a recursive solution for the problem in Eq.~\eqref{Eq:objec} in the spirit of the classical Bellman approach in MDPs~\cite{bellman1952theory}. This recursion allows us to solve the optimization problem presented in Eq.~\eqref{Eq:objec} by breaking it down into smaller sub-problems. \\
We begin by demonstrating that the belief state EFE $G^{(\pi)}(b_t)$, defined in Eq.~\eqref{Eq:Objective}, follows a Bellman-like recursion in the context of a POMDP. Since this recursion is defined for the belief state EFE $G^{(\pi)}(b_t)$, which serves as a unified objective function for both RL and AIF in the POMDP setting, we refer to this Bellman-like recursion as the unified Bellman equation.
\begin{proposition} {(Unified Bellman equation for ${G}^{(\pi)}(b\t)$)} \label{cor:q1} \\
The belief state EFE ${G}^{(\pi)}(b\t)$ defined in Eq.~\eqref{Eq:Objective} for a POMDP satisfying Assumption~\ref{ASS:regula} can be computed recursively starting from the belief state $b\t$ and following policy $\pi$, as follows:
\begin{eqnarray}
  {G}^{(\pi)}(b\t) &=& \mathbb{E}_{ \pi(a\t|b\t)  P(r\t|b\t,a\t)} \bigg[  \beta \, \text{log} \pi(a\t|b\t) - \alpha \, r\t   \label{Eq:unified_Bellman} \\
   &+& \mathbb{E}_{ {P}(b\nt,o\nt, \s\nt|b\t, a\t)} \big[ \zeta \,   \text{log} \frac{P(\s\nt|  b\t, a\t)}{ q({\s}_{t+1}|{o}_{t+1}, b\t)} + \gamma  {G}^{(\pi)}(b\nt) \big]  \bigg]. \nonumber
\end{eqnarray}
\end{proposition}
\begin{proof}
See Appendix~\ref{APP:q1}.
\end{proof}
 \vspace{.1in} \\
The unified Bellman equation establishes a recursive relationship between the belief state EFE at a given belief state $b\t$ and the belief state EFE at its successor belief state $b\nt$. This recursive relationship allows us to calculate the belief state EFE $G^{(\pi)}(b\t)$ at time instant $t$ by considering the belief state EFE for the remaining time horizon, starting from the next belief state $b\nt$.

To demonstrate how the unified Bellman equation in Proposition~\ref{cor:q1} facilitates the optimization problem in Eq.~\eqref{Eq:objec}, we present the following theorem, which proves that the optimal belief state EFE $G^{*}(b\t)$ can be computed recursively using the optimal belief state EFE at the successor belief state $b\nt$, i.e., $G^{*}(b\nt)$. We refer to this recursive computation of the optimal belief state EFE as the unified Bellman optimality equation.
\begin{theorem} {(Unified Bellman optimality equation for $G^*(b\t)$)} \\ \label{theo:optimal_policy}
The optimal belief state EFE $G^*(b\t)=\min_{\pi} G^{(\pi)}(b\t)$ in a POMDP satisfying Assumption~\ref{ASS:regula} can be calculated recursively starting from the belief state $b\t$ and following the optimal policy $\pi^*$ as follows:
\begin{eqnarray}
   {G}^{*}(b\t) &=& \min_{\pi(a\t|b\t)} \mathbb{E}_{ \pi(a\t|b\t)} \bigg[ \mathbb{E}_{  p(r\t|b\t,a\t)} \big[\beta \, \text{log} \pi(a\t|b\t) - \alpha \, r\t  \label{Eq:unified_optimality_Bellman} \\
 &+& \mathbb{E}_{ P(b\nt,o\nt, \s\nt| b\t,a\t)} [ \zeta \,   \text{log} \frac{P(\s\nt|  b\t, a\t)}{ q({\s}_{t+1}|{o}_{t+1}, b\t)} + \gamma  {G}^{*}(b\nt) ] \big]  \bigg]. \nonumber
\end{eqnarray}
\end{theorem}
\begin{proof}
See Appendix~\ref{APP:optimal_policy}.
\end{proof} \vspace{.1in}
\\
Using the unified Bellman optimality equation, the following corollary demonstrates how finding the optimal instant policy $\pi^*(a_t|b_t) \in \arg\min_{\pi(a_t|b_t)} G^*(b_t)$ for selecting the instant action $a_t$ based on the belief state $b_t$, involves finding an optimal policy $\tilde{\pi}^* \in \arg\min_{\tilde{\pi}} G^{(\tilde{\pi})}(b_{t+1})$ for the remaining actions, starting from the belief state $b_{t+1}$ resulting from the first action $a_t$.
\begin{corollary}  \label{corr:optimal_principle} 
The instant optimal policy $\pi^*(a\t|b\t) \in \arg\min_{\pi(a\t|b\t)} G^{*}(b\t)$ for choosing instant action $a\t$ can be achieved in terms of the optimal policy $\tilde{\pi}^* \in \arg\min_{\tilde{\pi}} G^{(\tilde{\pi})}(b\nt)$ with regard to the next belief state $b\nt$ resulting from $a\t$:
\begin{eqnarray}
   \pi^*(a\t|b\t) &\!\!\! \in \!\!\! &    \arg\min_{\pi(a\t|b\t)} \mathbb{E}_{ \pi(a\t|b\t)} \bigg[ \mathbb{E}_{  P(r\t|b\t,a\t)} [\beta \, \text{log} \pi(a\t|b\t) - \alpha \, r\t  ]  \nonumber   \\
  &\!\!\! + \!\!\!& \mathbb{E}_{ P(b\nt,o\nt, \s\nt| s\t,a\t)} \big[ \zeta \,   \text{log} \frac{P(\s\nt|  b\t, a\t)}{ q({\s}_{t+1}|{o}_{t+1}, b\t)} + \gamma \min_{\tilde{\pi}} {G}^{(\tilde{\pi})}(b\nt) \big]  \bigg]. \label{Eq:optimal_policy_arg}
\end{eqnarray}
\end{corollary}
\begin{proof}
See Appendix~\ref{APP:optimal_principle}.
\end{proof} 
\vspace{.1in} \\
Eq.~\eqref{Eq:optimal_policy_arg} demonstrates how the optimization problem $\pi^* \in \arg\min_{\pi} G^{*}(b\t)$ in Eq.~\eqref{Eq:objec} with an infinite time horizon, can be decomposed into two sub-problems with separate time horizons. The first sub-problem has a time horizon equal to $1$ and involves finding an instant optimal policy $\pi^*(a_t|b_t)$ based on the belief state $b_t$ as $\pi^*(a_t|b_t) \in \arg\min_{\pi(a_t|b_t)} G^*(b_t)$. The second sub-problem determines the optimal policy for the remaining time horizon from $t+1$ onwards based on the next belief state $b_{t+1}$ as $\tilde{\pi}^* \in \arg\min_{\tilde{\pi}} G^{(\tilde{\pi})}(b_{t+1})$. In other words, the optimal policy $\pi^* \in \arg\min_{\pi} G^{(\pi)}(b_t)$ can be decomposed as $\pi^*=(\pi^*(.|b_t), \tilde{\pi}^*)$, and it can be constructed by recursively combining the policies of the sub-problems with one-time horizons over time, namely $\pi^*(a\ta|b\ta)$ for $\tau=\{ t, t+1, ... \}$.  The crucial improvement achieved from Corollary~\ref{corr:optimal_principle} over current AIF methods is that choosing the action $a\t$ from $\pi^*(a\t|b\t)$ is now performed based on subsequent counterfactual action $a\nt$ from $\pi^*(a\nt|b\nt)$, as opposed to considering all future courses of actions.  This approach allows for more efficient decision-making by focusing on the immediate consequences of the selected action, rather than exploring all possible future trajectories.
\subsection{Existence of a unique optimal policy}
In Theorem~\ref{theo:optimal_policy}, we have shown that the optimal belief state EFE, $G^*(b\t)$, follows the unified Bellman optimality recursion. This recursion allows for the recursive combination and computation of $\pi^*(a\ta|b\ta)$ for $\tau= \{t, t+1, ... \}$, enabling the derivation of an optimal policy $\pi^* \in \arg\min_{\pi} G^{(\pi)}(b\t)$. In this sub-section, we aim to demonstrate the existence of a unique optimal policy $\pi^*(a\t|b\t)$ that satisfies Eq.~\eqref{Eq:optimal_policy_arg}, as well as provide a method for computing it.
\begin{theorem}{(Existence of a unique optimal policy)} \\ \label{the:instant_optimal}
Minimizing  the right-hand side of the unified Bellman optimality equation (i.e., Eq.\eqref{Eq:unified_optimality_Bellman}) with respect to $\pi(a\t|b\t)$ in a POMDP satisfying Assumption~\ref{ASS:regula} leads to the following unique minimizer: 
\begin{eqnarray}
\!\! \! \! \!\!\!\!\!  \!\!\! \!\!\pi^*(a\t|b\t) & \!\! \!\! \!=\! \! \!\! \!&  \arg\min_{\pi(a\t|b\t)} G^{*}(b\t) \nonumber  \\
\!\! \! \! \! \!\!\!\!\!\!\!\! &\!\! \!\!\! = \!\!\! \!\!&  \sigma  \!\! \left( \!\frac{ \mathbb{E}_{  P(r\t|b\t,a\t)} [ \alpha r\t  \!  - \!  \mathbb{E}_{ P(b\nt,o\nt, \s\nt| b\t,a\t)} \! \left[ \! \zeta \,   \text{log} \frac{P(\s\nt|  b\t, a\t)}{ q({\s}_{t+1}|{o}_{t+1},b\t)} \!  + \!  \gamma  {G}^{*}(b\nt) ] \right] }{\beta} \!\right)\!\!\!.
 \label{Eq:pi_instant}
\end{eqnarray}
\end{theorem}
\begin{proof}
See Appendix~\ref{APP:instant_optimal}.
\end{proof} \vspace{.1in} \\
 Theorem~\ref{the:instant_optimal} showed that $\pi^*(a\t|b\t)$ (the unique minimizer of Eq.~\eqref{Eq:unified_optimality_Bellman}) is a Softmax function of the following term:
 \begin{eqnarray}
\mathbb{E}_{  P(r\t|b\t,a\t)} \left[ \alpha r\t  \!  - \!  \mathbb{E}_{ P(b\nt,o\nt, \s\nt| b\t,a\t)} \big[ \zeta \,   \text{log} \frac{P(\s\nt|  b\t, a\t)}{ q({\s}_{t+1}|{o}_{t+1}, b\t)} \!  + \!  \gamma  {G}^{*}(b\nt) \big] \right],  \label{Eq:temp_action}
\end{eqnarray}
which is a function of a given belief state-action pair $(b\t, a\t)$. To gain intuition about this expression,  we rewrite it as follows
 \begin{eqnarray}
\lefteqn {\!\!\!\!\!\!\!   \mathbb{E}_{  P(r\t|b\t,a\t)} \left[ \alpha r\t  \!  - \!  \mathbb{E}_{ P(b\nt,o\nt, \s\nt| b\t,a\t)} \big[ \zeta \,   \text{log} \frac{P(\s\nt|  b\t, a\t)}{ q({\s}_{t+1}|{o}_{t+1}, b\t)} \!  + \!  \gamma  {G}^{*}(b\nt) \big] \right]  } \nonumber \\
&& \!\!\!\!\!\!\!\!\!\!\!\!\!\!\!      = \! - \underbrace{ \mathbb{E}_{ P(r\t|b\t,a\t)} \! \left[ -\alpha r\t  \!  + \!  \mathbb{E}_{ P(b\nt,o\nt, \s\nt| b\t,a\t)} \big[ \zeta \,   \text{log} \frac{P(\s\nt|  b\t, a\t)}{ q({\s}_{t+1}|{o}_{t+1}, b\t)} \!  + \!  \gamma  {G}^{*}(b\nt) \big] \right] }_{\mathcal{G}^{*}(b\t,a\t)}\!.\label{Eq:opt_action_EFE} 
\end{eqnarray}
We can now simplify and rewrite $\pi^*(a\t|b\t)$ in Eq.~\eqref{Eq:pi_instant} in terms of $\mathcal{G}^{*}(b\t,a\t)$ as follows:
\begin{eqnarray}
 \!\!\!\! \pi^*(a\t|b\t) &\!\!\!\!=\!\!\!\!&  \sigma \!  \left( \frac{ \mathbb{E}_{ P(r\t|b\t,a\t)} \left[\! \alpha r\t  \!  - \!  \mathbb{E}_{ P(b\nt,o\nt, \s\nt| b\t,a\t)} \big[ \zeta \,   \text{log} \frac{P(\s\nt|  b\t, a\t)}{ q({\s}_{t+1}|{o}_{t+1}, b\t)} \!  + \!  \gamma  {G}^{*}(b\nt) \big] \right] }{\beta} \right) \nonumber \\
\!\!\!\!\!\!\!\!\!\!\!\!\!\!\! &\!\!\!\!=\!\!\!\!&  \sigma \left( \frac{- \mathcal{G}^{*}(b\t,a\t)}{\beta} \right)\\ 
\!\!\!\!\!\!\!\!\!\!\!\!\!\!\! &\!\!\!\!=\!\!\!\!&  \frac{ \text{exp} \left( \frac{- \mathcal{G}^{*}(b\t,a\t)}{\beta} \right) }   { Z^{*}(b\t) }, \label{Eq:Z}
 \end{eqnarray}
 where $\text{exp}(.)$ is the exponential function, and $Z^{*}(b\t)=\int_{\mA} \text{exp} \left( - \frac{\mathcal{G}^{*}(b\t,a')}{\beta} \right)\, da'$. 
 We can control the stochasticity of $\pi^*(a\t|b\t)$ by adjusting the value of $\beta$. As $\beta \rightarrow 0$, $\pi^*(a_t|b_t)$ converges to a deterministic policy, while as $\beta \rightarrow \infty$, $\pi^*(a_t|b_t)$ approaches a uniform distribution over the action space $\mathcal{A}$.
\\
However, computing $Z^{*}(b\t)=\int_{\mA} \text{exp} \left( - \frac{\mathcal{G}^{*}(b\t,a')}{\beta} \right)\, da'$ in the denominator of the Softmax function in Eq.~\eqref{Eq:Z} is computationally challenging or intractable. This is because it involves the integration of the complex function $\mathcal{G}^{*}(b\t,a\t)$ over the continuous action space $\mA$. To address this issue, common practices in handling intractable integrals over continuous spaces~\cite{wright2006numerical} involve methods like discretization of the action space into a finite number of points or regions, or employing numerical integration techniques such as Monte Carlo integration, numerical quadrature, or adaptive integration algorithms. These methods aim to approximate the integral over the continuous space numerically.
However,  discretization and numerical integration methods can introduce errors and lead to a loss of information. Problem reformulation~\cite{boyd2004convex}, on the other hand, allows us to maintain the inherent continuity of the original problem, enabling a more accurate representation of the underlying optimization problem. By reformulating the problem and imposing constraints, we can optimize the policy within a set of stochastic belief state-action policies $\bar{\Pi}$ that are tractable over continuous action spaces. These constraints limit the search space of policies to a subset of feasible solutions $\bar{\Pi}$, ensuring that the optimization problem remains computationally tractable. Consequently, using the problem formulation technique, our objective function in Eq.~\eqref{Eq:objec} is transformed into finding an optimal stochastic policy $\bar{\pi}^* \in \bar{\Pi}$ that minimizes the belief state EFE corresponding to the constrained policy $\bar{\pi} \in \bar{\Pi}$, denoted as $G^{(\bar{\pi})}(b\t)$:
\begin{eqnarray}
 \bar{\pi}^* \in \arg\min_{\bar{\pi} \in \bar{\Pi}} G^{(\bar{\pi})}(b\t). \label{Eq:objec_constr}
\end{eqnarray}
The choice of the set $\bar{\Pi}$ depends on the specific problem and its constraints. It is driven by factors such as the complexity of the problem, the desired level of policy expressiveness, and computational tractability.
In some cases, a more complex policy may be necessary to capture the nuances and intricacies of the problem at hand. This could involve using flexible parametric distributions or even non-parametric representations for the policies. These complex policies allow for more expressive modelling of the belief state-action policies, accommodating a wide range of possible behaviors. On the other hand, in certain scenarios, a simpler distribution may be sufficient to effectively model the belief state-action policies. This could involve using parametric distributions with fewer parameters or selecting a specific functional form that is well-suited to the problem's characteristics. These simpler policies offer computational advantages, as they typically require fewer computational resources and can be easier to optimize.

Since the unified Bellman equation and unified Bellman optimality equation presented in Eqs.~\eqref{Eq:unified_Bellman} and \eqref{Eq:unified_optimality_Bellman} hold for the belief state EFE under any unconstrained policy $\pi$, these equations also hold for the belief state EFE corresponding to the constrained policy $\bar{\pi}$, i.e., $G^{(\bar{\pi})}(b\t)$, and the belief state EFE corresponding to the constrained optimal policy $\bar{\pi}^* \in \bar{\Pi}$, denoted as $G^{(\bar{\pi}^*)}(b\t) = \min_{\bar{\pi} \in \bar{\Pi}} G^{(\bar{\pi})}(b\t)$. Therefore, we can express the unified Bellman equation and the unified Bellman optimality equation for $G^{(\bar{\pi})}(b\t)$ and $G^{(\bar{\pi^*})}(b\t)$ as
\begin{eqnarray}
  {G}^{(\bar{\pi})}(b\t) &=& \mathbb{E}_{ \bar{\pi}(a\t|b\t)  P(r\t|b\t,a\t)} \bigg[  \beta \, \text{log} \bar{\pi}(a\t|b\t) - \alpha \, r\t   \label{Eq:unified_Bellman_mod}\\
   &+& \mathbb{E}_{ {P}(b\nt,o\nt, \s\nt|b\t, a\t)} \big[ \zeta \,   \text{log} \frac{P(\s\nt|  b\t, a\t)}{ q({\s}_{t+1}|{o}_{t+1}, b\t)} + \gamma  {G}^{(\bar{\pi})}(b\nt) \big]  \bigg]. \nonumber
\end{eqnarray}
and
\begin{eqnarray}
   {G}^{(\bar{\pi}^*)}(b\t) &=& \min_{\bar{\pi}(a\t|b\t) \in \bar{\Pi}} \mathbb{E}_{ \bar{\pi}(a\t|b\t)} \bigg[ \mathbb{E}_{ P(r\t|b\t,a\t)} \big[\beta \, \text{log} \bar{\pi}(a\t|b\t) - \alpha \, r\t  \label{Eq:unified_optimality_Bellman_mod} \\
 &+& \mathbb{E}_{ P(b\nt,o\nt, \s\nt| b\t,a\t)} [ \zeta \,   \text{log} \frac{P(\s\nt|  b\t, a\t)}{ q({\s}_{t+1}|{o}_{t+1}, b\t)} + \gamma  {G}^{(\bar{\pi}^*)}(b\nt) ] \big]  \bigg]. \nonumber
\end{eqnarray}
Using the unified Bellman optimality equation in Eq.~\eqref{Eq:unified_optimality_Bellman_mod}, the instant constrained  optimal policy $\bar{\pi}^*(a\t|b\t) \in \arg\min_{\bar{\pi}(a\t|b\t) \in \bar{\Pi}} {G}^{(\bar{\pi}^*)}(b\t)$  can be achieved as
\begin{eqnarray}
   \bar{\pi}^*(a\t|b\t) &\in&  \arg\min_{\bar{\pi}(a\t|b\t) \in \bar{\Pi}} {G}^{(\bar{\pi}^*)}(b\t) \label{Eq:inst_constraint_policy} \\
   &=& \arg\min_{\bar{\pi}(a\t|b\t) \in \bar{\Pi}} \mathbb{E}_{ \bar{\pi}(a\t|b\t)} \bigg[ \mathbb{E}_{ P(r\t|b\t,a\t)} [\beta \, \text{log} \bar{\pi}(a\t|b\t) - \alpha \, r\t  ]  \nonumber   \\
  & +& \mathbb{E}_{ P(b\nt,o\nt, \s\nt| b\t,a\t)} \big[ \zeta \,   \text{log} \frac{P(\s\nt|  b\t, a\t)}{ q({\s}_{t+1}|{o}_{t+1}, b\t)} + \gamma  {G}^{(\bar{\pi}^*)}(b\nt) \big]  \bigg]. 
\end{eqnarray}
Solving the constrained minimization problem in Eq.~\eqref{Eq:inst_constraint_policy} is non-trivial as it may not have a closed-form solution for $\bar{\pi}^*(a\t|b\t)$.  However, according to Theorem~\ref{the:instant_optimal}, we know that the global solution of the unconstrained problem $ \arg\min_{\bar{\pi}(a\t|b\t)} {G}^{(\bar{\pi}^*)}(b\t)$ is $\sigma\left(\frac{-\mathcal{G}^{(\bar{\pi}^*)}(b\t,a\t)}{\beta}\right)$, where $\mathcal{G}^{(\bar{\pi}^*)}(b\t,a\t)$ is defined as
 \begin{eqnarray}
\!\!\!  \mathcal{G}^{(\bar{\pi}^*)}(b\t,a\t) &\! = \!& \mathbb{E}_{ P(r\t|b\t,a\t)} \big[ \!  - \! \alpha r\t  \! + \!  \mathbb{E}_{ P(b\nt,o\nt, \s\nt| b\t,a\t)} [ \zeta   \text{log} \frac{P(\s\nt| b\t, a\t)}{ q({\s}_{t+1}|{o}_{t+1}, b\t)}  \nonumber \\
\!\!\!  \!\!\!  &\!  + \!&  \gamma  {G}^{(\bar{\pi}^*)}(b\nt) ] \big]. \label{Eq:opt_constr_action_EFE}
\end{eqnarray}
If we choose the policy set $\bar{\Pi}$ to be convex, like a set of Gaussian or T-distributions over the compact action space $\mathcal{A}$, the constrained optimization problem in Eq.~\eqref{Eq:inst_constraint_policy} for a POMDP satisfying Assumption~\ref{ASS:regula} meets  Slater's condition~\cite{boyd2004convex}.\footnote{For more details about Slater's condition, we refer the reader to~\cite{boyd2004convex}.} This implies the existence of a unique minimizer $\bar{\pi}^*(a\t|b\t) \in \bar{\Pi}$ that is obtained by projecting the solution of the unconstrained problem, $\sigma\left(\frac{-\mathcal{G}^{(\bar{\pi}^*)}(b\t,a\t)}{\beta}\right)$, onto the constrained set $\bar{\Pi}$. This projection operation involves finding the closest distribution in $\bar{\Pi}$ to $\pi^*(a\t|b\t)=\sigma \left( \frac{- \mathcal{G}^{(\bar{\pi}^*)}(b\t,a\t)}{\beta} \right)$, according to some distance metric or divergence.
\\
Consistent with standard practices in RL and AIF formulations~\cite{haarnoja2018soft, friston2017active}, we use information projection based on the KL-divergence to perform this mapping. Therefore, the unique solution of the constrained optimization problem in Eq.~\eqref{Eq:inst_constraint_policy} is:
\begin{eqnarray}
\!\!\!\!\!\!\!\!\!\!\!\!\!\!\! \bar{\pi}^{*}(.|b\t)= \arg\min_{{ \bar{\pi} \in \bar{\Pi} }} D_{KL} \left[\bar{\pi}(.|b\t), \sigma \left( \frac{- \mathcal{G}^{(\bar{\pi}^*)}(b\t,.)}{\beta} \right)\right]. \label{Eq:project}
\end{eqnarray}
As the intractable term $Z^{*}(b\t) $ is independent of action $a\t$, Eq.~\eqref{Eq:project} can be simplified as:
\begin{eqnarray}
\bar{\pi}^{*}(.|b\t)  &=& \arg\min_{{ \bar{\pi} \in \bar{\Pi} }} D_{KL} \left[\bar{\pi}(.|b\t), \sigma \left( \frac{- \mathcal{G}^{(\bar{\pi}^*)}(b\t,.)}{\beta} \right)\right]  \\
&=& \arg\min_{{{\bar{\pi} \in \bar{\Pi}} }}  \mathbb{E}_{\bar{\pi}(a\t|b\t)}  \left[ \text{log} \bar{\pi}(a\t|b\t)+ \frac{\mathcal{G}^{(\bar{\pi}^*)}(b\t,a\t)}{\beta} \right]. \label{Eq:project_simple}
\end{eqnarray}
Eq.~\eqref{Eq:project_simple} demonstrates that by selecting actions from $\bar{\pi}^{*}(.|b\t)$ instead of $\pi^*(.|b\t)=\sigma \left( \frac{- \mathcal{G}^{({\pi}^*)}(b\t,\cdot)}{\beta} \right)$, the intractable term $Z^{*}(b\t)$ is no longer involved in the action selection process.

\textbf{Note 5.2:} Throughout the rest of this paper, any policy $\pi$ and any optimal policy ${\pi}^{*}$ mentioned are considered unconstrained. Conversely, any policy $\bar{\pi}$ and optimal policy $\bar{\pi}^{*}$ mentioned are assumed to belong to the constrained policy space $\bar{\Pi}$.
\subsection{Existence of a unique optimal belief state-action EFE}
In the previous sub-section, assuming the existence of $\mathcal{G}^{(\bar{\pi}^*)}(b\t,a\t)$, we have demonstrated that the solution to the minimization problem in Eq.~\eqref{Eq:inst_constraint_policy} is unique and can be computed as follows:
\begin{eqnarray}
\!\!\!\!\!\!\!\!\!\!\!\!\!\!\! \bar{\pi}^{*}(.|b\t)= \arg\min_{{ \bar{\pi} \in \bar{\Pi} }} D_{KL} \left[\bar{\pi}(.|b\t), \sigma \left( \frac{- \mathcal{G}^{(\bar{\pi}^*)}(b\t,.)}{\beta} \right)\right], \label{Eq:project2}
\end{eqnarray}
 Our objective in this sub-section is to demonstrate the existence of a unique value for $\mathcal{G}^{(\bar{\pi}^*)}(b\t,a\t)$ that leads to deriving the optimal policy $\bar{\pi}^{*}$ via Eq.~\eqref{Eq:project2}. This demonstration involves three steps.
 \\ 
 First, we show through Lemma~\ref{lem:constr_action_EFE}  that $\mathcal{G}^{(\bar{\pi}^*)}(b\t,a\t)$ in Eq.~\eqref{Eq:opt_constr_action_EFE} can be expressed as $\mathcal{G}^{(\bar{\pi}^*)}(b\t,a\t)=\min_{\bar{\pi}  \in \bar{\Pi}} \mathcal{G}^{(\bar{\pi})}(b\t,a\t)$, where
\begin{eqnarray}
 \mathcal{G}^{(\bar{\pi})}(b\t,a\t) &=& \! \mathbb{E}_{ P(r\t|b\t,a\t)} \bigg[\!\! - \! \alpha r\t  \!\!  + \!  \mathbb{E}_{ P(b\nt,o\nt, \s\nt| b\t,a\t)} \big[ \zeta   \text{log} \frac{P(\s\nt|  b\t, a\t)}{ q({\s}_{t+1}|{o}_{t+1}, b\t)} \!  \nonumber\\
& +&   \gamma  {G}^{(\bar{\pi})}(b\nt) \big] \bigg]. 
\end{eqnarray}
\begin{lemma}\label{lem:constr_action_EFE}
$\mathcal{G}^{(\bar{\pi}^*)}(b\t,a\t)$ in Eq.~\eqref{Eq:opt_constr_action_EFE} can be rewritten as $\mathcal{G}^{(\bar{\pi}^*)}(b\t,a\t)=\min_{\bar{\pi}  \in \bar{\Pi}} \mathcal{G}^{(\bar{\pi})}(b\t,a\t)$, where $\mathcal{G}^{(\bar{\pi})}(b\t,a\t)$ is given by
\begin{eqnarray}
 \mathcal{G}^{(\bar{\pi})}(b\t,a\t) &=& \! \mathbb{E}_{ P(r\t|b\t,a\t)} \bigg[\!\! - \! \alpha r\t  \!\!  + \!  \mathbb{E}_{ P(b\nt,o\nt, \s\nt| b\t,a\t)} \big[ \zeta   \text{log} \frac{P(\s\nt|  b\t, a\t)}{ q({\s}_{t+1}|{o}_{t+1}, b\t)} \!  \nonumber\\
& +&   \gamma  {G}^{(\bar{\pi})}(b\nt) \big] \bigg]. \label{Eq:g_G}
\end{eqnarray}
\end{lemma}
\begin{proof}
See Appendix~\ref{APP:constr_action_EFE}.
\end{proof}   \vspace{.1in}  \\
 While ${G}^{(\bar{\pi})}(b\t)$ represents the belief state EFE for a given belief state $b\t$ following a given policy $\bar{\pi}$,  $\mathcal{G}^{(\bar{\pi})}(b\t,a\t)$ in Eq.~\eqref{Eq:g_G} quantifies the belief state EFE ${G}^{(\bar{\pi})}(b\t)$ conditioned on a given action $a\t \in \mA$ followed by the policy $\bar{\pi}$ afterwards. Henceforth, we refer to $\mathcal{G}^{(\bar{\pi})}(b\t,a\t)$ and $\mathcal{G}^{(\bar{\pi}^*)}(b\t,a\t)=\min_{\bar{\pi} \in \bar{\Pi}} \mathcal{G}^{(\bar{\pi})}(b\t,a\t)$ as the belief state-action EFE and the optimal belief state-action EFE, respectively. 
\\
Second, by utilizing the result of Lemma~\ref{lem:constr_action_EFE}, we establish the following proposition, which introduces the unified Bellman equation for the belief state-action $\mathcal{G}^{(\bar{\pi})}(b\t,a\t)$.
\begin{proposition} {(Unified Bellman equation for $\mathcal{G}^{({\pi})}(b\t,a\t)$)} \\ \label{pro:unified_Bellman_action}
The belief state-action $\mathcal{G}^{(\bar{\pi})}(b\t,a\t)$  in a POMDP satisfying Assumption~\ref{ASS:regula} can be calculated recursively starting from the belief state $b\t$ and action $a\t$ and then following the policy $\bar{\pi} \in \bar{\Pi}$ as follows:
 \begin{eqnarray}
\!\!\!\!\!\!\!\! \!\!\! \!\!\!\!\! \mathcal{G}^{(\bar{\pi})}(b\t,a\t)  &\!\!=\!\!&  \mathbb{E}_{  P(r\t|b\t,a\t)} \bigg[  -\alpha \, r\t+\mathbb{E}_{   P(b\nt,o\nt, \s\nt| b\t,a\t) }  \big[ \zeta \,   \text{log} \frac{P(\s\nt|  b\t, a\t)}{ q({\s}_{t+1}|{o}_{t+1}, b\t)}   \nonumber \\
&\!\!+\!\!&   \gamma \mathbb{E}_{\bar{\pi}(a\nt|b\nt)}[ \beta \, \text{log} \bar{\pi}(a\nt|b\nt) + \mathcal{G}^{(\bar{\pi})}(b\nt, a\nt)]  \big] \bigg]. \label{Eq:unified_Bellman_action}
\end{eqnarray}
\end{proposition} 
\begin{proof}
See Appendix~\ref{APP:unified_Bellman_action}.
\end{proof}   \vspace{.1in}  \\
Third, by utilizing the result of Proposition~\ref{pro:unified_Bellman_action}, we establish through Theorem~\ref{lem:policy_evaluation} that the belief state-action EFE $\mathcal{G}^{(\bar{\pi})}(b\t,a\t)$, for a given policy $\bar{\pi} \in \bar{\Pi}$, exists and is a unique fixed point of the operator $T^{unified}_{\bar{\pi}} \mathcal{G}(b\t,a\t)$.  This operator, which we refer to as the unified Bellman operator, is defined as follows:
 \begin{eqnarray}
\!\!\!T^{unified}_{\bar{\pi}} \mathcal{G}(b\t,a\t)  &\!\!\!\!\!  := \!\!\!\!\! &  -\mathbb{E}_{  P(r\t|b\t,a\t)} \bigg[  -\alpha \, r\t+\mathbb{E}_{   P(b\nt,o\nt, \s\nt| b\t,a\t) }  \big[ \zeta \,   \text{log} \frac{P(\s\nt|  b\t, a\t)}{ q({\s}_{t+1}|{o}_{t+1}, b\t)}   \nonumber \\
&\!\!\! +\!\!\! &   \gamma \mathbb{E}_{\bar{\pi}(a\nt|b\nt)}[ \beta \, \text{log} \bar{\pi}(a\nt|b\nt) + \mathcal{G}(b\nt, a\nt)]  \big] \bigg].  \label{Eq:unified_Bellman_operator}
\end{eqnarray}
\begin{theorem} {(Existence of a unique belief state-action $\mathcal{G}^{(\bar{\pi})}(b\t,a\t)$ for a given $\bar{\pi}$)} \\ \label{lem:policy_evaluation}
The belief state-action EFE $\mathcal{G^{(\bar{\pi})}}(b\t,a\t)$ for all $b\t \in \mB$ and $a\t \in \mA$ in a POMDP satisfying Assumption~\ref{ASS:regula}  exists and is a unique fixed point $\mathcal{G^{(\bar{\pi})}}(b\t,a\t)=T^{unified}_{\bar{\pi}} \mathcal{G}^{(\bar{\pi})}(b\t,a\t)$ of the unified Bellman operator $T^{unified}_{\bar{\pi}}$ in Eq.~\eqref{Eq:unified_Bellman_operator}.
\end{theorem}
\begin{proof}
See Appendix~\ref{APP:policy_evaluation}.
\end{proof} \vspace{.1in} 
\\
Finally, using Theorem~\ref{lem:policy_evaluation}, Corollary~\ref{theo:action_optimal_policy} shows that that the optimal belief state-action EFE $\mathcal{G}^{(\bar{\pi}^*)}(b\t,a\t)= \min_{\bar{\pi} \in \bar{\Pi}} \mathcal{G^{(\bar{\pi})}}(b\t,a\t)$ exists and is the unique fixed point given by $\mathcal{G}^{(\bar{\pi}^*)}(b\t,a\t)= T^{unified}_{\bar{\pi}^*} \mathcal{G}^{(\bar{\pi}^*)}(b\t,a\t)$.
\begin{corollary} {(Existence of a unique optimal belief state-action $\mathcal{G}^{(\bar{\pi}^*)}(b\t,a\t)$)} \label{theo:action_optimal_policy} \\
The optimal belief state-action EFE $\mathcal{G}^{(\bar{\pi}^*)}(b\t,a\t)=\min_{\bar{\pi} \in \Pi}\mathcal{G}^{(\bar{\pi})}(b\t,a\t)$ defined in a POMDP satisfying Assumption~\ref{ASS:regula} exits and is the unique fixed point of the following equation:
 \begin{eqnarray}
 \!\!\!\!\!\!\!\!\!\!\!\! \mathcal{G}^{(\bar{\pi}^*)}(b\t,a\t)&\!\!\!\!= \!\!\!\!& T^{unified}_{\bar{\pi}^*  } \mathcal{G}^{(\bar{\pi}^*)}(b\t,a\t)\\ 
  &\!\!\! \!= \!\!\! \!& \mathbb{E}_{ P(r\t|b\t,a\t)} \bigg[  -\alpha \, r\t+\mathbb{E}_{   P(b\nt,o\nt, \s\nt| b\t,a\t) }  \big[ \zeta \,   \text{log} \frac{P(\s\nt|  b\t, a\t)}{ q({\s}_{t+1}|{o}_{t+1}, b\t)}   \nonumber \\
&\!\!\!\!+\!\!\!\!&   \gamma  \mathbb{E}_{\bar{\pi}^*(a\nt|b\nt)  }[ \beta \, \text{log} \bar{\pi}^*(a\nt|b\nt) + \mathcal{G}^{(\bar{\pi}^*)}(b\nt, a\nt)]  \big] \bigg].  \label{Eq:unified_action_optimal-Bellman_operator}
\end{eqnarray}
\end{corollary}
\begin{proof}
See Appendix~\ref{APP:action_optimal_policy}.
\end{proof} \vspace{.1in}
\\
Corollary~\ref{theo:action_optimal_policy} demonstrates the existence and uniqueness of $\mathcal{G}^{(\bar{\pi}^*)}(b_t,a_t)$, which plays a crucial role in Eq.~\eqref{Eq:project2} to determine the optimal policy $\bar{\pi}^*(a\t|b\t)$. Furthermore, Eq.~\eqref{Eq:unified_action_optimal-Bellman_operator} extends the unified Bellman optimality equation to incorporate the optimal belief state-action EFE $\mathcal{G}^{(\bar{\pi}^*)}(b\t,a\t)$, and therefore, we refer to it as the unified Bellman optimality equation for $\mathcal{G}^{(\bar{\pi}^*)}(b\t,a\t)$.
%
 \subsection{Derivation of unified policy iteration}
So far, we have proven  the existence and uniqueness of the optimal policy $\bar{\pi}^{*}(.|b\t)$ as
\begin{eqnarray}
\!\!\!\!\!\!\!\!\!\!\!\!\!\!\! \bar{\pi}^{*}(.|b\t)= \arg\min_{{ \bar{\pi} \in \bar{\Pi} }} D_{KL} \left[\bar{\pi}(.|b\t), \sigma \left( \frac{- \mathcal{G}^{(\bar{\pi}^*)}(b\t,.)}{\beta} \right)\right], \label{Eq:project3}
\end{eqnarray}
where
 \begin{eqnarray}
 \!\!\!\!\!\!\!\! \!\!\! \!\!\!\!\!\!\!  \mathcal{G}^{(\bar{\pi}^*)}(b\t,a\t) &\!\!\!=\!\!\!&  \mathbb{E}_{ {P}(r\t|b\t,a\t)} \bigg[  -\alpha \, r\t+\mathbb{E}_{  P(b\nt,o\nt, \s\nt| b\t,a\t) }  \big[ \zeta \,   \text{log} \frac{P(\s\nt|  b\t, a\t)}{ q({\s}_{t+1}|{o}_{t+1}, b\t)}   \nonumber \\
&\!\!\!+\!\!\!&   \gamma  \mathbb{E}_{\bar{\pi}^*(a\nt|b\nt)}[ \beta \, \text{log} \bar{\pi}^*(a\nt|b\nt) + \mathcal{G}^{(\bar{\pi}^*)}(b\nt, a\nt)]  \big] \bigg]. \label{Eq:unified_const_action_Bellman3}
\end{eqnarray}
From Eqs.~\eqref{Eq:project3} and \eqref{Eq:unified_const_action_Bellman3}, it is evident that the learning of the optimal policy $\bar{\pi}^*(a_t|b_t)$ at time $t$ depends on the optimal belief state-action $\mathcal{G}^{(\bar{\pi}^*)}(b_t,a_t)$, which, in turn, depends on the optimal policy $\bar{\pi}^*(a_{t+1}|b_{t+1})$ at the subsequent time step. As a result,  $\bar{\pi}^*(a_t|b_t)$ and its corresponding belief state-action $\mathcal{G}^{(\bar{\pi}^*)}(b\t,a\t)$ follow a recursive pattern over time, enabling their recursive determination.
To initiate this recursive process, we can start with an arbitrary policy, which serves as the initial optimal policy (baseline policy), and initialize its corresponding belief state-action EFE with arbitrary values. We then update the belief state-action EFE corresponding to the baseline policy using Eq.~\eqref{Eq:unified_const_action_Bellman3}. Next, we update the baseline policy using Eq.~\eqref{Eq:project3}. This updated policy becomes the new baseline policy, and the process is repeated iteratively over $k \in \{0, 1, 2, ... \}$.
Inspired by this recursive procedure, we propose an iterative algorithm called unified policy iteration for concurrent learning of $\bar{\pi}^*$ and $\mathcal{G}^{(\bar{\pi}^*)}$. The algorithm involves updating the baseline policy and its corresponding belief state-action EFE iteratively in a computationally efficient manner. We provide a proof of convergence for the proposed unified policy iteration, showing that it converges to the optimal policy $\bar{\pi}^*$ and its corresponding belief state-action EFE $\mathcal{G}^{(\bar{\pi}^*)}$. This algorithm offers an effective approach for solving the POMDP problem with continuous state, action, and observation spaces, providing a practical framework for decision-making in complex environments.

Consider a randomly selected policy $\bar{\pi}^{(old)} \in \bar{\Pi}$ and its randomly initialized belief state-action EFE $\mathcal{G}^{({\bar{\pi}}^{(old)})}$. We can update  $\mathcal{G}^{({\bar{\pi}}^{(old)})}$ by applying the unified Bellman operator $T^{unified}_{\bar{\pi}^{(old)}} \mathcal{G}^{({\bar{\pi}}^{(old)})}(b\t,a\t)$, as follows
 \begin{eqnarray}
\!\!\!\!\!\!\!\! \!\!\! \!\!\!\!\! \mathcal{G}^{(\bar{\pi}^{(old)})}(b\t,a\t)   &\!\!\!\! \leftarrow \!\!\!\!& T^{unified}_{\bar{\pi}^{(old)}} \mathcal{G}^{({\bar{\pi}}^{(old)})} (b\t,a\t) \\
&\!\!\!\!=\!\!\!\!& \mathbb{E}_{  P(r\t|b\t,a\t)} \bigg[  -\alpha \, r\t+\mathbb{E}_{   P(b\nt,o\nt, \s\nt| b\t,a\t) }  \big[ \zeta \,   \text{log} \frac{P(\s\nt|  b\t, a\t)}{ q({\s}_{t+1}|{o}_{t+1}, b\t)}   \nonumber \\
\!\!\! &\!\!\!\!+\!\!\!\!&   \gamma \mathbb{E}_{\bar{\pi}^{(old)}(a\nt|b\nt)}[ \beta \, \text{log} \bar{\pi}^{(old)}(a\nt|b\nt) + \mathcal{G}^{(\bar{\pi}^{(old)})}(b\nt, a\nt)]  \big] \bigg]. \label{Eq:update_Bellman_action}
\end{eqnarray}
Next, let's consider a new policy ${\bar{\pi}}^{(new)}(.|b\t)$ resulting from the following equation:
\begin{equation}
\!\!\!\!\! {\bar{\pi}}^{(new)}(.|b\t)\! =\!\arg\min_{{{\bar{\pi} \in \bar{\Pi}} }} \! D_{KL} \!\! \left[\bar{\pi}(.|b\t), \sigma \left( \frac{- \mathcal{G}^{({\bar{\pi}}^{(old)})}(b\t,.)}{\beta} \right)\right]. \label{Eq:project_improv}
\end{equation}
The next lemma proves that the belief state-action EFE corresponding to ${\bar{\pi}}^{(new)}$ has a lower value than the belief state-action EFE corresponding to the baseline policy $\bar{\pi}^{(old)}$. We thus refer to the new policy $\bar{\pi}^{(new)}$ as an improved version of the baseline policy $\bar{\pi}^{(old)}$, and we refer to the next lemma as the unified policy improvement.
\begin{lemma} {(Unified policy improvement)} \\ \label{lem:policy_impr}
Let ${\bar{\pi}}^{(old)} \in \Pi$ be a randomly selected policy, and let ${\bar{\pi}}^{(new)}(.|b\t) $ be ${\bar{\pi}}^{(new)}(.|b\t) =\arg\min_{{{\bar{\pi} \in \bar{\Pi}} }}  D_{KL} [\bar{\pi}(.|b\t), \sigma ( \frac{- \mathcal{G}^{({\bar{\pi}}^{(old)})}(b\t,.)}{\beta} )]$ in a POMDP satisfying Assumption~\ref{ASS:regula}. Then:
\begin{eqnarray}
\!\!\!\!\!\!\!\mathcal{G}^{(\bar{\pi}^{(new)})}(b\t, a\t) \leq \mathcal{G}^{(\bar{\pi}^{(old)})}(b\t ,a\t)
\end{eqnarray}
for all $(b\t, a\t) \in \mB \times \mA $.
\end{lemma}
\begin{proof}
See Appendix~\ref{APP:policy_impr}.
\end{proof} \vspace{.1in}\\
Using Lemma~\ref{lem:policy_impr}, we formally state our proposed unified policy iteration algorithm through the following theorem.
\begin{theorem} {(Unified policy iteration)}  \label{the:policy_iteraion} \\
Consider a POMDP satisfying Assumption~\ref{ASS:regula}.  By starting from an initial policy $\bar{\pi}_0 \in \bar{\Pi}$ and an initial mapping $\mathcal{G}_0: \mathcal{B} \times \mathcal{A} \rightarrow \mathbb{R}$ and recursively applying the following updates for $k= 0, 1, 2, \ldots $:
\begin{align}
\mathcal{G}_{k+1}(b_t,a_t) &= T^{unified}_{\bar{\pi}_k} \mathcal{G}_k(b_t,a_t), \label{Eq:value_update} \\
\bar{\pi}_{k+1}(\cdot|b_t) &= \arg\min_{\bar{\pi} \in \bar{\Pi}} D_{KL} \left[\bar{\pi}(\cdot|b_t), \sigma \left( \frac{-\mathcal{G}\nk(b_t,\cdot)}{\beta} \right) \right], \label{Eq:policy_update} 
\end{align}
the iterative process converges to $\bar{\pi}^*=\arg\min_{\bar{\pi} \in \bar{\Pi}} {G}^{(\bar{\pi})}(b_t)$.
\end{theorem}
\begin{proof}
See Appendix~\ref{APP:policy_iteraion}.
\end{proof} 
\vspace{.1in}\\
Theorem~\ref{the:policy_iteraion} demonstrates that the optimal policy $\bar{\pi}^*=\arg\min_{\bar{\pi} \in \bar{\Pi}} {G}^{(\bar{\pi})}(b_t)$ can be obtained by recursively alternating between the belief state-action EFE update step, given by Eq.~\eqref{Eq:value_update}, and the policy improvement step, given by Eq.~\eqref{Eq:policy_update}.
%
\subsection{Unified reward function for AIF and RL in POMDPs} \label{sec:unified_R}
The unified Bellman equation for the belief state EFE $G^{(\bar{\pi})} (b_t)$ and the Bellman optimality equation for $G^{(\bar{\pi}^*)} (b_t)$ in a POMDP, as stated in Eqs.~\eqref{Eq:unified_Bellman_mod} and \eqref{Eq:unified_optimality_Bellman_mod} respectively, exhibit a strong resemblance to the Bellman equation for the state value function $V^{(\pi)}(s_t)$ and the Bellman optimality equation for $V^*(s_t)$ in an MDP, as represented by Eqs.~\eqref{Eq:Bellman} and \eqref{Eq:Bellman_opt} respectively. Given this similarity and considering that the proposed unified inference algorithm aims to minimize the belief state EFE while RL endeavours to maximize the state value function, we can interpret $G^{(\bar{\pi})} (b_t)$ as $-\bar{V}^{(\bar{\pi})}(b_t)$, where $\bar{V}^{(\bar{\pi})}(b_t)$ is referred to as the belief state value function, defined as follows:
\begin{eqnarray}
\bar{V}^{(\bar{\pi})}(b\t) &\!\!\!=\!\!\!&  \mathbb{E}_{ \prod_{\tau =t}^{\infty} \pi(a_{\tau}|b_{\tau}) {P}(s_{\tau+1}, r\ta, o_{\tau+1}, b_{\tau+1}|  b\ta, a\ta) } \bigg[ \sum_{\tau=t}^{\infty}  \gamma^{\tau-t} \bigg(\alpha \, r\ta- \zeta \, \text{log} \frac {P({\s_{\tau+1}}|a\ta,b\ta) }{q({\s_{\tau+1}}|o_{\tau+1}, b\ta)}  \nonumber \\
 \!\!\!\!\!\! &\!\!\!-\!\!\!&     \beta \, \text{log}\bar{\pi}(a_{\tau}|b_{\tau})   \bigg)  \bigg]. \label{Eq:G-v_relation}
\end{eqnarray}
The belief state value function $\bar{V}^{(\bar{\pi})}(b\t)$ in Eq.~\eqref{Eq:G-v_relation} can be considered as a value function for a POMDP with the reward function $r^{unified}\ta$ at time step $\tau$ defined as:
\begin{eqnarray}
r^{\text{unified}}\ta &\!\!\!=\!\!\!& \alpha \underbrace{ r\ta}_{r^{\text{extrinsic}}\ta}+ \zeta \underbrace{  \left( -  \text{log} \frac {P({\s_{\tau+1}}|a\ta,b\ta) }{q({\s_{\tau+1}}|o_{\tau+1}, b\t)} \right)}_{r^{\text{intrinsic}}\ta} + \beta \, \underbrace{  \left(-\text{log} \bar{\pi}(a\ta|b\ta) \right) }_{r^{\text{entropy}}\ta}. \label{Eq:R_unified}
\end{eqnarray}
The reward function $r^{\text{unified}}\ta$ represents the agent's total reward at time step $\tau$ after performing action $a\ta$ in belief state $b\ta$.  It encompasses the reward $r_{\tau}$ obtained, the transition to the next state $s_{\tau+1}$, and the observation $o_{\tau+1}$ within the POMDP framework.  The first term of $r^{\text{unified}}\ta$ (i.e., $r^{\text{extrinsic}}\ta$) corresponds to the extrinsic reward, which is the reward function commonly used in MDP-based RL. The second term ($r^{\text{intrinsic}}\ta$) represents the intrinsic reward that encourages the agent to visit states providing the most information about the hidden states of the environment.
The combination of the extrinsic reward $r^{\text{extrinsic}}\ta$ and the intrinsic reward $r^{\text{intrinsic}}\ta$ balances exploration for information gain with the exploitation of extrinsic rewards.  The third term in Eq.~\eqref{Eq:R_unified} arises from taking the expected value over future actions and corresponds to entropy maximization.  This term, referred to as the entropy reward $r^{\text{entropy}}_{\tau}$, promotes random exploration and prevents the policy from becoming overly deterministic, thereby encouraging the learning of stochastic policies $\bar{\pi}$. 
As explained in Section~\ref{Sec:problem setting}, learning a stochastic policy enables the agent to adapt to different situations resulting from environmental changes, leading to improved stability in the proposed unified policy iteration approach. The reward function $r^{\text{unified}}\ta$ combines both the extrinsic reward employed in RL and the intrinsic reward that promotes information-seeking exploration as used in AIF. Hence,  we designate $r^{\text{unified}}\ta$ as a unified reward function applicable to both AIF and RL in POMDPs.

Given that the state-action value function $Q^{(\pi)} (\s_t,a_t)$ and the belief state-action EFE $\mathcal{G}^{(\bar{\pi})} (b_t,a_t)$ are obtained by conditioning $V^{(\pi)}(\s_t)$ and $G^{(\bar{\pi})} (b_t)$ on the current action $A_t=a_t$, we can similarly interpret $\mathcal{G}^{(\bar{\pi})} (b_t,a_t)$ as negative counterpart of the belief state-action value function  $\bar{Q}^{(\bar{\pi})} (b_t,a_t)$, defined as follows:
\begin{eqnarray}
\!\!\!\! \bar{Q}^{(\bar{\pi})} (b\t,a\t) &\!\!\!\!=\!\!\!\!&  \mathbb{E}_{ \prod_{\tau =t}^{\infty} \pi(a_{\tau+1}|b_{\tau+1}) {P}(s_{\tau+1}, r\ta, o_{\tau+1}, b_{\tau+1}|  b\ta, a\ta) } \bigg[ \sum_{\tau=t}^{\infty}  \gamma^{\tau-t} \bigg(\alpha \, r\ta- \zeta \, \text{log} \frac {P({\s_{\tau+1}}|a\ta,b\ta) }{q({\s_{\tau+1}}|o_{\tau+1}, b\ta)}  \nonumber \\
 \!\!\!\!\!\! &\!\!\!\!-\!\!\!\!&     \beta \, \text{log}\bar{\pi}(a_{\tau+1}|b_{\tau+1})   \bigg)  \bigg]. \label{Eq:g-Q_relation}
\end{eqnarray}

By setting $\alpha = 1$, $\beta = 0$, and $\zeta = 0$ in the unified reward function $r^{\text{unified}}\ta$ in Eq.~\eqref{Eq:R_unified}, the optimal policy $\bar{\pi}^* = \arg\max_{\bar{\pi} \in \bar{\Pi}} \bar{V}^{(\bar{\pi})}(b_t)$ aims to maximize the expected long-term extrinsic reward alone, which is the objective in MDP-based RL algorithms.  However, by setting $\alpha = 1$, $\beta = 0$, and $\zeta = 1$, the optimal policy $\bar{\pi}^* = \arg\max_{\bar{\pi} \in \bar{\Pi}} \bar{V}^{(\bar{\pi})}(b_t)$ aims to maximize both the expected long-term extrinsic reward through $r^{extrinsic}\ta$ and the expected long-term information gain through  $r^{intrinsic}\ta$, which is crucial in the context of POMDPs. Therefore, we can extend RL methods developed for MDPs to POMDPs by augmenting the external reward $r\ta$ in these RL approaches with the intrinsic reward $-\log \frac{P({s_{\tau+1}}|b\ta, a\ta)}{q({s_{\tau+1}}|o_{\tau+1}, b\ta)}$.
\\
Furthermore, by setting $\alpha = 0$, $\beta = 0$, and $\zeta = 1$ in the unified reward function $r^{unified}\ta$ in Eq.~\eqref{Eq:R_unified}, the optimal policy $\bar{\pi}^* = \arg\max_{\bar{\pi} \in \bar{\Pi}} \bar{V}^{(\bar{\pi})}(b_t)$ solely focuses on maximizing the expected information gain through the intrinsic reward term  $r^{intrinsic}\ta$. Therefore, by replacing the extrinsic reward $r\ta$ in these RL approaches with the intrinsic reward $-\log \frac{P({\s_{\tau+1}}|b\ta, a\ta)}{q({\s_{\tau+1}}|o_{\tau+1}, b\t)}$, we can extend extrinsic reward-dependent RL algorithms to a setting where extrinsic reward values have not been determined by an external supervisor. This extension is particularly important in scenarios where designing extrinsic reward values is costly or specifying them is challenging.
\section{Unified inference model} \label{sec:design}
As mentioned in Assumption~\ref{ASS:percep}, the proposed unified inference assumes that the agent performs perceptual inference and learning before action selection at each time step by minimizing the VFE. The action selection phase is facilitated by the proposed unified policy iteration method outlined in Theorem~\ref{the:policy_iteraion}.
\\
However, in scenarios with continuous state, action, and observation spaces, the infinite number of possible states, actions, and observations makes it impractical to explicitly represent the generative model, variational posterior distribution, belief state-action EFE, and policy using a finite set of values~\cite{sutton2018reinforcement, haarnoja2018soft}. 
To overcome this challenge, following common practices in both RL and AIF~\cite{haarnoja2018soft, mnih2013playing, ueltzhoffer2018deep, lee2020stochastic}, we utilize DNNs to approximate the generative model, variational posterior distribution, belief state-action EFE, and policy.
\\
{ Considering that perceptual inference and learning are derived from existing literature, detailed explanations are provided in Appendix~\ref{sub-sec:condionined_perc}. This section primarily focuses on our contributions: learning the optimal belief state-action policy through our proposed unified policy iteration method.}
\subsection{Unified actor-critic} \label{Seb_unified _AC}
 In this sub-section, we focus on approximating the belief state-action EFE $\mathcal{G}^{(\bar{\pi})}(b\t,a\t)$ and the belief state-action policy $\bar{\pi}$ within the context of the proposed unified policy iteration, which involves alternating between the belief state-action EFE update step, as given by Eq.~\eqref{Eq:value_update}, and the policy update step, as given by Eq.~\eqref{Eq:policy_update}.   Inspired by actor-critic algorithms in RL, we refer to the approximated policy as the actor and the approximated belief state-action EFE as the critic. Hence, we term our approximation of the unified policy iteration framework as the unified actor-critic.
\\
By utilizing the relationship $G^{(\bar{\pi})}(b_t)=-\bar{V}^{(\bar{\pi})}(b_t)$, which arises from the unified reward function $r^{\text{unified}}_t$ in Eq.~\eqref{Eq:R_unified}, we can leverage and adapt a wide range of advanced MDP-based actor-critic algorithms to learn the critic $\mathcal{G}^{(\bar{\pi})}(b_t, a_t)$ and the actor $\bar{\pi}(a_t | b_t)$ in our unified actor-critic algorithm. This enables us to take advantage of the existing methods and techniques developed for MDPs and extend them to address the challenges posed by partial observability in POMDPs.

To approximate $\mathcal{G}^{(\bar{\pi})}(b_t, a_t)$ and $\bar{\pi}(a_t | b_t)$, we parameterize them with $\psi$ and $\phi$, respectively. However, directly inputting the infinite-dimensional continuous belief state $b_t$ to the DNNs modeling the policy and the belief state-action EFE function is impractical. Instead, we use a belief state representation $h\t \in \mathbb{R}^{D_{\mathcal{H}}}$, where $D_{\mathcal{H}}$ is the dimension of $h_t$. This representation is explained in Appendix~\ref{sub-sec:condionined_perc}.
We then model the belief state-action EFE and the policy as $\mathcal{G}^{(\bar{\pi})}_{\psi}(h_t, a_t)$ and $\bar{\pi}_{\phi}(a_t | h_t)$, respectively.
 Thus, the belief state-action EFE update and the policy improvement steps can be expressed as follows
 \begin{eqnarray}
\!\!\!\!\!\!\!\! \!\!\!  \mathcal{G}^{(\bar{\pi})}_{\psi}(h\t,a\t)  &\!\!=\!\!& \mathbb{E}_{ P(r\t|h\t,a\t)} \bigg[  -\alpha \, r\t+\mathbb{E}_{   P(h\nt,o\nt, \s\nt| h\t,a\t) }  \big[ \zeta \,   \text{log} \frac{P(\s\nt|  h\t, a\t)}{ q({\s}_{t+1}|{o}_{t+1})}   \nonumber \\
&\!\!+\!\!&   \gamma \mathbb{E}_{\bar{\pi}(a\nt|h\nt)}[ \beta \, \text{log} \bar{\pi}(a\nt|h\nt) + \mathcal{G}^{(\bar{\pi})}_{\psi}(h\nt, a\nt)]  \big] \bigg], \label{Eq:update_Bellman_action3}
\end{eqnarray}
and
\begin{equation}
\!\!\!\!\! {\bar{\pi}}_{\phi}(.|h\t)\! =\!\arg\min_{{{\bar{\pi} \in \bar{\Pi}} }} \! D_{KL} \!\! \left[\bar{\pi}(.|h\t), \sigma \left( \frac{- \mathcal{G}^{({\bar{\pi}})}_{\psi}(h\t,.)}{\beta} \right)\right]. \label{Eq:project_improv3}
\end{equation}

We train $\psi$ to minimize the squared residual error derived from Eq.~\eqref{Eq:update_Bellman_action3} using batches of size $B$.
Each batch can consist of $M$ sequential data points denoted as $\{( a\pk, r_k, o_k)_{k=1}^{M} \}_{i=1}^B$, sampled from a replay buffer $\mathcal{D}$ containing real environment interactions. Alternatively, batches may include $N$ simulated interaction data points $\{ (a\ta, r\ta, o_{\tau+1})_{\tau=t}^{t+N} \}_{i=1}^B$, starting from $h\t$ and  simulating interactions up to $N$ steps.
We can also use a combination of both real and imagined interactions for training $\psi$. Depending on the type of data used, we devise three unified actor-critic approaches: model-based unified actor-critic, model-free unified actor-critic, and hybrid unified actor-critic.
\subsubsection{Model-free unified actor-critic} \label{sec:MF_AC}
 Inspired by model-free RL algorithms, our proposed model-free unified actor-critic learns $\psi$ by minimizing the squared residual error from Eq.~\eqref{Eq:update_Bellman_action3} over a batch of real data $\{ (a\k, r\k, o\nk)_{k=1}^{M} \}_{i=1}^B$ sampled from the replay buffer $\mathcal{D}$. This involves minimizing the squared residual error $L_{\mathcal{G}}^{\text{MF}}(\psi)$, which is defined as 
\begin{eqnarray}
 \lefteqn{ \!\!\!\!\!\!\!\!\!\!\!\!\!    L_{\mathcal{G}}^{\text{MF}}({\psi})  = \mathbb{E}_{\mathcal{D}( a\k, r\k, o\nk) } \bigg[ \frac{1}{2}  \bigg( \mathcal{G}^{(\bar{\pi})}_{\psi}({h}\t,{a}\k)+ \alpha \, r\k -\mathbb{E}_{  q(s\nt|h\t,o\nk) }  \big[ \zeta \,   \text{log} \frac{P(\s\nt|  h\t, a\t)}{ q({\s}_{t+1}|{o}_{t+1}, h\t) } \nonumber} \\
 && \!\!\!\!\!\!\!\!\!\!\!\!\!\!\!\!\!  - \, \gamma \mathbb{E}_{P(h\nt|h\t, a\k,  o\nk, s\nt) \bar{\pi}_{\phi}(a\nt|h\nt)}[ \beta \, \text{log} \bar{\pi}_{\phi}(a\nt|h\nt) + \mathcal{G}^{(\bar{\pi})}_{\psi}(h\nt, a\nt)]  \big] \bigg)^2 \bigg] \!, \label{Eq:loss_psi_MF}
\end{eqnarray}
After updating $\psi$, the parameter $\phi$ of the policy $\bar{\pi}_{\phi}(a_{t}|h_{t})$ can be trained by minimizing the right-hand side of Eq.~\eqref{Eq:project_improv3}:
\begin{eqnarray}
L_{\bar{\pi}}^{\text{MF}}({\phi}) &=& \  D_{KL} \left[\bar{\pi}_{\phi}(.|h\t), \sigma \left( \frac{- \mathcal{G}_{\psi}^{(\bar{\pi})}(h\t,.)}{\beta} \right) \right]    \label{Eq:obj_param} \\
&=& \mathbb{E}_{ \bar{\pi}_{\phi}(a\t|h\t)} \left[  \beta\, \text{log} \bar{\pi}_{\phi}( a\t|h\t )   + \mathcal{G}^{(\bar{\pi})}_{\psi} (h\t,a\t) \right].  \label{Eq:pi_loss} 
\end{eqnarray}
To minimize Eq.~\eqref{Eq:pi_loss}, we need to calculate its gradient with respect to $\phi$ by sampling from $\bar{\pi}_{\phi}( a_t|h_t )$. However, computing this gradient involves differentiating with respect to $\bar{\pi}_{\phi}(a_t|h_t)$, which we sample from. Therefore, we use the reparameterization trick~\cite{kingma2013auto} for sampling from $\bar{\pi}_{\phi}(a_t|h_t)$.

By employing the relationship $G^{(\bar{\pi})}_{\psi}(h_t)=-\bar{V}^{(\bar{\pi})}_{\psi}(h_t)$ and considering a specific case of the model-free unified actor-critic algorithm, where $\alpha=1$ and $\zeta=0$,  we recover the variational recurrent model (VRM) algorithm proposed in~\cite{han2020variational} as an extension of SAC~\cite{haarnoja2018soft} to POMDPs. Additionally, by assigning $\alpha=0$ and $\beta \neq 0$ in the model-free unified actor-critic algorithm, we can formulate the so-called reward-free variants of SAC~\cite{haarnoja2018soft, haarnoja2018soft2} in settings without an extrinsic reward function.
Moreover, by setting $\alpha=1$ and assigning non-zero values to $\zeta$ (i.e., $\zeta \neq 0$), our model-free unified actor-critic algorithm is transformed into a generalized version of SAC for POMDPs. This generalized version maximizes both the expected future extrinsic reward and information gain. We refer to this extension as the generalized SAC (G-SAC).
\vspace{.1in}\\
 \textbf{Generalized SAC (G-SAC)} generalizes model-free SAC~\cite{haarnoja2018soft} to POMDPs with the following loss functions:\footnote{ In the original formulation of SAC~\cite{haarnoja2018soft}, they introduced an additional function approximator for the state value function, but later they found it to be unnecessary~\cite{haarnoja2018soft2}.}
\begin{eqnarray}
 \lefteqn{ \!\!\!\!\!\!\!\!\!\!\!\!   L_{\mathcal{G}}^{\text{G-SAC}}({\psi}) = \mathbb{E}_{\mathcal{D}( a\k, r\k, o\nk) } \bigg[ \frac{1}{2}  \bigg( \mathcal{G}^{(\bar{\pi})}_{\psi}({h}\t,{a}\k)+  r\k -\mathbb{E}_{  q(s\nt|h\t,o\nk) }  \big[ \zeta \,   \text{log} \frac{P(\s\nt|  h\t, a\t)}{ q({\s}_{t+1}|{o}_{t+1}, h\t) } \nonumber} \\
 && \!\!\!\!\!\!\!\!\!\!\!\!\!\!\!\!\! - \, \gamma \mathbb{E}_{P(h\nt|h\t, a\k,  o\nk, s\nt) \bar{\pi}_{\phi}(a\nt|h\nt)}[ \beta \, \text{log} \bar{\pi}_{\phi}(a\nt|h\nt) + \mathcal{G}^{(\bar{\pi})}_{\psi}(h\nt, a\nt)]  \big] \bigg)^2 \bigg] \!,  \label{Eq:unified_SAC} 
\end{eqnarray}
and
\begin{eqnarray}
 L_{\bar{\pi}}^{\text{G-SAC}}({\phi}) = \mathbb{E}_{ \bar{\pi}_{\phi}(a\t|h\t)} \left[  \beta\, \text{log} \bar{\pi}_{\phi}( a\t|h\t )   + \mathcal{G}^{(\bar{\pi})}_{\psi} (h\t,a\t) \right].\label{Eq:unified_SAC_policy}
 \end{eqnarray}

\textbf{Note 6.1:} It is important to clarify that the proposed model-free unified actor-critic algorithm incorporates a learned belief state generative model to compute the intrinsic reward $r^{\text{intrinsic}}_t$ and the belief state representation $h_t$.  Therefore, from a strict RL perspective, it is not considered a purely model-free algorithm. However, in this context, the term "model-free" refers to the fact that the generative model is not used for data selection during belief state-action EFE (critic) learning. Instead, the agent directly learns the policy and belief state-action EFE from observed interactions with the environment.

During the iterative process of the model-free unified actor-critic, the agent interacts with the environment, updates its belief state-action EFE based on observed rewards and observations, and learns through trial and error. Model-free methods, not relying on a learned model for data selection, are more robust to modeling errors or inaccuracies. However, this trial-and-error approach can be time-consuming and sample inefficient, often requiring numerous samples to converge to an optimal or near-optimal policy. To tackle this, we introduce a model-based version of the unified actor-critic method in the next phase to learn the optimal policy.
\subsubsection{Model-based unified actor-critic} \label{sec:MB_AC}
The proposed model-based unified actor-critic algorithm utilizes the learned belief state generative model to simulate and predict future observations and rewards resulting from different actions. This enables the agent to make decisions and optimize its policy without needing direct interaction with the environment.
\\
To implement the model-based unified actor-critic approach, the agent collects trajectories $\{ (a\ta, r\ta, o_{\tau+1})_{\tau=t}^{t+N} \}_{i=1}^B$ starting from the initial belief state $h_t$ and forwarding till $N$. These trajectories  are generated by imagining following the policy $a\ta \sim \bar{\pi}_{\phi}(a\ta | h\ta)$, utilizing the transition model $P(s_{\tau+1}|h\ta, a\ta)$, the reward model $P(r\ta|h\ta,a\ta)$, and the likelihood model $P(o_{\tau+1}|h\ta,s_{\tau+1})$.
It's important to note that the reward model $P(r_t|\s_t,a_t)$ also needs to be learned, with the process detailed in Appendix~\ref{APP:reward_model}.
\\
Next, the critic parameter $\psi$ is updated by minimizing the following loss function:
\begin{eqnarray}
\! L_{\mathcal{G}}^{\text{MB}}({\psi}) &\!\!\!\! =\!\!\!\! & \mathbb{E}_{ \bar{\pi}_{\phi} (a\t|h\t) P(r\t|h\t,a\t) } \bigg[ \frac{1}{2}  \bigg( \mathcal{G}^{(\bar{\pi})}_{\psi}({h}\t,{a}\t)+ \alpha \, r\t \nonumber \\
 \!\!\!\!\!\!\!\!\!\!\!\!\!\! &\!\!\! \! -\!\!\!\! &  \mathbb{E}_{ P(\s\nt|  h\t, a\t) P(o\nt|  h\t, s\nt)}  \big[ \zeta \,   \text{log} \frac{P(\s\nt|  h\t, a\t)}{ q({\s}_{t+1}|{o}_{t+1}, h\t) }  \label{Eq:loss_psi_MB} \\
 \!\!\!\!\!\!\!\!\!\!\!\!\!\! &\!\!\! \! -\!\!\!\! & \, \gamma \, \mathbb{E}_{P(h\nt|h\t, a\t,  o\nt, s\nt) \bar{\pi}_{\phi}(a\nt|h\nt)}  [   \beta \, \text{log} \bar{\pi}_{\phi}({a}\nt|{h}\nt)+  \mathcal{G}^{(\bar{\pi})}_{\psi}({h}_{t+1}, {a}_{t+1}) ]  \bigg)^2 \bigg], \nonumber 
\end{eqnarray}
where  the expectations are estimated under the imagined trajectories $\{ (a\ta, r\ta, o_{\tau+1})_{\tau=t}^{t+N} \}_{i=1}^B$. The actor parameter $\phi$ is estimated by minimizing the following loss function:
\begin{eqnarray}
L_{\bar{\pi}}^{\text{MB}}({\phi}) &=& \mathbb{E}_{ \bar{\pi}_{\phi}(a\t|h\t)} \big[  \beta\, \text{log} \bar{\pi}_{\phi}( a\t|h\t )   + \mathcal{G}^{(\bar{\pi})}_{\psi} (h\t,a\t) \big].  \label{Eq:pi_loss_2} 
\end{eqnarray}
%

\textbf{Note 6.2:} This model-based unified actor-critic approach reduces memory usage and improves sample efficiency compared to the model-free unified actor-critic. However, it requires learning the reward model $P(r_\tau|h_\tau,a_\tau)$, unlike the model-free unified actor-critic. Learning the reward model of an environment is typically more challenging than learning the transition and likelihood models since prediction errors in the state and observation of the environment provide a richer source of information than reward prediction errors. This is because rewards are usually scalar, while the environment's state and observation are typically characterized by high dimensionality.

By relating $G^{(\bar{\pi})}_{\psi}(h_t)$ to $-\bar{V}^{(\bar{\pi})}_{\psi}(h_t)$, we can recover the model-based actor-critic Dreamer algorithm~\cite{Hafner2020Dream}, which focuses on maximizing rewards, as a special case of our model-based unified actor-critic algorithm when $\alpha=1$ and $\zeta=0$.\footnote{Dreamer computes estimates of $G^{(\bar{\pi})}(h_t)$ instead of $\mathcal{G}^{(\bar{\pi})}(h_t, a\t)$.  However, these two estimates are equivalent due to the relationship between $G^{(\bar{\pi})}(h_t)$ and $\mathcal{G}^{(\bar{\pi})}(h_t, a\t)$ illustrated in Appendix~\ref{APP:unified_Bellman_action}. Specifically, we have ${G}^{(\bar{\pi})}(h\t)  = \mathbb{E}_{\bar{\pi}(a\t|h\t)} \left[\beta \, \text{log} {\bar{\pi}}(a\t|h\t) + \mathcal{G}^{(\bar{\pi})}(h\t, a\t)     \right]$. } Additionally, we can create a reward-free version of the Dreamer by assigning $\alpha=0$ and $\beta \neq 0$ in the model-based unified actor-critic algorithm. Furthermore, by selecting $\alpha=1$ and assigning non-zero values to $\zeta$,  we extend Dreamer to POMDPs, aiming to maximize both the expected long-term extrinsic reward and the expected long-term information gain during action selection. This extended version is referred to as the Generalized Dreamer (G-Dreamer).
\vspace{.1in}\\
\textbf{Generalized Dreamer (G-Dreamer)} generalizes model-based Dreamer framework~\cite{Hafner2020Dream} to POMDPs  with the following loss functions:
\begin{eqnarray}
 \! L_{\mathcal{G}}^{\text{G-Dreamer}}({\psi}) &\!\!\!\! =\!\!\!\! & \mathbb{E}_{ \bar{\pi}_{\phi} (a\t|h\t) P(r\t|h\t,a\t) } \bigg[ \frac{1}{2}  \bigg( \mathcal{G}^{(\bar{\pi})}_{\psi}({h}\t,{a}\t)+  r\t \nonumber \\
 \!\!\!\!\!\!\!\!\!\!\!\!\!\! &\!\!\! \! -\!\!\!\! &  \mathbb{E}_{ P(\s\nt|  h\t, a\t) P(o\nt|  h\t, s\nt)}  \big[ \zeta \,   \text{log} \frac{P(\s\nt|  h\t, a\t)}{ q({\s}_{t+1}|{o}_{t+1}, h\t) }  \label{Eq:unified_Dream}  \\
 \!\!\!\!\!\!\!\!\!\!\!\!\!\! \!\!\!\!\!\!\!\!\!\!\!\!\!\! &\!\!\! \! -\!\!\!\! & \, \gamma \, \mathbb{E}_{P(h\nt|h\t, a\k, o\nt, s\nt) \bar{\pi}_{\phi}(a\nt|h\nt)}  [   \beta \, \text{log} \bar{\pi}_{\phi}({a}\nt|{h}\nt)+  \mathcal{G}^{(\bar{\pi})}_{\psi}({h}_{t+1}, {a}_{t+1}) ]  \bigg)^2 \bigg], \nonumber  
\end{eqnarray}
and
\begin{eqnarray}
L_{\bar{\pi}}^{\text{G-Dreamer}}({\phi}) &=& \mathbb{E}_{ \bar{\pi}_{\phi}(a\t|h\t)} \big[  \beta\, \text{log} \bar{\pi}_{\phi}( a\t|h\t )   + \mathcal{G}^{(\bar{\pi})}_{\psi} (h\t,a\t) \big].   \label{Eq:pi_loss_Dream} 
\end{eqnarray}
 %
\subsubsection{Hybrid unified actor-critic} \label{sec:hybrid_AC}
While model-based unified actor-critic methods offer the potential for higher sample efficiency and advanced trajectory integration, their reliance on accurate models for performance is crucial. Errors in learned generative models can accumulate over time, leading to significant deviations from desired behavior~\cite{ma2021contrastive}. Learning precise models, especially with complex observations, is challenging. Inspired by~\cite{ma2021contrastive}, to mitigate model errors, we adopt a hybrid unified actor-critic learning scheme. This approach combines the sample efficiency of model-based learning with the robustness of model-free learning to address inaccuracies.
\\
 In the hybrid unified actor-critic scheme, the critic parameter ${\psi}$ is learned using both real trajectories $\{ (a\k, r\k, o\nk)_{k=1}^{M} \}_{i=1}^B$ and imagined trajectories $\{ (a\ta, r\ta, o_{\tau+1})_{\tau=t}^{t+N} \}_{i=1}^B$:
\begin{eqnarray}
L_{\mathcal{G}}^{\text{Hybrid}}({\psi}) &=& L_{\mathcal{G}}^{\text{MB}} ({\psi})+ c\, L_{\mathcal{G}}^{\text{MF}}({\psi}),  \label{Eq:loss_psi_hybrid}
\end{eqnarray}
where $c$ represents the scaling factor that determines the relative importance of the model-free critic loss function $L_{\mathcal{G}}^{\text{MF}}({\psi})$ compared to the model-based critic loss function  $L_{\mathcal{G}}^{\text{MB}} ({\psi})$. The actor parameter $\phi$ is learned by minimizing the following loss function:
\begin{eqnarray}
L_{\bar{\pi}}^{\text{Hybrid}}({\phi}) &=& \mathbb{E}_{\bar{\pi}_{\phi}(a\t|h\t)} \big[  \beta\, \text{log} \bar{\pi}_{\phi}( a\t|h\t )   + \mathcal{G}^{(\bar{\pi})}_{\psi} (h\t,a\t) \big].  \label{Eq:pi_loss_hybrid} 
\end{eqnarray}

The complete learning process of the unified inference model is outlined in Algorithm~\ref{Algorithm:unified iteration} of the appendix.
\section{Related work} \label{sec:related}
This section provides a concise overview of relevant literature, specifically addressing the extension of RL and AIF techniques to POMDPs featuring continuous spaces.
\subsection{RL approaches for continuous space POMDPs } \label{sub-sec:RL_related}
While RL conventionally targets MDPs, recent progress has broadened its application to POMDPs. These advancements frequently employ memory-based neural networks to encode past observations and actions or employ belief state inference.
\subsubsection{Memory-based approaches} \label{sub-sec:memeory_related}
\citeA{hausknecht2015deep} developed a variant of~\cite{mnih2013playing} to handle POMDPs by incorporating a recurrent layer, such as Long Short-Term Memory (LSTM) or Gated Recurrent Unit (GRU), to capture the history of observations. However, this method did not consider the history of actions, focusing solely on the observation sequence. Later,~\citeA{zhu2017improving, heess2015memory, nian2020dcrac} utilized recurrent layers to capture both the observation and action history. They demonstrated that it is possible to store only the necessary statistics of the history using recurrent layers instead of storing the entire preceding history. It is worth noting that these works primarily focused on tasks with discrete action spaces rather than continuous ones.
 \citeA{haklidir2021guided} proposed the guided SAC approach, which augments the original SAC with a guiding policy. The guided SAC architecture consists of two actors and a critic, where the original actor incorporates the history of observations and actions, while the guiding actor uses the true state as input. Although the guided SAC has been applied to tasks with continuous observation and action spaces, it still requires storing the history of observations and actions and relies on an external supervisor to provide additional information about the true state of the environment, which remains a challenge.
\citeA{meng2021memory, ni2022recurrent, yang2021recurrent} extended actor-critic algorithms to POMDPs by adding recurrent layers to both the actor and critic components to compress history into a into a fixed-size representation. This compresses history into a manageable form, passing to the actor and critic. This adaptation allowed the models to effectively handle continuous action spaces.

{ Although memory-based approaches have demonstrated promise in tasks involving continuous state, action, and observation spaces, our proposed unified inference framework, which expands AIF into continuous spaces through the extension of the EFE to stochastic belief state-action policies,  offers three main advantages:
\vspace{.1in}\\
\textit{i) Computational and memory efficiency:} Firstly, although we also transform the continuous belief state $b_t$ into a fixed-length representation $h_t$ using an LSTM, the LSTM in our method updates the belief state representation based on fixed-size inputs $h_{t-1}$, $a_{t-1}$, $o_t$, and $s_t$ (see Appendix~\ref{sub-sec:condionined_perc} for more details). In contrast, the memory-based methods in \cite{meng2021memory, ni2022recurrent, yang2021recurrent} require passing the entire history to their LSTM, leading to substantial memory demands in large or infinite time horizon problems. Although our belief state inference-based method may not require as much memory since it does not need to store extensive historical data, it still demands significant computational resources. These demands arise primarily from the complexity involved in learning the generative model, belief state, and belief state representation.  However, belief state inference-based algorithms, including ours, often become more computationally efficient over time compared to memory-based methods. This efficiency stems from two main factors: a) Belief state inference-based approaches learn generative models and belief states through feed-forward neural networks, which are less computationally demanding than RNNs \cite{pascanu2013difficulty}. b) In memory-based methods, the dimension of inputs to LSTMs grows as history accumulates, eventually surpassing the fixed input dimension of our LSTM for belief state representation learning in our infinite horizon POMDP setting. This increase in the input dimension of RNNs leads to increasingly computationally expensive matrix multiplications at each step, while the computational complexity of our algorithms remains fixed across time steps. This illustrates one of the advantages of basing our proposed unified inference on AIF, as it allows us to leverage inference from AIF for decision-making through a belief state-action policy.
\\
Therefore, over time, the rapid increase in input dimensions in memory-based methods surpasses the memory and computational demand of learning generative models, belief states, and belief state representations in our method, especially in scenarios with large/infinite horizons. A detailed comparison of the memory and computational complexities between our algorithm and memory-based baselines is provided in Appendix~\ref{app: computation}.
\vspace{.1in}\\
\textit{ii) Enhanced exploration:} Beyond the computational and memory efficiencies, our approach incorporates an information gain exploratory term when making decisions, highlighting another necessity of AIF for managing the uncertainty inherent in partially observable environments. 
 In recent years, extensive exploration methods for RL in fully observable environments have emerged~\cite{pathak2017curiosity, choi2018contingency, savinov2018episodic}. However, due to the limited observability of states in partially observable tasks, designing intrinsic exploration methods is non-trivial and challenging.
For tasks with partial observability, one prominent line of exploration utilized in memory-based RL relates to prediction error-based approaches, such as the intrinsic curiosity model (ICM)~\cite{pathak2017curiosity} and random network distillation (RND)~\cite{burda2018exploration}. 
These methods, which can be enhanced by incorporating memory mechanisms such as RNNs to handle partial observability, learn a forward model to predict the next state and sometimes a backward model to infer past states.  The prediction loss from these models serves as intrinsic motivation for the agent to explore new and informative states. \citeA{oh2019learning} introduced a triplet ranking loss to push the prediction output of the forward dynamics model to be far from the output generated by taking alternative actions.
However, these prediction error-based approaches face challenges in accurately discerning the novelty of an agent's state in partially observable settings for two main reasons: 1) As states are not fully observable, applying these methods in partially observable environments requires constructing forward and backward models based solely on observations, which are noisy or incomplete representations of states. Relying solely on local observations is insufficient for accurately inferring novelty over the true world state in tasks with partial observability. This is because two observations might appear identical at various locations on the map, although their underlying true world states are fundamentally different. 2) Prediction errors might restrict the expressiveness of the inferred intrinsic reward scores. This challenge is especially prominent in environments characterized by continuous state spaces, where state changes occur subtly over short intervals. Consequently, agents may incorrectly perceive such states as familiar or well-understood, leading to minimal prediction errors and inadequate novelty detection. Moreover, in stochastic environments where outcomes following the same action can vary unpredictably, prediction errors may persistently remain high due to the intrinsic randomness of the environment. This situation often results in the agent receiving high intrinsic rewards for exploring parts of the environment it may already comprehend, albeit appearing different due to randomness. As a result, these methods often fall short in generating a robust novelty measure capable of accurately distinguishing novel states from those previously encountered by the agent~\cite{yin2021sequential}.
\\
Apart from such prediction error-based approaches,~\citeA{houthooft2016vime} proposed variational information maximizing exploration (VIME), which utilizes a Bayesian neural network to learn a state forward model and considers the information gain from the network parameters as an intrinsic reward function. Although this method does not suffer from the issue of subtle changes in states and high variability in stochastic environments (issue 2) found in other prediction error-based methods, it still struggles with issue 1: it relies on observations, which contain noisy or incomplete information about the true states of the environment.
\\
In contrast, our unified inference approach actively engages with the environment to sequentially infer the information gain, which then determines the intrinsic reward of a state. Our method not only encourages exploration but does so in a manner inherently aligned with the agent's imperative to reduce uncertainty and enhance its understanding of the true states of the environment. Even when states exhibit similarities, the identification of subtle differences that offer new insights about the environment contributes to increased information gain. This approach provides a more nuanced and sensitive measure compared to traditional prediction error-based methods.
Moreover, the incorporation of information gain encompasses the entire distribution of potential states, rather than solely predicting the most probable next state. This broader perspective renders our method more resilient to the variability introduced by stochasticity. The inclusion of information gain in our generalized actor-critic methods, as detailed in Sub-sections~\ref{sub-sec:POMDP_results} and \ref{sub-sec:exploration}, not only enhances robustness to noisy observations but also improves sample efficiency compared to other memory-based and exploration baselines.
\vspace{.1in}\\
\textit{iii) Flexibility in extrinsic reward design:} Another significant advantage of our approach over memory-based algorithms in POMDPs is its flexibility in designing extrinsic rewards. This flexibility stems from AIF's inherent capability to integrate the information gain term. In contrast, memory-based POMDP methods often rely heavily on well-defined extrinsic reward functions, which can be limiting in complex environments where specifying every desired outcome is impractical. Our method allows for a more arbitrary design of extrinsic rewards, reducing dependence on explicit reward structures. By reducing reliance on rigid reward structures, our method provides a robust framework for developing more adaptable decision-making agents capable of operating in a broader array of environments, especially where crafting specific rewards is challenging or rewards are inherently sparse.}

{ In conclusion, our proposed unified inference framework inherently embodies the three aforementioned advantages: computational and memory efficiency, improved exploration through state information gain, and the ability to function without extrinsic rewards. These benefits, derived from AIF, provide a comprehensive approach within the free energy principle, surpassing memory-based RL methods.}
\subsubsection{Belief state inference-based approaches}\label{sub-sec:belief_related}
\citeA{igl2018deep} introduced deep variational RL, which utilizes particle filtering to infer belief states and learn an optimal Markovian policy. This approach minimizes the ELBO to maximize the expected long-term extrinsic reward using the A2C algorithm. Building upon this work, \citeA{lee2020stochastic} and \citeA{han2020variational} proposed stochastic latent actor-critic (SLAC) and VRM, respectively, to extend SAC to POMDPs. SLAC and VRM also employ the ELBO objective to learn belief states and generative models. SLAC focuses on pixel-based robotic control tasks, where velocity information is inferred from third-person images of the robot. While SLAC utilizes the inferred belief state solely in the critic network, VRM incorporates the belief state in both the actor and critic networks. VRM does not utilize the generative model for action selection, and as shown in Sub-section~\ref{sec:MF_AC}, VRM can be derived from our proposed model-free actor-critic. It is important to note that these methods do not explicitly consider the information gain associated with the inferred belief state for action selection.

Some other works, including Dreamer~\cite{Hafner2020Dream} and PlaNet~\cite{hafner2019learning}, learn belief state and generative model along with the extrinsic reward function model. These methods adopt a model-based learning approach that utilizes image observations to maximize the expected long-term reward. However, unlike our inference model, Dreamer does not consider the belief state representation transition model and, therefore, does not incorporate this representation into its variational distribution, generative model, and reward model. As discussed in Sub-section~\ref{sec:MB_AC}, Dreamer can be viewed as a special case of our proposed model-based unified actor-critic. A recent evolution of Dreamer, called Dreamer-v2~\cite{hafnermastering}, has been proposed; however, it is applicable only to discrete state spaces. 
It is important to note that these model-based RL methods do not explicitly take into account the information gained regarding the inferred belief state for action selection. In contrast, our proposed model-based unified actor-critic approach encourages the agent to engage in information-seeking exploration, enabling it to leverage that information to reduce uncertainty about the true states of the environment. 
\\
Additionally, \citeA{ma2020contrastive} and \citeA{laskin2020curl} focused on learning the generative model and belief state in pixel-based environments with image-based observations, utilizing a form of consistency enforcement known as contrastive learning between states and their corresponding observations. However, contrastive learning poses distinct challenges that are beyond the scope of our work. Nevertheless, our idea of using a hybrid model-based and model-free approach in our proposed unified actor-critic is inspired by their work.

{ Recently, there have been several RL works~\cite{yin2021sequential, mazzaglia2022curiosity,klissarov2019variational} that have introduced information gain as an intrinsic reward function for exploration. These methods involve inferring the belief state and learning a generative model by minimizing the ELBO.  They then use actor-critic RL methods to learn an optimal policy that maximizes both the expected long-term extrinsic reward and information gain. In these approaches, the information gain term is added in an ad-hoc manner. In contrast, our method provides a theoretical justification by extending the EFE within AIF to infinite horizon POMDPs with continuous state, action, and observation spaces to incorporate stochastic belief state-action policies.
\\
Therefore, our unified inference method frames belief state inference, generative model learning, and the integration of information gain as an intrinsic reward function, aspects that have previously been considered (partially) heuristic in those studies, under a comprehensive interpretation of free energy principle~\cite{friston2017active}. }
\subsection{AIF approaches for continuous space POMDP}
As mentioned earlier, AIF approaches are mostly limited to discrete spaces or short, finite-horizon POMDPs. This limitation arises from the computational expense of evaluating each plan based on its EFE and selecting the next action from the plan with the lowest EFE.
However, recent efforts have been made to extend AIF to POMDPs with continuous observation and/or action spaces. These efforts include approaches that focus on deriving the optimal distribution over actions (i.e., state-action policy) instead of plans or utilizing MCTS for plan selection. These advancements aim to address the computational challenges associated with scaling AIF to continuous spaces.
\subsubsection{State-action policy learning} 
\citeA{ueltzhoffer2018deep} introduced a method where action $a\t$ is sampled from a state-action policy $\pi(a\t|\s\t)$ instead of placing probabilities over a number of plans $\tilde{a}$. They minimize the EFE by approximating its gradient with respect to $\pi(a\t|\s\t)$. However, this approach requires knowledge of the partial derivatives of the observation given the action, which involves propagating through the unknown transition model $P(\s_{\tau+1}|\s\ta, a\ta)$. To address this challenge, \citeA{ueltzhoffer2018deep} used a black box evolutionary genetic optimizer, which is considerably sample-inefficient.
\\
Later, \citeA{millidge2020deep} proposed a similar scheme that includes the transition model $P(\s_{\tau+1}|\s\ta, a\ta)$. They heuristically set the optimal state-action policy as a Softmax function of the EFE. They introduced a recursive scheme to approximate the EFE through a bootstrapping method inspired by deep Q-learning~\cite{mnih2013playing,mnih2016asynchronous}, which is limited to discrete action spaces. However, our approach is specifically designed for continuous action spaces.
In addition, it is important to note that their approach assumes that the current hidden state is fully known to the agent after inference and does not consider the agent's belief state for action selection in the EFE and policy. In contrast, our approach focuses on learning the belief state representation, which is utilized for action selection in the belief state-action EFE and belief state-action policy. We also provide an analytical demonstration of the recursion in our proposed belief state-action EFE and establish its convergence in Proposition~\ref{pro:unified_Bellman_action} and Theorem~\ref{lem:policy_evaluation}. Furthermore, we prove the expression $\pi^*(a\t|b\t)=\sigma \left( \frac{- \mathcal{G}^{*}(b\t,a\t)}{\beta} \right)$ in Theorem~\ref{the:instant_optimal}. 
\\
In a related context, \citeA{friston2021sophisticated} {explored} a recursive form of the EFE in a problem with discrete state and action spaces, considering it as a more sophisticated form of AIF. This approach involves searching over sequences of belief states and considering the counterfactual consequences of actions rather than just states and actions themselves. Building upon the work of~\citeA{millidge2020deep}, \citeA{mazzaglia2021contrastive} assumed that $D_{KL}[q(.|o\t), P(.|\s\pt,a\pt)]=0$ and focused on contrastive learning for the generative model and belief state in environments with image-based observations.

In a subsequent study, \citeA{da2020relationship} investigated the relationship between AIF and RL methods in finite-horizon fully observable problems modeled as MDPs.  They demonstrated that the optimal plan $\tilde{a}^*$ that minimizes the EFE also maximizes the expected long-term reward in the RL setting, i.e., $\arg \min_{\tilde{a}} {G}^{(\tilde{a})}_{\text{AIF}}(o\t) \subset \arg \max_{\pi} V^{(\pi)}(\s\t)$. \citeA{shin2021prior} extended AIF to continuous observation spaces by showing that the minimum of the EFE follows a recursive form similar to the Bellman optimality equation in RL. Based on this similarity, they derived a deterministic optimal policy akin to deep Q-learning. However, their method is limited to discrete action spaces and only provides a deterministic policy. In contrast, our approach is designed for continuous action spaces and learns a stochastic policy, which enhances robustness to environmental changes.
\subsubsection{MCTS plan  selection}
\citeA{tschantz2020scaling} extended AIF to continuous observation and action space problems by limiting the decision-making time horizon. They parametrized a probability distribution over all possible plans and sampled multiple plans. Each sample was weighted proportionately to its EFE value, and the mean of the sampling distribution was then returned as the optimal plan $\tilde{a}^*$. However, this solution cannot capture the precise shape of the plans in AIF, as $q^*(\tilde{a})=\sigma \left( -{G}^{(\tilde{a})}_{\text{AIF}}(o\t) \right)$, and is primarily suitable for short time horizon problems.
\\
\citeA{fountas2020deep} and \citeA{maisto2021active} proposed an amortized version of MCTS for plan selection. They consider the probability of choosing action $a\t$ as the sum of the probabilities of all the plans that begin with action $a\t$. However, these methods are limited to discrete action spaces and are not applicable to continuous action spaces.
\section{Experimental results}  \label{sec:results}
This section describes the experimental design used in our study to evaluate the effectiveness of our unified inference approach for extending RL and AIF methods to POMDP settings with continuous state, action, and observation spaces. Our principal aim is to compare the performance of our approach with state-of-the-art techniques proposed in the literature. We assess our approach across various tasks characterized by partially observable continuous spaces, which are modeled as continuous space POMDPs.
\\
{ Furthermore, by focusing on the exploration behaviors, we compare the information gain intrinsic reward in our unified inference approach with other exploration methods in the RL literature.}
\\
Furthermore, we delve into the individual contributions of different components within our framework, the extrinsic reward, intrinsic reward, and entropy reward term, to the overall performance. To achieve this, we carry out a series of ablation studies on these tasks, analyzing the effects of removing or modifying these components.
\subsection{Comparative evaluation on partially observable continuous space tasks} \label{sub-sec:POMDP_results}
In this sub-section, we examine the effectiveness of our unified inference approach in overcoming the limitations of existing RL and AIF algorithms when applied to partially observable problems with continuous state, action, and observation spaces. To evaluate its performance, we conducted experiments under two specific conditions of partial observability:
\textit{(i)} Partial state information (partial observations): In this setting, the agent does not have access to complete information about the environment states. This condition emulates scenarios where the agent has limited visibility into the true state of the environment.
\textit{(ii)} Noisy state information (noisy observations):  In this condition, all observations received from the environment are noisy or contain inaccuracies. This situation replicates real-world scenarios where observations are affected by noise or errors, making it challenging to accurately estimate the underlying state of the environment.
\vspace{.1in} \\
\textbf{Environments and tasks:} The experimental evaluations encompassed four continuous space  Roboschool tasks (HalfCheetah, Walker2d, Hopper, and Ant) from the PyBullet~\cite{coumans2016pybullet}, which is the replacement of the deprecated OpenAI Roboschool~\cite{brockman2016openai}. These tasks have high-dimensional state spaces characterized by quantities such as positions, angles, velocities (angular and linear), and forces. Each episode in these tasks terminates on failure (e.g. when the hopper or walker falls over). The choice of these environments is driven by two key factors: (i) they offer challenging tasks with high-dimensional state spaces and sparse reward functions, and (ii) recent efforts have been made to enhance the sample efficiency of model-free and model-based RL methods in the partially observable variants of these benchmarks, providing suitable baselines for comparison.
\\
To create partially observable versions of these tasks, we made modifications to the task environments. The first modification restricts the agent's observations to velocities only, transforming the tasks into partial observation tasks. This modification proposed by~\citeA{han2020variational} is relevant in real-world scenarios where agents may estimate their speed but not have direct access to their position. For the noisy observation versions of the tasks, we added zero-mean Gaussian noise with a standard deviation of $\sigma=0.05$ to the original states returned from the environment. This modification allows us to simulate realistic sensor noise in the environment.  Further details of the modifications made to create partially and noisy observable environments are provided in Appendix~\ref{sec:APP_POMDP_modification}. We denote the partial observation modification and noisy observation modification of the four tasks as \{Hopper, Ant, Walker-2d, Cheetah\}-\{P\} and \{Hopper, Ant, Walker-2d, Cheetah\}-\{N\}, respectively. 
\vspace{.1in}  \\
\textbf{Baselines:} To evaluate the potential of the proposed unified inference framework in extending MDP-based RL methods to POMDPs, we compare the performance of our proposed model-free G-SAC and model-based G-Dreamer algorithms on the \{Hopper, Ant, Walker, Cheetah\}-\{P,N\} tasks with the following state-of-the-art model-based and model-free algorithms from the literature:
\begin{itemize}
\item \textbf{SAC:} SAC~\cite{haarnoja2018soft} is a model-free actor-critic RL algorithm designed for MDPs. We include experiments showing the performance of SAC based on true states (referred to as State-SAC) as an upper bound on performance. The State-SAC serves as an oracle approximation representing an upper bound on the performance that any POMDP method should strive to achieve.
\item \textbf{VRM:} VRM~\cite{han2020variational} is a POMDP-based actor-critic RL algorithm that learns a generative model, infers the belief state, and constructs the state value function in a model-free manner. By comparing our approach with VRM, we can evaluate the impact of the information gain exploration term on model-free learning in POMDP settings.
\item \textbf{Dreamer:} Dreamer~\cite{Hafner2020Dream} is a model-based actor-critic RL method designed for image observations.  It learns a generative model, infers the belief state, and learns the state value function through imagined trajectories using the generative model. To compare our approach with Dreamer, we made some modifications to its implementation (see Appendix~\ref{sec:APP_Implementation Details}). Despite these modifications, we expect the comparison to demonstrate the effects of the information gain exploration term on model-based learning in POMDP settings.
\item \textbf{Recurrent Model-Free:} Recurrent Model-Free~\cite{ni2022recurrent}  is a model-free memory-based RL algorithm for POMDPs. It explores different architectures, hyperparameters, and input configurations for the recurrent actor and recurrent value function across Roboschool environments and selects the best-performing configuration. We consider Recurrent Model-Free as a baseline to compare the performance of memory-based approaches with belief state inference-based methods in partially observable environments.
\end{itemize}
\textbf{Evaluation metrics:}  The performance of RL and AIF algorithms can be evaluated using multiple metrics. In this sub-section, we assess the performance based on the commonly used metric of cumulative extrinsic reward (return) after $1$ million steps. This metric quantifies the sum of all extrinsic rewards obtained by the agent up to that point, providing an indication of the agent's overall success in accomplishing its task within the given time frame. Additionally, we evaluate the sample efficiency of each baseline algorithm by measuring the number of steps required for them to reach the best performance achieved at $1$ million steps.
It should be noted that the number of environment steps in each episode is variable, depending on the termination.
\vspace{.1in} \\ 
\textbf{Experimental setup:}  We implemented SAC, VRM, and Recurrent Model-Free using their original implementations on the {Hopper, Ant, Walker, Cheetah}-{P,N} tasks. For state-SAC, we utilized the results from~\cite{raffin2022smooth}. Regarding Dreamer, we mostly followed the implementation described in the original paper~\cite{Hafner2020Dream}. However, there was one modification: since their work employed pixel observations, we replaced the convolutional neural networks (CNNs) and transposed CNNs with two-layer multi-layer perceptrons (MLPs) consisting of $256$ units each for the variational posterior and likelihood model learning. Feed-forward neural networks were used to represent the actors (policies) and critics (state-action value functions or the negative of belief state-action EFE). To ensure a fair comparison, we maintained identical hyperparameters for the actor and critic models of Dreamer, G-SAC, G-Dreamer, and VRM.
\\
The same set of hyperparameters was used for both the noisy observation and partial observation versions of each task. We trained each algorithm for a total of $1$ million environment steps on each task and conducted each experiment with $5$ different random seeds. For additional information regarding the model architectures, please refer to Appendix~\ref{sec:APP_Implementation Details}. 
\vspace{.1in} \\ 
\textbf{Results:} Table~\ref{Table:summary_POMDP} provides a summary of the results, displaying the mean return averaged over the last $20\%$ of the total $1$ million environment steps across the five random seeds. The complete learning curves can be found in Appendix~\ref{APP:visu}. We will now analyze the quantitative results in the following:
{ \begin{table*}[!t]
\caption{The mean of the return on  Roboschool tasks (partial and noisy observations)  averaged at the last $20\%$ of the total $1$ million environment
steps across $5$ seeds. The model-free baseline State-SAC is used as a reference for the performance. }\label{Table:summary_POMDP}
\centering
\begin{tabular}{|P{2.6cm}|P{1cm}|P{.8cm}|P{1.35cm}|P{.9cm}|P{1.cm}|P{1.4cm}|P{2cm}|}
 \hline
{ Task} & {State-SAC} & {SAC} &   {Dreamer} & {VRM} &  {G-SAC} &  {G-Dreamer} &  {Recurrent Model-Free}
\\ \hline
%
{ HalfCheetah-P} & { $2994 $} & $193 $ & $1926 $ & { $ 2938$} & $\mathbf{3859} $ &  $3413 $ & $1427 $
 \\ \hline
 {HalfCheetah-N} & { $2994$} & $-512$ & $1294 $ &  { $2170$} & $\mathbf{3655} $ &  $3226$ & $408 $
\\ \hline
{{Hopper-P}} & {  $2434 $ } & $724 $ &  $ 2043$ & $1781$  & $2468 $ & $\mathbf{2766}$ &  $1265 $ 
\\ \hline
{{Hopper-N}} & {  $2434 $ } & $466 $ &  { $1674 $} &  $1343 $  & $2308$ & $\mathbf{2448} $ &   $771 $
\\ \hline
{{Ant-P}} & {  $3394 $ } & $411$ &  $847$ &  { $1503 $} & $\mathbf{2743} $ & $2325 $ & $994 $
\\ \hline
{{Ant-N}} & {  $3394 $ } & $328 $ &  $762 $ &  { $1256 $} & $\mathbf{2545} $ & $1996$ & $367 $
\\ \hline
{Walker2d-P} & {  $2225 $ } & $305 $ &  $821 $ &  { $761 $} & $1837 $ & $\mathbf{2173}$ & $210$
\\ \hline
{Walker2d-N} & {  $2225 $ } & $123 $ &  $578 $ &  { $424$} & $1698 $ & $\mathbf{1956}$ & $116$
\\ \hline
\end{tabular}
\end{table*} }
\vspace{.1in}  \\
\textit{(i) Unified inference model successfully generalizes model-free and model-based RL to POMDPs}.
As expected, SAC encountered difficulties in solving the tasks with partial and noisy observations due to its MDP-based implementation. In contrast, our proposed G-SAC algorithm demonstrated superior performance compared to SAC. Furthermore, while Dreamer and VRM are regarded as state-of-the-art methods for POMDP tasks, G-Dreamer and G-SAC consistently outperformed them in all scenarios. This emphasizes the advantages of leveraging the belief state representation in both the actor and critic, along with the information gain term in our G-Dreamer and G-SAC algorithms.   The number of steps required for Dreamer and VRM to match the performance of G-Dreamer and G-SAC at the end of $1$ million steps was significantly higher, indicating that the information gain intrinsic term in G-SAC and G-Dreamer enhances sample efficiency. It should be noted that while G-SAC and G-Dreamer outperform Dreamer and VRM algorithms in terms of final performance and sample efficiency in both the partial observation and noisy observation cases, they achieve these improvements with comparable computational requirements to the compared algorithms. \\
{Remarkably, both our proposed methods, G-SAC and G-Dreamer, performed on par with, or even surpassed, the oracle State-SAC framework, which has access to the true state of the environment. This further confirms: 1) Effectiveness of the perceptual inference and learning method for inferring the belief state-conditioned variational posterior $q({\s_{t}}|o_{t}, b_{t-1})$ and the belief state generative model $P(o\t,\s\t, b\t|a_{t-1}, b_{t-1})$, as well as the belief state representation model for learning a belief representation $h\t$, used in our approaches. 2) Effectiveness of the information-seeking exploration facilitated by the intrinsic reward term in our unified objective function $G^{(\pi)}(b_{t})$.}
\vspace{.05in} \\
\textit{(ii) Belief state inference-based approaches are more robust than memory-based baseline given noisy observations.} Belief state inference-based approaches, such as VRM, Dreamer, and G-SAC, consistently outperform the memory-based baseline, Recurrent Model-Free, in both partial observations and noisy observations settings. This superiority can be attributed to the ability of VRM, Dreamer, and G-SAC to encode the observations and actions history into the variational posterior, enabling more effective encoding of underlying states compared to Recurrent Model-Free when dealing with velocity observations or noisy observations. The performance gap between these approaches becomes more pronounced in the case of noisy observations, highlighting the advantage of inferring the belief state for use in the actor and critic, especially in the presence of observation noise. 
{This advantage is linked to the complexity inherent in noisy environments, where various observations may correspond to a single underlying state due to noise interference. Belief state methods address this uncertainty by maintaining a probability distribution over possible states, allowing them to consolidate information from multiple noisy inputs and update their beliefs accordingly. Thus, even when faced with differing noisy observations that might relate to the same state, belief state methods can assimilate this uncertainty and assign probabilities to each possible state based on the collected data.
Conversely, memory-based methods often falter under such uncertainty. These methods typically depend on storing and processing an extensive history of observations and actions, which can become cumbersome in noisy settings where observations might be ambiguous or misleading. Lacking a structured way to handle and update uncertainty, memory-based approaches may struggle to differentiate between observations leading to the same underlying state, potentially resulting in inferior decision-making.
\\
It is noteworthy that while memory-based methods like Recurrent Model-Free typically result in higher memory consumption due to storing extensive past observations and actions, belief state inference methods such as VRM, G-SAC, and G-Dreamer often require less memory. However, they may necessitate additional computational resources for inference and learning the generative model.
Nonetheless, as elaborated in Appendix~\ref{app: computation}, the computational burden induced by ongoing belief state inference and generative model learning in VRM and our proposed G-SAC and G-Dreamer frameworks is generally lower than that of processing a sequence of past actions and observations through an RNN, as seen in Recurrent Model-Free. This contrast is especially noticeable in large or infinite time horizon problems with high-dimensional action and observation spaces.
 }
\vspace{.05in} \\
\textit{(iii) Maximizing expected long-term information gain improves robustness to the noisy observations}.
While the performance of VRM and Dreamer degraded from the partial observation setting to the noisy observation setting, G-SAC and G-Dreamer were able to maintain comparable performance in the presence of observation noise. This robustness to observation noise can be attributed to the KL-divergence term in the intrinsic information-seeking term, as highlighted by~\citeA{Hafner2022} in the context of divergence minimization frameworks. 

The results presented in this sub-section highlight the effectiveness of the proposed unified inference algorithm in various aspects:
\textit{(i)} It successfully generalizes MDP actor-critic methods to the POMDP setting, allowing for more effective exploration and learning under partial observability.
\textit{(ii)} It outperforms memory-based approaches in scenarios with noisy observations, indicating the advantage of leveraging the belief state representation in handling observation noise.
\textit{(iii)} The inclusion of the information gain intrinsic term into the generalized actor-critic methods improves their robustness to noisy observations.
\subsection{{Comparative evaluation of exploration methods}} \label{sub-sec:exploration}
{
In this sub-section, we evaluate the performance of our unified inference approach, employing information gain as an intrinsic reward for exploration, in contrast to RL methods incorporating alternative exploratory intrinsic rewards. We particularly investigate its efficacy across deterministic and stochastic partially observable environments. To facilitate this analysis, we utilize extrinsic reward-free agents dedicated solely to exploration. This emphasis enables us to comprehensively scrutinize and highlight the exploratory behaviors exhibited by the agents.
\vspace{.1in} \\
\textbf{Environments and tasks:}  
We conduct a series of experiments on a partially observable variant of the MountainCarContinuous-v0 environment, where only the velocity is observable. The problem's state space is continuous and includes the car's position and velocity along the horizontal axis. The action space is one-dimensional, allowing control over the force applied to the car for movement, and the transition function is deterministic.
We chose the MountainCarContinuous-v0 task as it is relatively easy to solve when extrinsic rewards are available. Thus, by considering the reward-free case, we can emphasize the exploration challenge and evaluate the effectiveness of different exploration methods.
As stated in Sub-section~\ref{sub-sec:memeory_related}, exploration based on the prediction error of a state forward model is sensitive to the inherent stochasticity of the environment~\cite{burda2018exploration}. Therefore, we also performed an experiment on the same task but with a stochastic transition function. We used the stochastic version of the environment introduced in~\cite{mazzaglia2022curiosity}, which adds a one-dimensional state referred to as the NoisyState, and a one-dimensional action ranging from $[-1,1]$ that acts as a remote for the NoisyState. When this action's value is higher than $0$, the remote is triggered, updating the NoisyState value by sampling uniformly from the $[-1,1]$ interval. 
\vspace{.1in}\\
\textbf{Baselines:}  We compare G-SAC against the following RL frameworks that incorporate exploratory terms as intrinsic rewards in SAC:
}
\begin{itemize}
\item { \textbf{ICM}: ICM~\cite{pathak2017curiosity} is a prediction error-based RL algorithm that generates intrinsic rewards through a state forward-backward dynamics model. The main source of intrinsic reward is the prediction error from the forward model. The backward transition model supports state feature learning by reconstructing the previous state's features.}
\item { \textbf{RND}:  RND~\cite{burda2018exploration} is a prediction error-based RL method where state features are learned using a fixed, randomly initialized neural network. Intrinsic rewards are calculated based on the prediction errors between the next state features and the outputs of a distillation network. This distillation network is continuously trained to emulate the outputs of the randomly initialized feature network.}
\item { \textbf{VIME} VIME~\cite{houthooft2016vime} employs a Bayesian neural network to learn the forward model. It utilizes intrinsic rewards based on the information gain about the parameters of the Bayesian network. These rewards are quantified by the change in information before and after updating the network with data from new interactions.}
\end{itemize}
{ \textbf{Performance metrics:}  We measure exploration ability directly by calculating an agent’s environment coverage.  Following the approach in~\cite{mazzaglia2022curiosity}, we discretize the state space into $100$ bins and evaluate the coverage percentage of the number of bins explored. An agent visiting a certain bin corresponds to the agent successfully accomplishing a task that requires reaching that particular area of the state space. Hence, it is crucial that a good exploration method can explore as many bins as possible.
\vspace{.1in} \\
\textbf{Experimental Setup:} For G-SAC, we use the same model architectures and hyperparameters described in Sub-section~\ref{sub-sec:POMDP_results} of the main manuscript. Since ICM, RND, and VIME were originally developed for fully observable environments, we adapted them for our partially observable setting of MountainCarContinuous-v0. Specifically, the forward and backward models are reconstructed based on observations rather than fully observable states. Consequently, the state-action policies are changed to history-dependent policies, as the observations no longer have the Markovian property. This adaptation allows both our method and these three baselines to incorporate exploratory intrinsic rewards, enabling a fair comparison.
\\
We then train G-SAC and the baselines on the deterministic and stochastic partially observable MountainCarContinuous-v0 for a total of $100$ episodes. An episode terminates when the agent reaches the goal or the episode length exceeds $1000$ steps.
\vspace{.1in} \\
\begin{figure}[!b]
    \centering 
\begin{subfigure}{0.49\textwidth}
\centering
  \includegraphics[width=1.5\linewidth, height=4.9cm, keepaspectratio=true ]{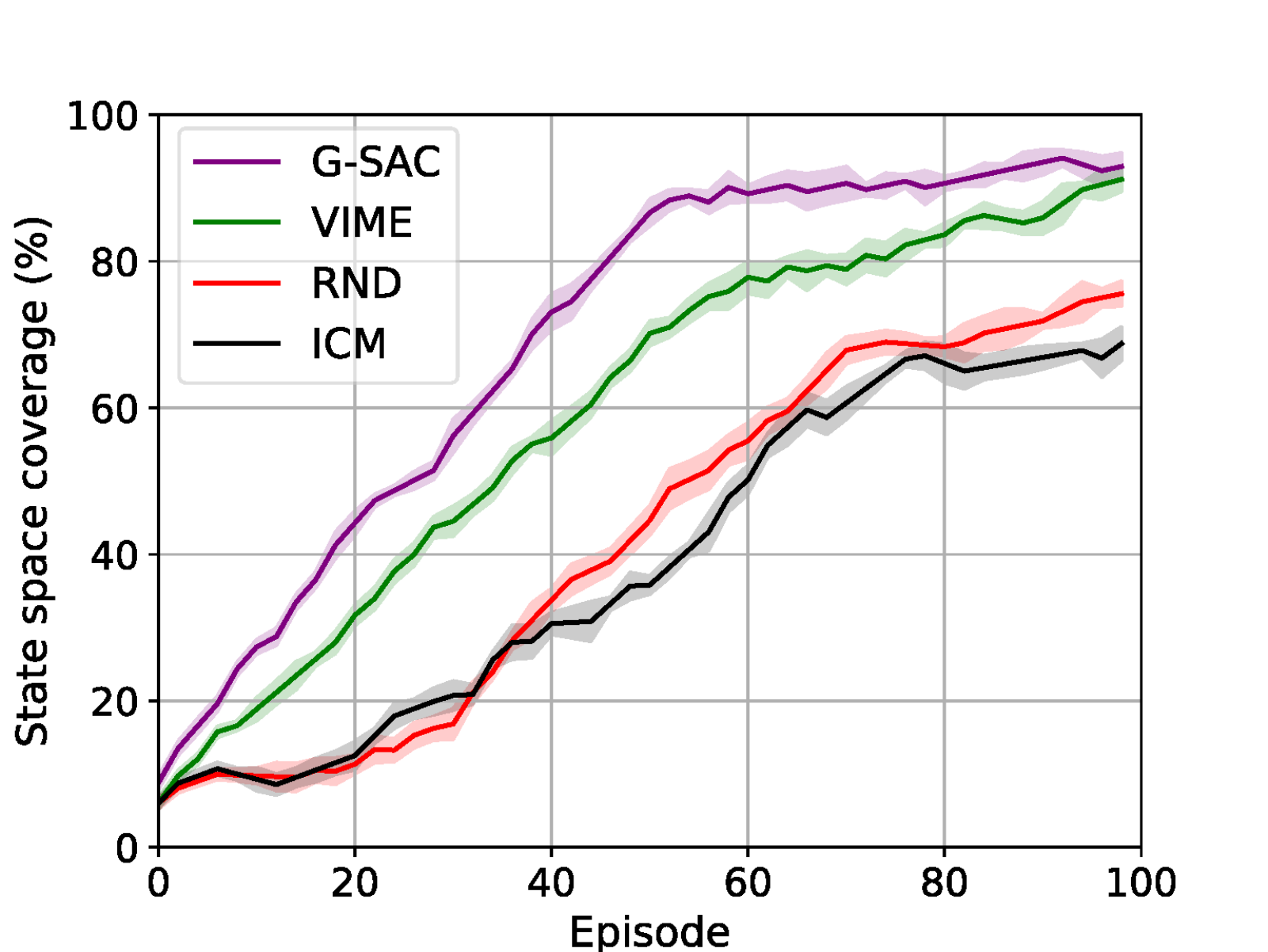}
  \caption{Partially observable deterministic mountain car}
\end{subfigure} 
\hspace{.03in} 
\begin{subfigure}{0.49\textwidth}
\centering
  \includegraphics[width=1.5\linewidth, height=4.9cm, keepaspectratio=true]{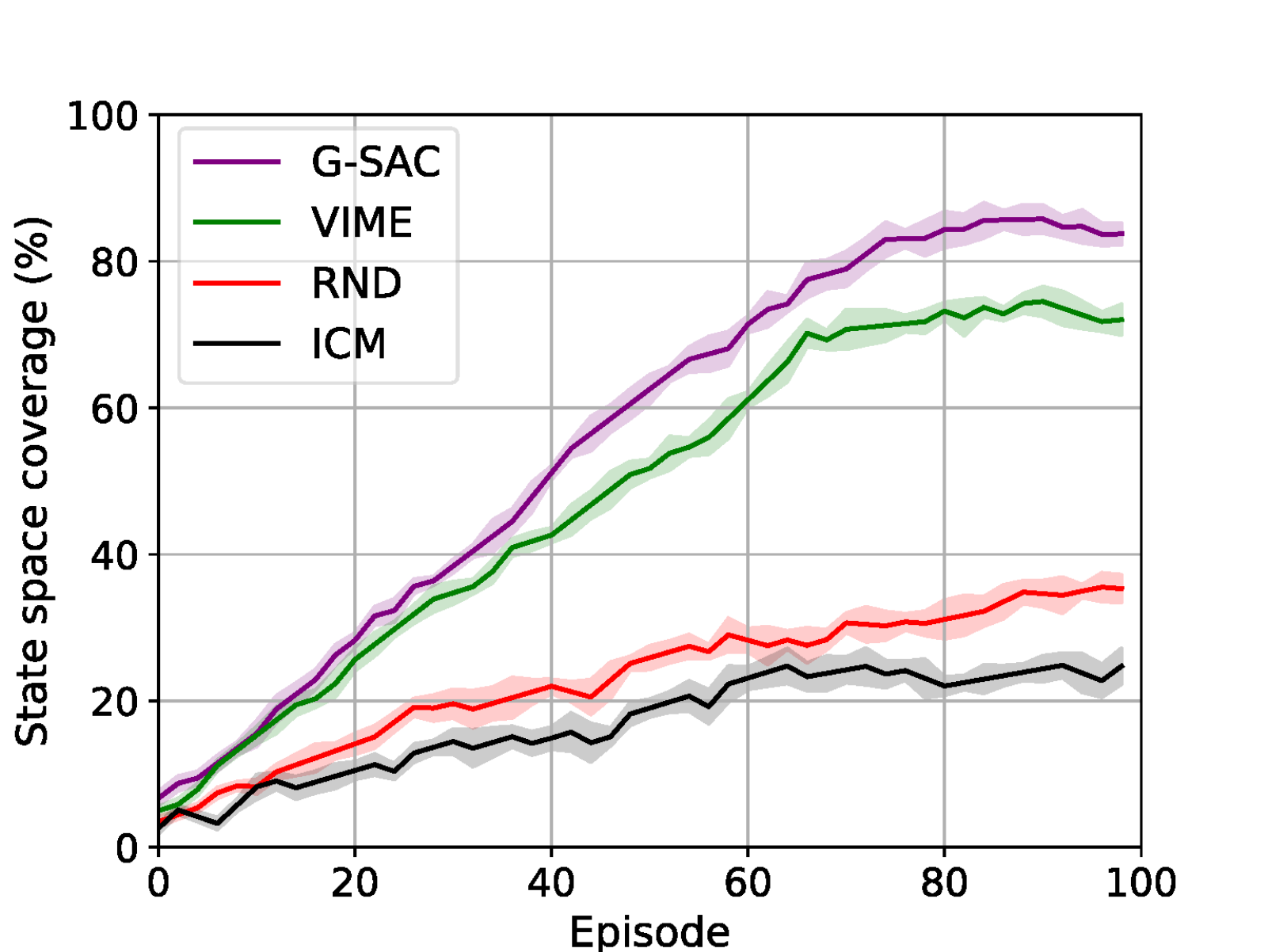}
  \caption{Partially observable stochastic mountain car}
\end{subfigure}
\caption{{ The average state-space coverage in terms of percentage of bins visited by the agents for deterministic and stochastic partially observable MountainCarContinuous-v0. The more state space coverage in an episode, the better the agent explores the environment and thus performs in that episode.}}
\label{fig:exploration_results}
\end{figure}
\textbf{Results:} Fig.~\ref{fig:exploration_results} presents training curves averaged over 10 different random seeds. The results show that methods utilizing intrinsic rewards based on information gain, specifically G-SAC and VIME, learn significantly faster in deterministic environments. This accelerated learning suggests that their exploration mechanisms are more effective than those of ICM and RND, which rely on prediction error-based exploration. The superiority of information gain-based methods largely stems from their effectiveness in handling state similarities within continuous state spaces. Unlike prediction error-based methods that struggle to differentiate between subtly different states, information gain methods assess the entire distribution of possible states, enabling more precise and meaningful exploration even in the presence of similar states. Notably, G-SAC performs slightly better than VIME because it directly assesses uncertainty reduction based on inferred states, while VIME considers uncertainty reduction based on observations, which include the car's velocity but not its position. Therefore, VIME cannot effectively capture uncertainty reduction in the car's position for exploration.}
\\
{Moreover, although VIME explores less than G-SAC, both methods demonstrate resilience to stochasticity. In contrast, the performance of ICM and RND is significantly compromised by randomness, with ICM being the most adversely affected. It is well-documented that intrinsic motivation strategies based on the prediction error of a forward model are vulnerable to the inherent stochasticity of the environment~\cite{burda2018exploration}. This degradation is largely due to the inherent stochasticity of environments where outcomes vary unpredictably following the same action, causing persistently high prediction errors and leading to misguided exploration efforts. This highlights the vulnerability of prediction error-based intrinsic motivation strategies in stochastic settings.}

{In conclusion, information gain-based methods like G-SAC and VIME, which respectively assess the entire distribution of possible latent states and parameters of neural networks generating next states rather than just the most likely ones, demonstrate robustness against subtle changes in an agent’s state and the variability introduced by stochastic conditions. This comprehensive consideration of potential states and parameters significantly improves their effectiveness, especially in stochastic environments with continuous state and action spaces.}
\subsection{Ablation studies}
In this sub-section, we conduct a comprehensive ablation study on the partial observation variants of the four Roboschool tasks discussed in Sub-section~\ref{sub-sec:POMDP_results}, namely \{Hopper, Ant, Walker-2d, Cheetah\}-\{P\}. The primary objective of this study is to gain a deeper understanding of the contribution of each individual component in the proposed unified actor-critic framework. We evaluate the performance based on the average return across  $5$ different random seeds.
\subsubsection{Stochastic policy versus deterministic policy} \label{subsec:stochastic_ablation}
The proposed G-SAC framework learns an optimal stochastic policy by minimizing the loss functions corresponding to the belief state-action EFE update step (Eq.~\eqref{Eq:unified_SAC}) and policy update step (Eq.~\eqref{Eq:unified_SAC_policy}), with the policy entropy term included in both of these loss functions. In the belief state-action EFE update step, the entropy term encourages exploration by reducing the belief state-action EFE in regions of the state space that lead to high-entropy behavior. In the policy update step, the entropy term helps prevent premature policy convergence.
\\
To assess the impact of policy stochasticity (policy entropy term) on G-SAC's performance, we compare it to an algorithm called G-DDPG, obtained by setting $\beta=0$ in G-SAC. G-DDPG is a generalization of the MDP-based Deep Deterministic Policy Gradient (DDPG) algorithm~\cite{silver2014deterministic} to POMDPs, where a deterministic policy is learned by removing the policy entropy term from the belief state-action EFE.
\vspace{.1in} \\ 
\textbf{Experimental setup:} For G-DDPG, we utilize identical hyperparameters and network architectures as those employed in G-SAC. We train G-DDPG on {Hopper, Ant, Walker-2d, Cheetah}-{P} for a total of $1$ million time steps using $5$ different random seeds.
\begin{figure*}[!t]
\centering
\begin{subfigure}{0.42\textwidth}
\!\!\!
\includegraphics[width=1.3\linewidth, keepaspectratio=true]{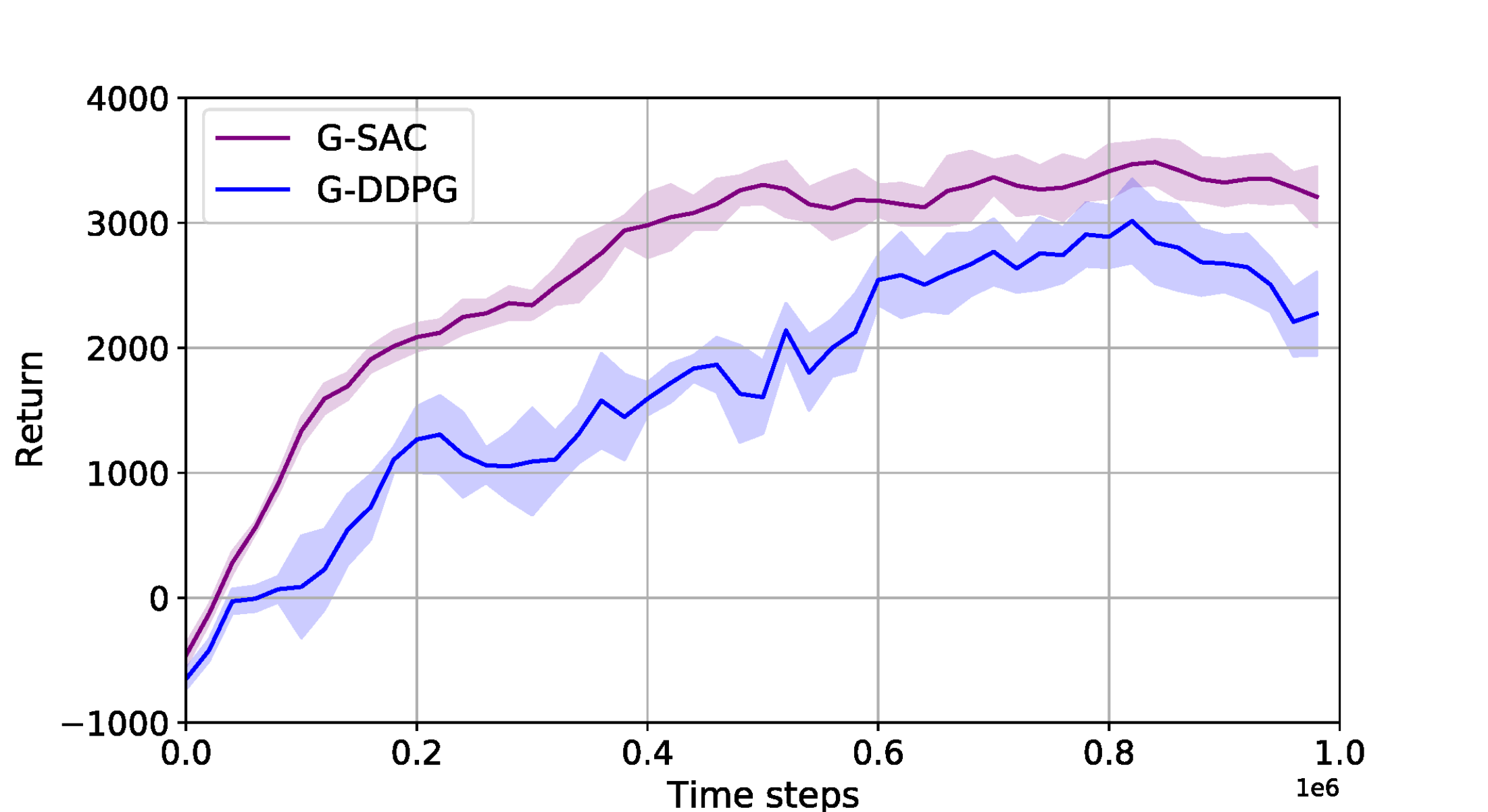}
\caption{HalfCheetah-P} 
\end{subfigure} 
 \quad \quad \quad \quad
\begin{subfigure}{0.455\textwidth}
  \includegraphics[width=1.11\linewidth,keepaspectratio=true]{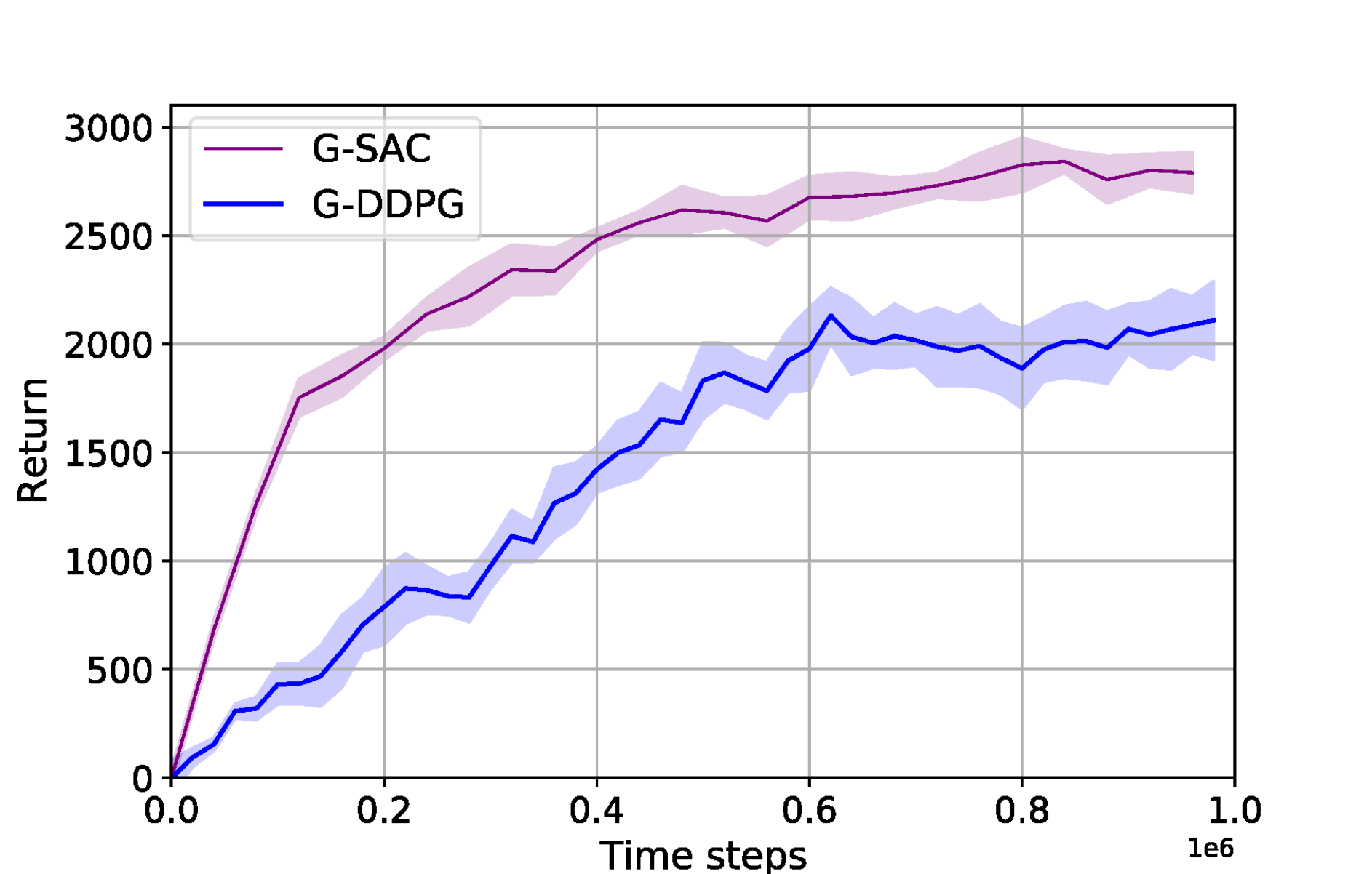}
  \caption{Hopper-P}
\end{subfigure}
\\
\begin{subfigure}{0.444\textwidth}
  \includegraphics[width=1.41\linewidth, height=5.2cm, keepaspectratio=true]{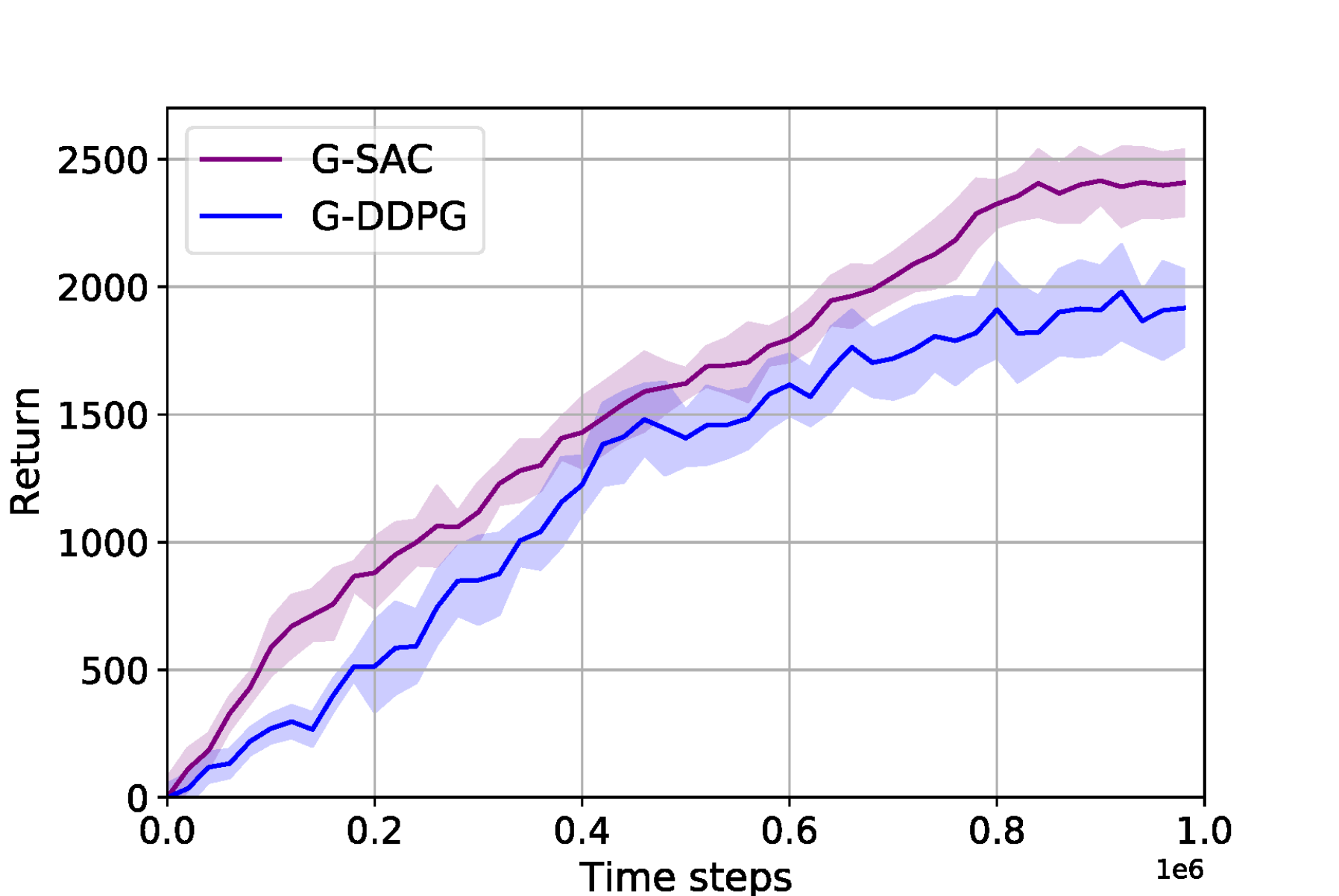}
  \caption {Ant-P}
\end{subfigure}
 \quad \quad \quad 
\begin{subfigure}{0.455\textwidth}
  \includegraphics[width=1.38\linewidth, height=5.1cm, keepaspectratio=true]{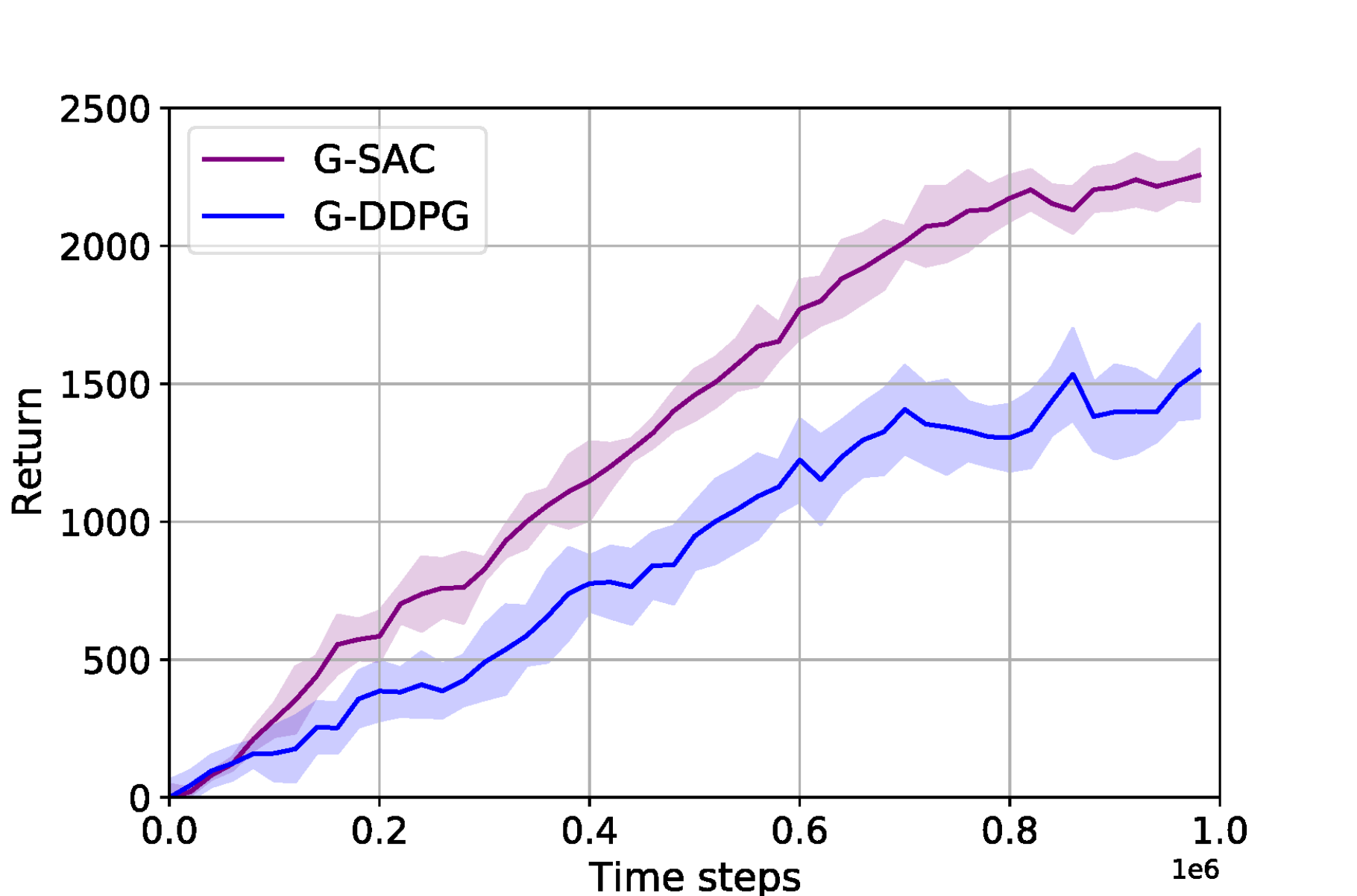}
  \caption {Walker2d-P}
\end{subfigure}
\caption{Ablation study comparing the average return of G-SAC and G-DDPG algorithms across the partial observation version of Roboschool tasks.}
\label{fig:determinstic_rewards}
\end{figure*}
\vspace{.1in} \\ 
\textbf{Results:} Fig.~\ref{fig:determinstic_rewards} presents the performance comparison between G-SAC and G-DDPG. The results indicate that G-DDPG suffers from premature convergence due to the absence of the entropy term. Furthermore, G-DDPG exhibits a higher standard deviation, resulting in reduced stability compared to G-SAC. This finding suggests that learning a stochastic policy with policy entropy maximization in POMDPs with uncertain environmental states can significantly enhance training stability, particularly for more challenging tasks where hyperparameter tuning can be difficult.
\subsubsection{Model-based actor-critic vs hybrid actor-critic} \label{subsec:MB_Hybrid_ablation}
When $\alpha=1$ in the hybrid unified actor-critic method, the critic loss function $L_{\bar{\pi}}^{\text{Hybrid}}({\psi})$ combines the critic loss functions of G-SAC ($L_{\mathcal{G}}^{\text{G-SAC}}({\psi})$) and G-Dreamer ($L_{\mathcal{G}}^{\text{G-Dreamer}}({\psi})$) using a scaling factor $c$. We adopt a hybrid G-Dreamer-SAC approach by setting $c$ to $1$ in $L_{\bar{\pi}}^{\text{Hybrid}}({\psi})$. To assess the influence of the G-SAC component on the performance of the hybrid G-Dreamer-SAC, we compare its performance with that of G-Dreamer, which can be considered a special case of the hybrid unified actor-critic when $c=0$.
\vspace{.1in} \\ 
\textbf{Results:} Fig.~\ref{fig:hybrid_MB} presents a performance comparison between G-Dreamer and G-Dreamer-SAC. G-Dreamer-SAC outperforms G-Dreamer on HalfCheetah-P and Ant-P, while performing on par with G-Dreamer on Hopper-P and Walker2d-P. This discrepancy can be attributed to the higher number of hidden values in the state vector that need to be inferred solely from velocities in HalfCheetah-P and Ant-P, compared to Hopper-P and Walker2d-P. Accurately learning the belief state and generative model becomes more challenging when a larger number of states are unknown. Therefore, relying solely on the learned generative model for actor and critic learning in HalfCheetah-P and Ant-P leads to inaccurate data trajectories. However, by utilizing ground-truth trajectories from the replay buffer, the G-SAC component of G-Dreamer-SAC can provide accurate data to compensate for the compositional errors of the generative models. As a result, G-Dreamer-SAC benefits from the sample efficiency of model-based learning while maintaining the robustness to complex observations, which are characteristic of model-free learning.
\begin{figure*}[!t]
\centering
\!\!\! \begin{subfigure}{0.44\textwidth}
\includegraphics[ width=1.15\linewidth, height=5.2cm, keepaspectratio=true]{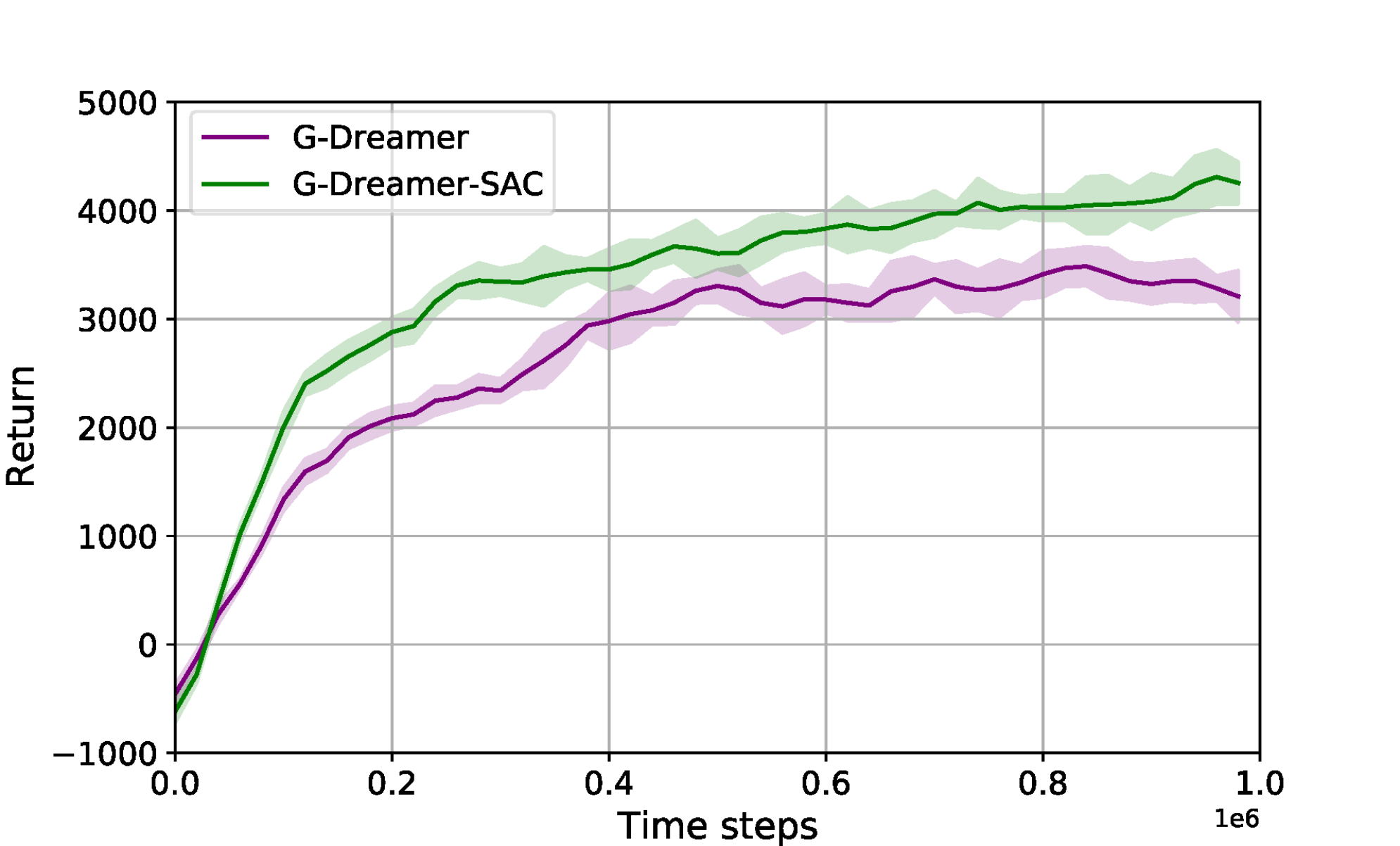}
\caption{HalfCheetah-P} 
\end{subfigure} 
\quad \quad \,\,\,
\begin{subfigure}{0.43\textwidth}
  \includegraphics[width=1.22\linewidth, height=4.8cm, keepaspectratio=true]{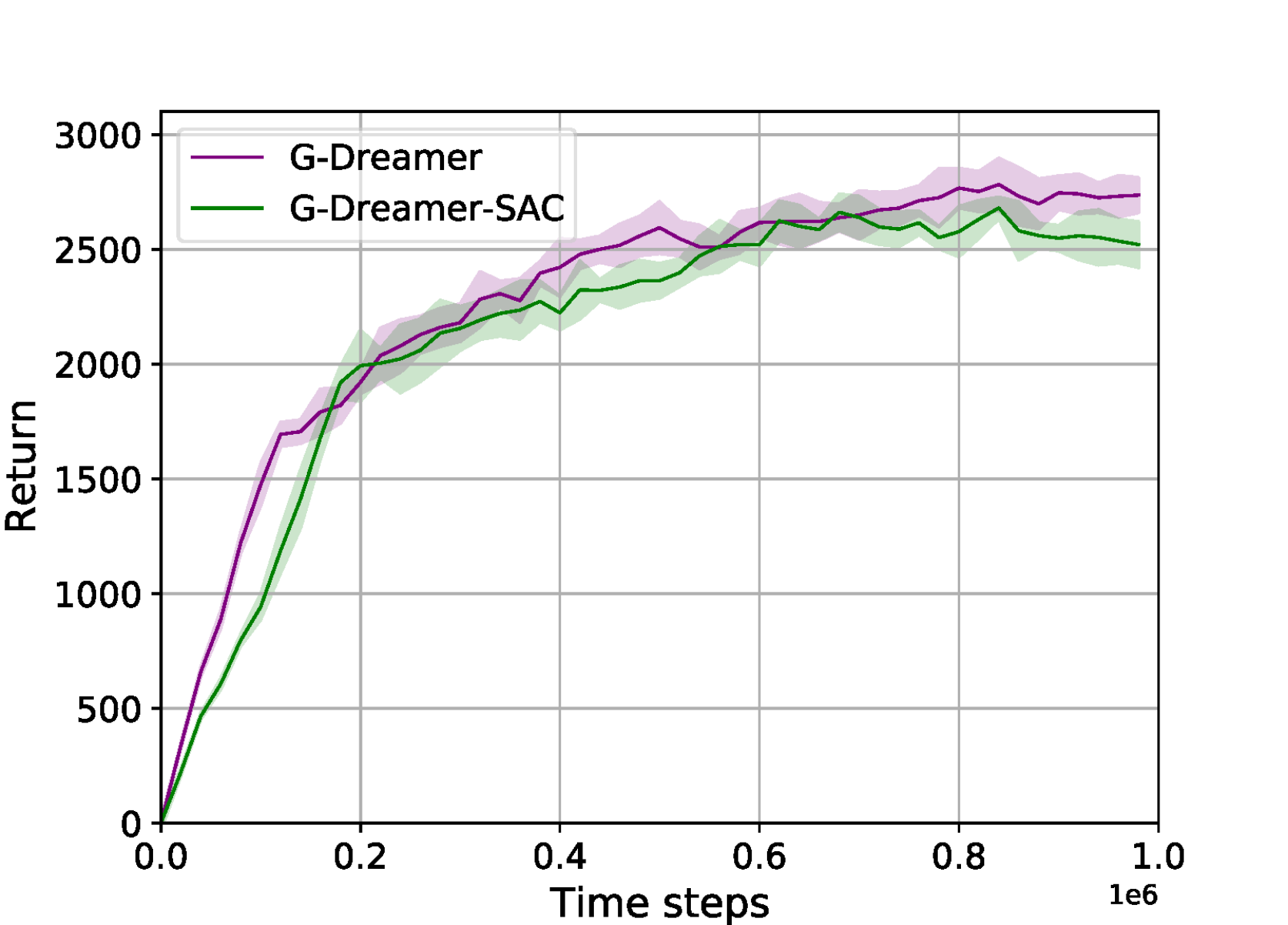}
  \caption{Hopper-P}
\end{subfigure}
\\
\begin{subfigure}{0.44\textwidth}
  \includegraphics[width=1.3\linewidth, height=5.cm, keepaspectratio=true]{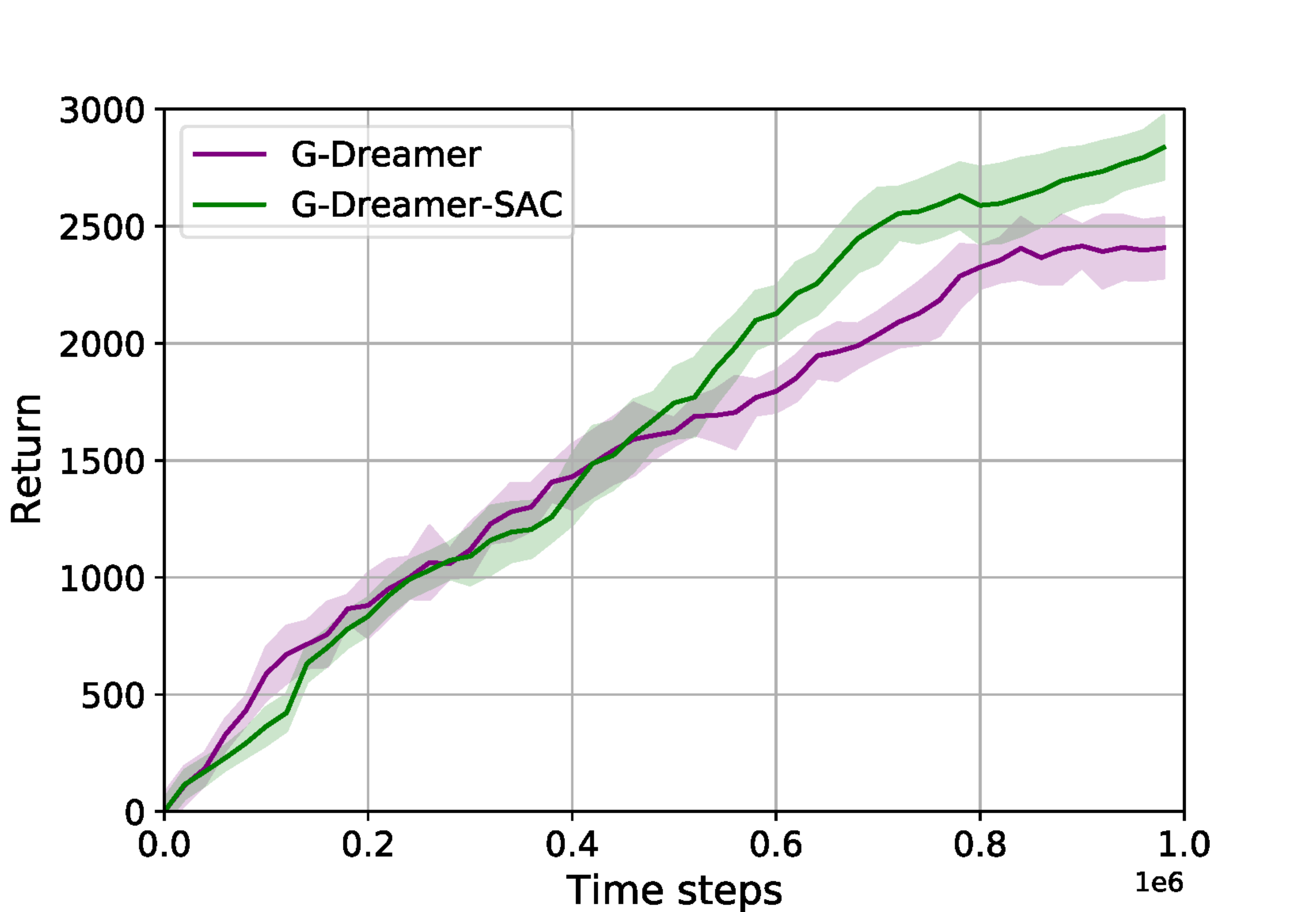}
  \caption {Ant-P}
\end{subfigure}
\quad \quad
\begin{subfigure}{0.44\textwidth}
  \includegraphics[width=1.3\linewidth, height=5.cm, keepaspectratio=true]{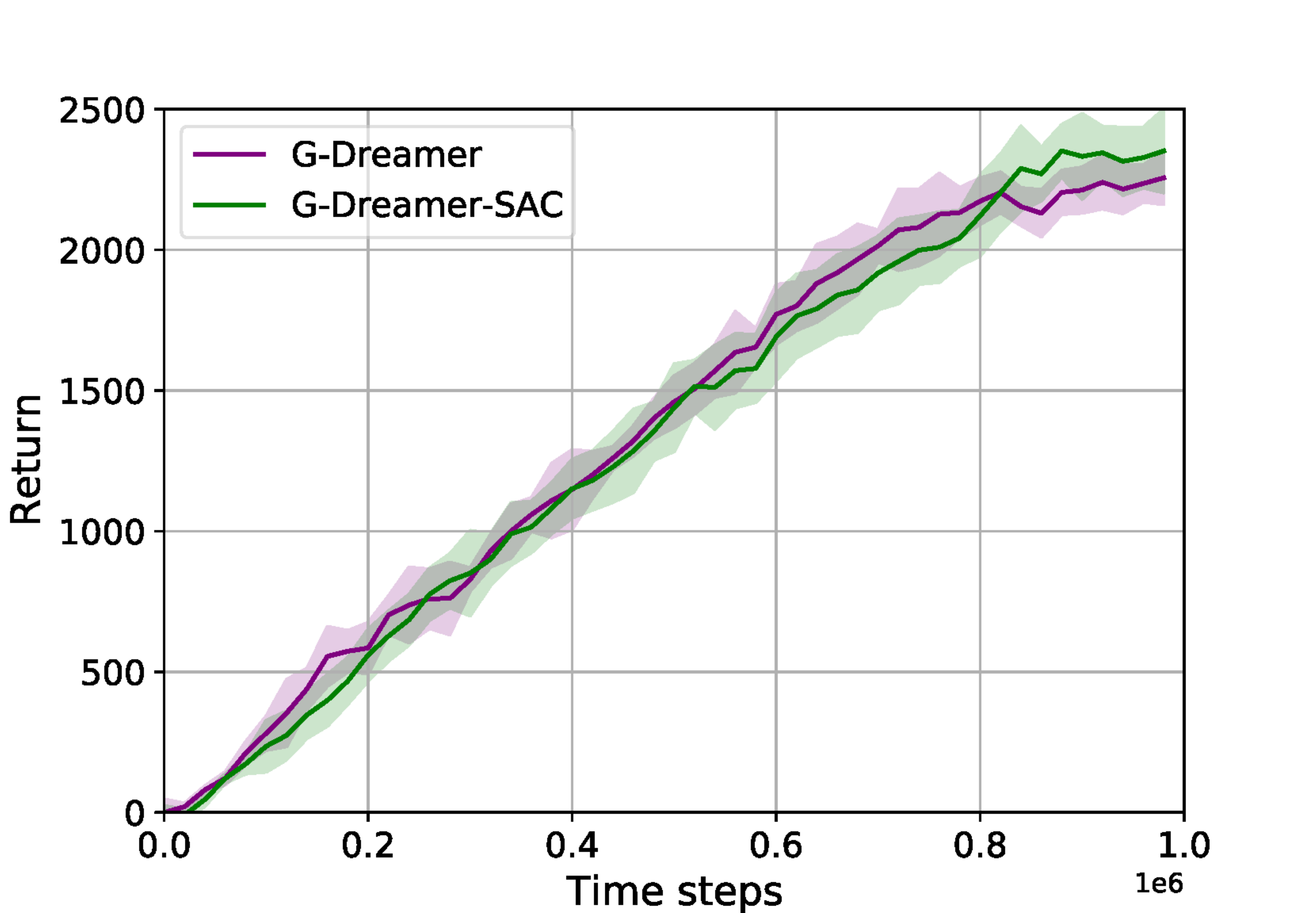}
  \caption {Walker2d-P}
\end{subfigure}
\caption{Ablation study comparing the average return of the hybrid G-Dreamer-SAC algorithm and the model-based G-Dreamer algorithm across four partial observation versions of Roboschool tasks.}
\label{fig:hybrid_MB}
\end{figure*}
\section{Conclusion and Perspectives}\label{sec:con}
In this paper, we extended the EFE formulation to stochastic Markovian belief state-action policies, which allowed for a unified objective function formulation encompassing both exploitation of extrinsic rewards and information-seeking exploration in POMDPs. We then introduced a unified policy iteration framework to optimize this objective function and provided proof of its convergence to the optimal solution.
Our proposed unified policy iteration framework not only generalized existing RL and AIF algorithms but also revealed a theoretical relationship between them, showing that the belief state EFE can be interpreted as a negative state value function. Additionally, our method successfully scaled up AIF to tasks with continuous state and action spaces and enhanced actor-critic RL algorithms to handle POMDPs while incorporating an inherent information-seeking exploratory term. We evaluated our approach on high-dimensional Roboschool tasks with partial and noisy observations, and our unified policy iteration algorithm outperformed recent alternatives in terms of expected long-term reward, sample efficiency, and robustness to estimation errors, while requiring comparable computational resources.
Furthermore, our experimental results indicated that the proposed information-seeking exploratory behavior is effective in guiding agents towards their goals even in the absence of extrinsic rewards, making it possible to operate in reward-free environments without the need for external supervisors to define task-specific reward functions.

However, there are still challenges to overcome, particularly in scaling AIF to high-dimensional environments, such as those based on images. Accurate variational posterior and generative models are required to reconstruct observations in detail. One potential solution is to incorporate a state-consistency loss that enforces consistency between states and their corresponding observations, which has shown promise in self-supervised learning methods~\cite{he2020momentum,grill2020bootstrap}. We plan to explore the combination of our unified inference agents with such learning methods in RL.
\\
Another exciting direction for future research is to investigate the impact of information-seeking exploration in a multi-task setting, where initially explored agents may benefit from transferring knowledge across tasks.

\clearpage
\bibliographystyle{apacite}
\bibliography{ref}
\clearpage


\begin{appendices}
\renewcommand{\thesubsection}{\Alph{subsection}} 
\renewcommand{\thesubsubsection}{\thesubsection.\arabic{subsubsection}}

\titleformat{\section}[block]
  {\normalfont\normalsize\bfseries} 
  {}
  {1em}
  {\!\!\!\!\! \!\!}  

\titleformat{\subsection}[block]
  {\normalfont\normalsize\bfseries} 
  {Appendix \thesubsection}
  {1em}
  {}

\counterwithin{figure}{subsection}
\counterwithin{table}{subsection}
\counterwithin{equation}{subsection}
\counterwithin{algorithm}{subsection}

\section{Appendices}
\subsection{{Markovian belief state-action versus history-dependent policies: memory and computational complexity}} \label{APP:belief_history}
{ In POMDPs, each time step introduces one new action and one new observation to the history, adding $D_{\mO} + D_{\mA}$-dimensional data to the existing history. Consequently, starting with a $D_{\mO}$-dimensional observation $o_0$, the total dimensionality of the history by time $t$ is given by  $D_{\mO}+ (D_{\mO} + D_{\mA})\times t$. This increase in the dimensionality of inputs necessitates more substantial memory for storage and larger matrix operations at each step, escalating computational demands.
\\
Furthermore, the addition of one observation and one action to the history at each time step leads to an exponential increase in the number of possible histories.  Specifically, this number expands at each step by a factor of $(|\Omega| \times |\mathcal{A}|)$, reflecting the agent's $|\mathcal{A}|$ possible actions and $|\Omega|$ possible observations. Thus, starting from any of  $|\Omega|$ potential initial observations $o_0 \in \Omega$, the number of possible histories grows exponentially as $(|\Omega| \times |\mathcal{A}|)^t$ across $t$ time steps, resulting in a total number of histories at time $t$ of $|\Omega| \times (|\Omega| \times |\mathcal{A}|)^t$. Consequently, the size of the policy space, which represents the number of possible policies derived from these histories, also undergoes the same exponential growth.
This exponential growth over time significantly heightens memory and computational requirements to store and evaluate each possible history and its corresponding policy. 
These challenges are particularly pronounced in scenarios with high-dimensional action and observation spaces, where $D_{\mA}$ and $D_{\mO}$  are large, and in environments with extensive or continuous action and observation spaces, where $|\mA|$ and $|\Omega|$ are also large. Additionally, these challenges are accentuated in environments with long time horizons, where $t$ is substantial.}

{ Unlike history-dependent policies, where the length of history grows linearly, belief states in belief state-action policies maintain a constant size, $|\mathcal{S}|$, at each time step. Moreover, the size of all possible belief states (policy space) remains fixed as $|\mB|$. Consequently, the memory required to store each belief state and its corresponding policy does not inherently increase over time, offering significant memory efficiency advantages over history-dependent approaches, particularly in environments with infinite horizons or high-dimensional action and observation spaces. However, storing belief states can be memory-intensive, especially in large or continuous state spaces where $|\mS|$ is large. Nonetheless, this memory requirement is generally more manageable compared to the expansion of history dimension and policy space in history-dependent policies~\cite{yang2021recurrent}. }

{To infer the belief state using Bayesian inference (as detailed in Sub-section~\ref{sub-sec:inference}),  the agent must learn the structure of the environment's observation and transition functions through a generative model. While belief state-action policies offer memory efficiency advantages, their computational complexity introduces distinct challenges, primarily contingent on the learning process for the generative model and the belief state update mechanism. Nonetheless, particularly in environments with infinite horizons, belief state-action policies often achieve greater computational efficiency over time by utilizing fixed-dimension inputs, which streamline processing and ensure consistent computational demands regardless of the duration of operation.}
\subsection{{Reward model}} \label{Sec:RM}
{As mentioned in Sub-section~\ref{Sec:GM},  interacting with a POMDP through a policy $\pi$ generates the sequence $(s_0,o_0,a_0, s_1, o_{1}, a_{1}, \ldots )$. During the interaction, the agent receives reward values $(r_0, r_1, \ldots, r\t, \ldots)$, where $r\t$ is generated according to the reward function $U(r_t|s_t,a_t)$.
Given  the trajectory $(s_0,a_0,s_1,a_1,\ldots)$, we can calculate the probability distribution of receiving the reward sequence $(r_0, r_1, \ldots)$ as follows:
\begin{eqnarray}
 \!\!\!\!\!\! p(r_0, r_1, ... |s_0,a_0, s_1,a_1,, ...) &\!\!\!=\!\!\!& \prod_{k=0}^{\infty} p(r\k|{s}\k,a\k) =\prod_{k=0}^{T} U(r\k|{s}\k,a\k). \label{Eq:R_GP} 
\end{eqnarray}
Due to the agent's lack of knowledge about the true reward function $U(r_t|s_t,a_t)$, it learns a reward model $P(r_t|s_t,a_t)$ in a supervised manner through environmental interactions, enabling the agent to estimate the reward sequence's distribution:
\begin{eqnarray}
 \!\!\!\!\!\! p(r_0, r_1, ...|s_0,a_0, s_1,a_1,, ...) &\!\!\! \approx \!\!\!&  \prod_{k=0}^{\infty} P(r\k|{s}\k,a\k). \label{Eq:R-GM}
\end{eqnarray}
}

{It is worth noting that learning the reward model poses challenges due to reward sparsity and its scalar nature. While the generative model's learning is crucial for a belief state-action policy, our study explores scenarios with and without reward model learning, as detailed in Sub-section~\ref{Seb_unified _AC}.}
%
\subsection{{The POMDP regularity assumptions}} \label{app:regularity}
{ The regularity assumptions play a critical role in establishing the convergence of algorithms that solve POMDPs~\cite{puterman2014markov, paternain2020policy, dufour2012approximation, dufour2013finite}. These assumptions ensure the existence of solutions for specific types of optimization problems that arise in POMDPs. The compactness assumptions on the state, action, observation, and reward spaces are standard requirements in the continuous space MDP and POMDP literature~\cite{puterman2014markov, paternain2020policy, dufour2012approximation, dufour2013finite}. They also facilitate the numerical implementation of MDP and POMDP problems with continuous spaces. For instance, in grid-based methods, which involve dividing the continuous state or action space into a finite set of discrete grid points, the compactness assumption ensures that the discretization is not excessively coarse, which could lead to inaccuracies in the solutions.  It's worth noting that the compactness assumption of the extrinsic reward space $\mathcal{R}$ is not a strict requirement as it is user-defined.
\\
The Lipschitz continuity assumption ensures that the transition function, observation function, and reward function are well-behaved and exhibit smoothness and boundedness. This assumption prevents overly steep or abrupt changes in these functions, which could result in instability and convergence issues. The Lipschitz continuity assumption is crucial for effective and reliable decision-making in continuous spaces, as it guarantees the stability and convergence of the algorithms~\cite{puterman2014markov, paternain2020policy, dufour2012approximation, dufour2013finite}.}

\subsection{Theoretical analysis}
This section presents a detailed theoretical analysis of the main paper. To facilitate understanding, Table~\ref{Table:1} lists the relevant notations used in our proposed unified inference framework.
\begin{table*}[!t]
\caption{ Overview of the notations used in our proposed unified inference framework.}\label{Table:1}
\centering
\begin{tabular}{p{4.5cm} |l}
\hline
\textbf{\!\!\!\ Expression\!\!\!} &  \textbf{Explanation} \\
\hline
$T \in \{0, \mathbb{N} \}$ & Finite time horizon in AIF\\
$t \in \{ 0, 1,  ..., \}$ & Current time step\\
$\tau \in \{t, t+1,   ..., \}$ & Future time step\\
$\s_t \in \mS$ & Current hidden state\\
$D_{\mS} \in \mathbb{N}$ & Dimension of state space \\
$\mS \subset \mathbb{R} ^{D_{\mS}}$ & Continuous state space\\
$b_t \in \mB$ & Current belief state\\
$\mB$ & Continuous belief state space\\
$h_t \in \mB$ & Representation of belief state $b\t$\\
$D_{\mathcal{H}} \in \mathbb{N}$ & Dimension of belief state representation \\
$\mathcal{H} \subset \mathbb{R} ^{D_{\mathcal{H}}}$ & Continuous belief state representation space\\
$o_t \in \mO$ &   Current observation \\
$D_{\mO} \in \mathbb{N}$ & Dimension of observation space \\
$\mO \subset \mathbb{R} ^{D_{\mO}}$ & Continuous observation space\\
$a_t \in \mA$ &   Current action (decision) \\
$D_{\mA} \in \mathbb{N}$ & Dimension of action space \\
$\mA \subset \mathbb{R} ^{D_{\mA}}$ & Continuous action space\\
$r_t \in \mathbb{R}$ &   Current (extrinsic)  reward \\
$\mR \subset \mathbb{R} $ & Continuous (extrinsic) reward space\\
$\tilde{a}=(a_t, a\nt, ... ,a_{T-1} )$ & Plan in AIF (sequence of future actions including current time step $t$) \\
$\bar{\Pi}$ & A set of desired stochastic Markovian belief state-action policies $\bar{\pi}$ \\
$\mathbb{E}[.]$ &  Expectation function \\
$\text{log}$ &  Natural log \\
$P(.)$ &  Agent's probabilistic model\\
$H(.)$ &  Shannon entropy \\
$D_{KL}(.)$ & Kullback Leibler divergence \\
$\sigma(.)$ & Boltzmann (softmax) function \\
$P(o_t, \s_t, b\t|a\pt, b\pt, )$ &  Belief state generative model\\
$P(o_t| \s_t, b\pt)$ &   Belief state-conditioned observation (likelihood) model  \\
$P(\s\t|b\pt, a\pt)$ &  Belief state-conditioned transition model \\
$q(\s_{t}| o_{t}, b\pt)$ &  Belief state-conditioned variational posterior distribution\\
$F\t $ &  Original definition of VFE (i.e., ELBO) in AIF~\cite{friston2017active}\\
$F^{\text{VRNN}}$ &  Extension of VFE (i.e., ELBO) in VRNN~\cite{chung2015recurrent}\\
${G}_{\text{AIF}}(o\t) $  & Original definition of the EFE in AIF~\cite{friston2017active} \\
${G}_{\text{AIF}}^{(\tilde{a})}(o\t) $  & Original definition of the EFE~\cite{friston2017active} for plan $\tilde{a}$   \\
${G}^{(\pi)}(b\t) $  & Our proposed belief state EFE for belief state-action policy $\pi$ \\
$\mathcal{G}^{({\pi})}(b\t,a\t) $  & Our proposed belief state-action EFE for belief state-action policy ${\pi}$  \\
\end{tabular}
\end{table*}
\FloatBarrier
\begin{table*}[h] 
\nonumber
\centering
\begin{tabular}{p{4.5cm} |l}
${{\pi}}^{*}=\arg\min_{{\pi}} {G}^{(\pi)}(b\t)$  & Optimal unconstrained belief state-action policy\\
${\bar{\pi}}^{*}=\arg\min_{\bar{\pi} \in \bar{\Pi}} {G}^{(\bar{\pi})}(b\t)$  & Optimal constrained belief state-action policy\\
$V^{(\pi)}(\s\t) $ & State value function in RL for state-action policy ${\pi}$  \\
$Q^{(\pi)}(\s\t,a\t) $ & State-action value function in RL  for state-action policy ${\pi}$\\

$\gamma \in [0,1)$ & Discount factor \\
$0 \leq \alpha, \beta, \zeta   < \infty$ & Scaling factors \\
{ $\Delta(.)$} &  {The set of all probability distributions over the specified set} \\
$|.|$ &  Cardinality operator \\
\end{tabular}
\end{table*}
\subsubsection{Proof of Theorem.~\ref{the:EFE_stoch}}\label{APP:EFE_stoch}
\begin{proof}
By conditioning the EFE ${G}_{\text{AIF}}(o\t)$ in Eq.~\eqref{Eq:EFE_def} on the belief state-action policy $\pi$ and taking the limit as $T \rightarrow \infty$, we obtain the resulting expression ${G}_{\text{Unified}}^{(\pi)}(o\t)$ as:
\begin{eqnarray}
 {G}_{\text{Unified}}^{(\pi)}(o\t)  &=& \mathbb{E}_{q({\s}_{t:\infty}, {o}_{t+1:\infty}, {a}_{t:\infty}|o\t, \pi)} \bigg[ \sum_{\tau =t}^{\infty}  \text{log} \frac{q({a}_{t:\infty}| \pi) P(\s_{\tau+1}|s\ta, a\ta)}{\tilde{P}({o}_{\tau+1}|{a}\ta)q({\s}_{\tau+1}|o_{\tau+1}) } \bigg]. \label{Eq:EFE_def3}
\end{eqnarray}
Using the biased generative model from AIF literature~\cite{friston2017active,pezzulo2015active}, the variational distribution $q({\s}_{t:\infty}, {o}_{t+1:\infty}, {a}_{t:\infty}|o\t, \pi)$ can be factored as:
\begin{eqnarray}
\!\!\!\! \!\!\!\! q({\s}_{t:\infty}, {o}_{t+1:\infty}, {a}_{t:\infty}|o\t, \pi) &=&  q({a}_{t:\infty}| \pi)  \prod_{\tau =t}^{\infty} q({\s_{\tau}}|o_{\tau})\tilde{P}({o_{\tau+1}}|{\s_{\tau}}, a\ta). \label{Eq:factor2}
\end{eqnarray}
Since action $a\t$ is chosen based on the stochastic belief state-action policy $\pi$ and does not depend on the other actions, we have $q({a}_{t:\infty}| \pi)=\prod_{\tau =t}^{\infty} q(a_{\tau}| \pi)$. Consequently, Eq.~\eqref{Eq:factor2} can be rewritten as:
\begin{eqnarray}
\!\!\!\! \!\!\!\! q({\s}_{t:\infty}, {o}_{t+1:\infty}, {a}_{t:\infty}|o\t, \pi) &=&   \prod_{\tau =t}^{\infty} q(a_{\tau}| \pi) q({\s_{\tau}}|o_{\tau})\tilde{P}({o_{\tau+1}}|{\s_{\tau}}, a\ta). \label{Eq:fact_exten}
\end{eqnarray}
Now, by substituting $ q({a}_{t:\infty}| \pi)$  with $\prod_{\tau =t}^{\infty} q(a_{\tau}| \pi)$ and  $q({\s}_{t:\infty}, {o}_{t+1:\infty}, {a}_{t:\infty}|o\t, \pi)$ with the right-hand side of Eq.~\eqref{Eq:fact_exten}, ${G}_{\text{Unified}}^{(\pi)}(o\t)$ in Eq.~\eqref{Eq:EFE_def3} can be rewritten as:
\begin{eqnarray}
     \!\!\!\! \!\!\!\!  \!\!\!\! \!\!\!\!  \!\!\!\!  {G}_{\text{Unified}}^{(\pi)}(o\t)   &\!\!\!=\!\!\!& \mathbb{E}_{ \prod_{\tau =t}^{\infty} q(a_{\tau}| \pi) q({\s_{\tau}}|o_{\tau})\tilde{P}({o_{\tau+1}}|{\s_{\tau}}, a\ta)} \bigg[ \sum_{\tau =t}^{T-1} \text{log} q(a_{\tau}| \pi)      \\
 \!\!\!\!  \!\!\!\!  \!\!\!\!  \!\!\!\!  \!\!\!\!  &\!\!\!-\!\!\!& \text{log} \tilde{P}({o}_{\tau+1}|\pi, s\ta) + \text{log} \frac {P({\s_{\tau+1}}|a\ta, s_{\tau}) }{q({\s_{\tau+1}}|o_{\tau+1})}  \bigg].  
\end{eqnarray}
Furthermore, since policy $\pi$ selects action $a\ta$ based on the belief state $b\ta$, i.e., $a\ta \sim \pi(a\ta|b\ta)$, we can express ${G}_{\text{Unified}}^{(\pi)}(o\t)$ using the law of total expectation as follows:
\begin{eqnarray}
 \!\!\!\!  {G}_{\text{Unified}}^{(\pi)}(o\t)   &\!\!\!\!\!=\!\!\!\!\!& \mathbb{E}_{P(b_{t:\infty}) } \bigg[  \mathbb{E}_{q({\s_{t}}|o_{t}) \prod_{\tau =t}^{\infty} q(a_{\tau}| \pi, b\ta) q({\s_{\tau+1}}|o_{\tau+1}, b\ta)\tilde{P}({o_{\tau+1}}|{\s_{\tau}}, a\ta, b\ta)} \big[ \sum_{\tau =t}^{\infty} \text{log} q(a_{\tau}| \pi,  b\ta)  \nonumber    \\
\!\!\!\! &\!\!\!\!\!-\!\!\!\!\!& \text{log} \tilde{P}({o}_{\tau+1}|a\ta, s\ta, b\ta) + \text{log} \frac {P({\s_{\tau+1}}|a\ta, s_{\tau}, b\ta) }{q({\s_{\tau+1}}|o_{\tau+1}, b\ta)} \big]  \bigg] \\
\!\!\!\! &\!\!\!\!\! \overset{\mathrm{(a)}}{=}\!\!\!\!\!& \mathbb{E}_{P(b_{t:\infty}) } \bigg[ \mathbb{E}_{ q({\s_{t}}|o_{t}) \prod_{\tau =t}^{\infty} \pi(a_{\tau}| b\ta) q({\s_{\tau+1}}|o_{\tau+1}, b\ta)\tilde{P}({o_{\tau+1}}|{\s_{\tau}}, a\ta, b\ta)} \big[ \sum_{\tau =t}^{\infty} \text{log} \pi(a_{\tau}| b\ta)      \nonumber  \\
\!\!\!\! &\!\!\!-\!\!\!& \text{log} \tilde{P}({o}_{\tau+1}|a\ta, s\ta, b\ta) + \text{log} \frac {P({\s_{\tau+1}}|a\ta, s_{\tau}, b\ta) }{q({\s_{\tau+1}}|o_{\tau+1}, b\ta)} \big]  \bigg]
\\
\!\!\!\! &\!\!\!\!\! \overset{\mathrm{(b)}}{=} \!\!\!\!\!& \mathbb{E}_{P(b_{t:\infty}) } \bigg[ \mathbb{E}_{ \prod_{\tau =t}^{\infty} \pi(a_{\tau}| b\ta) q({\s_{\tau+1}}|o_{\tau+1}, b\ta)\tilde{P}({o_{\tau+1}}| a\ta, b\ta)} \big[ \sum_{\tau =t}^{\infty} \text{log} \pi(a_{\tau}| b\ta)       \nonumber \\
&\!\!\!-\!\!\!& \text{log} \tilde{P}({o}_{\tau+1}|a\ta,  b\ta) + \text{log} \frac {P({\s_{\tau+1}}|a\ta,  b\ta) }{q({\s_{\tau+1}}|o_{\tau+1}, b\ta)} \big]  \bigg]
\\
\!\!\!\! &\!\!\!\!\!\overset{\mathrm{(c)}}{=} \!\!\!\!\!&  \mathbb{E}_{ P(b_{t}| s\t, o_{t}, b_{t-1}, a_{t-1}) \prod_{\tau =t}^{\infty} P(b_{\tau+1}| s_{\tau+1},o_{\tau+1}, b_{\tau}, a_{\tau}) \pi(a\ta|b\ta) q(s_{\tau+1}| o_{\tau+1}, b_{\tau}) \tilde{P}(o_{\tau+1}|a\ta, b\ta)   }    \nonumber  \\
\!\!\!\!\!\!\!\!\! &&  \bigg[ \sum_{\tau =t}^{\infty} \text{log} \pi(a_{\tau}| {b_{\tau}}) - \text{log} \tilde{P}({o}_{\tau+1}|a\ta, b\ta) + \text{log} \frac {P({\s_{\tau+1}}|a\ta, b_{\tau}) }{q({\s_{\tau+1}}|o_{\tau+1}, b_{\tau})}  \bigg], \label{Eq:Beleif_EFE_beleif4}  
\end{eqnarray}
%
where (a) follows from $q(a_{\tau}| \pi, b\ta) = \pi(a\ta|b\ta)$, and (b) follows because the belief state $b\ta$ contains the necessary information about the hidden state $s\ta$ that is relevant for predicting $s_{\tau+1}$ and $o_{\tau+1}$. As a result, the direct dependence on $s\ta$ becomes redundant when conditioned on the belief state $b\ta$. (c) follows due to the factorization resulting from the belief state update function in Eq.~\eqref{Eq:beleif-state} of the main manuscript:
\begin{eqnarray}
\!\!\!\! P(b_{t:\infty}) \!\!& \!\! = \!\! &\!\! \prod_{\tau =t}^{\infty}  P(b_{\tau}| b_{\tau-1}) ] = P(b_{t}| s\t, o_{t}, b_{t-1}, a_{t-1}) \\
\!\!& \!\! \times \!\! &\!\!  \prod_{\tau =t}^{\infty} \mathbb{E}_{ \prod_{\tau =t}^{\infty} \pi(a\ta|b\ta) q(s_{\tau+1}| o_{\tau+1}, b_{\tau}) \tilde{P}(o_{\tau+1}|a\ta, b\ta)   } [{P(b_{\tau+1}| s_{\tau+1}, o_{\tau+1}, b_{\tau}, a_{\tau})}]. \nonumber
\end{eqnarray}
As stated in Assumption~\ref{ASS:percep}, following AIF, we assume that prior to action selection at time step $t$, the agent performs inference and approximates its belief state $b_t$ by a variational posterior (i.e., here $q(.|o_t, b\pt)$). Hence, the belief state $b_t$ is uniquely determined, and thus $P(b_{t}| s\t, o_{t}, b_{t-1}, a_{t-1})$ is equivalent to $\delta(b_t- q(.|o_t, b_{t-1}))$. Therefore, taking the expected value with respect to $P(b_{t}| s\t, o_{t}, b_{t-1}, a_{t-1})$ in Eq.~\eqref{Eq:Beleif_EFE_beleif4} simply evaluates ${G}_{\text{Unified}}^{(\pi)}(o\t)$ at the given $b\t=q(.|o\t, b\pt)$. Therefore, ${G}_{\text{Unified}}^{(\pi)}(o\t)$ can be expressed as a function of both $o\t$ and $b\t$. Since $b_t$ contains sufficient information about $o_t$, it is adequate to express ${G}_{\text{Unified}}^{(\pi)}(o\t)$ in Eq.~\eqref{Eq:Beleif_EFE_beleif4} solely as a function of $b_t$, i.e.,
\begin{eqnarray}
\!\!\!\!\!\!\!\!\! {G}^{(\pi)}_{\text{Unified}} (b\t)&\!\!=\!\!& \mathbb{E}_{ \prod_{\tau =t}^{\infty} P(b_{\tau+1}| s_{\tau+1}, o_{\tau+1}, b_{\tau}, a_{\tau}) \tilde{P}(o_{\tau+1}|a\ta, b\ta) \pi(a\ta|b\ta) q(s_{\tau+1}| o_{\tau+1}, b_{\tau}) } \bigg[ \sum_{\tau =t}^{\infty} \text{log} \pi(a_{\tau}| {b_{\tau}})    \nonumber  \\
\!\!\!\!\!\!\!\!\!  &\!\!-\!\!& \text{log} \tilde{P}({o}_{\tau+1}|a\ta, b\ta) + \text{log} \frac {P({\s_{\tau+1}}|a\ta, b_{\tau}) }{q({\s_{\tau+1}}|o_{\tau+1}, b_{\tau})} \bigg]. \label{Eq:Beleif_EFE_beleif_APP}  
\end{eqnarray} 
\end{proof}
\subsubsection{Proof of Proposition~\ref{cor:q1}} \label{APP:q1}
\begin{proof}
The result is obtained by taking out the terms $\beta \pi(a_{t}|b_{t})$, $\alpha r\t$, and $\zeta \text{log} \frac{P(\s\nt|  b\t, a\t)}{ q({\s}_{t+1}|{o}_{t+1}, b\t)}$ from $G^{(\pi)} (b\t)$ in Eq.~\eqref{Eq:Objective}:
\begin{eqnarray}
\!\!\!\!\!\! G^{(\pi)}(b\t) &=&  \mathbb{E}_{ \prod_{\tau =t}^{\infty} \pi(a_{\tau}|b_{\tau}) {P}(s_{\tau+1},  o_{\tau+1}, b_{\tau+1}|  b\ta, a\ta) P(r\ta|b\ta, a\ta)  } \big[ \sum_{\tau =t}^{\infty}  \gamma^{\tau-t} \big( \beta \, \text{log}\pi(a_{\tau}|b_{\tau})   \nonumber \\
 \!\!\!\!\!\! &-& \alpha \, r\ta   +  \zeta \, \text{log} \frac {P({\s_{\tau+1}}|b\ta,a\ta) }{q({\s_{\tau+1}}|o_{\tau+1}, b\ta)}  \big] \nonumber  \\
&\!\!\! \!\!\overset{\mathrm{(a)}}{=} \!\!\! \!\!& \mathbb{E}_{\pi(a\t|b\t)  P(s\nt, o\nt|b\t,a\t)  P(r\t|b\t,a\t)} \bigg[ \beta \, \text{log} \pi(a\t|b\t) - \alpha \, r\t + \zeta \, \text{log} \frac{P(\s\nt|  b\t, a\t)}{ q({\s}_{t+1}|{o}_{t+1}, b\t)} \nonumber \\
 \!\!\!\!\!\!\!\!\!\!\!\!  &\!\!\!\!\!+\!\!\! \!\!&   \gamma \, \mathbb{E}_{P(b\nt|s\nt, b\t,a\t, o\nt)} \big[ \mathbb{E}_{\prod_{\tau' =t+1}^{\infty} \pi(a_{\tau'}|b_{\tau'})  P({\s_{\tau'+1}},{o_{\tau'+1}}, b_{\tau'+1}| b_{\tau'}, a_{\tau'}) P(r_{\tau'}|b_{\tau'},a_{\tau'})}   \nonumber  \\
  & & \!\!\!\!\! [\sum_{\tau' =t+1}^{\infty} \gamma^{\tau'-t} \big(  \beta \, \text{log} \pi(a_{\tau'}|b_{\tau'}) - \alpha \, r_{\tau'} + \zeta \, \text{log} \frac{P(\s_{\tau'+1}| a_{\tau'}, b_{\tau'})}{ q({\s}_{\tau'+1}|{o}_{\tau'+1}, b_{\tau'})} \big) ] \big] \bigg] \nonumber  \\
 \!\!\!\!\!\!\!\!\!\!\!\!\!   &\!\!\! \!\!\overset{\mathrm{(b)}}{=}\!\!\! \!\!& \mathbb{E}_{\pi(a\t|b\t)  P(r\t|b\t,a\t) P(s\nt, o\nt|b\t,a\t)} \bigg[ \beta \, \text{log} \pi(a\t|b\t) - \alpha \, r\t + \zeta \, \text{log} \frac{P(\s\nt|  b\t, a\t)}{ q({\s}_{t+1}|{o}_{t+1}, b\t)}   \nonumber   \\
     \!\!\!\!\!\!\!\!\!\!\!\!\! &\!\!\!\!\! +\!\!\!\!\!&   \gamma \mathbb{E}_{P(b\nt|b\t,a\t, o\nt)} \big[ {G}^{(\pi)}(b\nt) \big] \bigg]  \nonumber  \\
   \!\!\!\!\!\!\!\!\!\!\!\!\!  & \!\!\!\!\! \overset{\mathrm{(c)}}{=} \!\! \!\!\!& \mathbb{E}_{ \pi(a\t|b\t)  P(r\t|b\t,a\t)} \bigg[ \beta \, \text{log} \pi(a\t|b\t) - \alpha \, r\t + \mathbb{E}_{ P(s\nt, o\nt, b\nt|b\t,a\t)} \big[ \zeta \,   \text{log} \frac{P(\s\nt|  b\t, a\t)}{ q({\s}_{t+1}|{o}_{t+1}, b\t)}  \nonumber \\
\!\!\!\!\!\!\!\!\!\!\!\!\!\!  &\!\!\!\!\!+ \!\!\!\!\!& \gamma  {G}^{(\pi)}(b\nt) \big]  \bigg], \nonumber 
\end{eqnarray}
where (a) follows by applying the law of total expectation, (b) is because of the definition of the belief state EFE, and (c) follows since $\pi(a\t|b\t)$ and $r\t$ are independent of $P(s\nt, o\nt, b\nt|b\t,a\t)$.
\end{proof}
\subsubsection{Proof of Theorem~\ref{theo:optimal_policy}}  \label{APP:optimal_policy}
\begin{proof}
According to Proposition~\ref{cor:q1}, we have:
\begin{eqnarray}
 {G}^{(\pi)}(b\t) &=& \mathbb{E}_{ \pi(a\t|b\t)  P(r\t|b\t,a\t)} \bigg[  \beta \, \text{log} \pi(a\t|b\t) - \alpha \, r\t    \\
   &+& \mathbb{E}_{ {P}(b\nt,o\nt, \s\nt|b\t, a\t)} \big[ \zeta \,   \text{log} \frac{P(\s\nt|  b\t, a\t)}{ q({\s}_{t+1}|{o}_{t+1},b\t)} + \gamma  {G}^{(\pi)}(b\nt) \big]  \bigg]. \nonumber
\end{eqnarray}
Therefore:
\begin{eqnarray}
\lefteqn{\!\!\!\!\!\!\!\!\!\!\!\!\!\!\!\!\!\!\!\!\!\!\!\!\!\! \!\!\!\!\!\!\!\!  {G}^{*}(b\t)= \min_{\pi} {G}^{(\pi)}(b\t) = \min_{\pi} \mathbb{E}_{ \pi(a\t|b\t)  P(r\t|b\t,a\t)} \bigg[ \beta \, \text{log} \pi(a\t|b\t) - \alpha \, r\t } \nonumber \\
 &&\!\!\!\!\!\!\!\!\!\!\!\!\!\!\!\!\!\!\!\!\!\!\!\!    + \mathbb{E}_{{P}(b\nt,o\nt, \s\nt|b\t, a\t)} \big[ \zeta \,   \text{log} \frac{P(\s\nt|  b\t, a\t)}{ q({\s}_{t+1}|{o}_{t+1}, b\t)} + \gamma  {G}^{(\pi)}(b\nt) \big]  \bigg] \nonumber \\
 &&\!\!\!\!\!\!\!\!\!\!\!\!\!\!\!\!\!\!\!\!\!\!\! \overset{\mathrm{(a)}}{=} \min_{\pi(.|b\t), \tilde{\pi}} \mathbb{E}_{ \pi(a\t|b\t)} \bigg[ \mathbb{E}_{ P(r\t|b\t,a\t)} \big[ \beta \, \text{log} \pi(a\t|b\t) - \alpha \, r\t \\
 &&\!\!\!\!\!\!\!\!\!\!\!\!\!\!\!\!\!\!\!\!\!\!\!\! + \mathbb{E}_{ {P}(b\nt,o\nt, \s\nt|b\t, a\t)} [ \zeta \,   \text{log} \frac{p(\s\nt|  \s\t, a\t)}{ q({\s}_{t+1}|{o}_{t+1})} + \gamma  {G}^{(\tilde{\pi})}(\s\nt) ] \big]  \bigg] \nonumber \\
 &&\!\!\!\!\!\!\!\!\!\!\!\!\!\!\!\!\!\!\!\!\!\!\! \overset{\mathrm{(b)}}{=} \min_{\pi(.|b\t)} \mathbb{E}_{ \pi(a\t|b\t)} \bigg[ \mathbb{E}_{ P(r\t|b\t,a\t)} \big[ \beta \, \text{log} \pi(a\t|b\t) - \alpha \, r\t \\
 &&\!\!\!\!\!\!\!\!\!\!\!\!\!\!\!\!\!\!\!\!\!\!\!\!  + \mathbb{E}_{ {P}(b\nt,o\nt, \s\nt|b\t, a\t)} [ \zeta \,   \text{log} \frac{P(\s\nt|  b\t, a\t)}{ q({\s}_{t+1}|{o}_{t+1}, b\t)} + \gamma \min_{\tilde{\pi}} {G}^{(\tilde{\pi})}(b\nt)] \big]  \bigg] \nonumber \\
&&\!\!\!\!\!\!\!\!\!\!\!\!\!\!\!\!\!\!\!\!\!\!\! \overset{\mathrm{(c)}}{=} \min_{\pi(.|b\t)} \mathbb{E}_{ \pi(a\t|b\t)} \bigg[ \mathbb{E}_{ P(r\t|b\t,a\t)} \big[  \beta \, \text{log} \pi(a\t|b\t) - \alpha \, r\t \\
 &&\!\!\!\!\!\!\!\!\!\!\!\!\!\!\!\!\!\!\!\!\!\!\!\!  + \mathbb{E}_{ {P}(b\nt,o\nt, \s\nt|b\t, a\t)} [ \zeta \,   \text{log} \frac{P(\s\nt|  b\t, a\t)}{ q({\s}_{t+1}|{o}_{t+1}, b\t)} + \gamma {G}^{*}(b\nt)] \big]  \bigg], \nonumber
 \end{eqnarray}
where (a) is because of the decomposition of the policy $\pi$ as $\pi=(\pi(.|\s\t), \tilde{\pi})$. Let $\tilde{\pi}^*=\arg\min_{\tilde{\pi}} {G}^{(\tilde{\pi})}(b\nt)$, then (b) follows because of the trivial inequality:
\begin{eqnarray}
\!\!\!\!\!\!\!\!\!\!\!\!\!     \min_{\tilde{\pi}}  \mathbb{E}_{ {P}(b\nt,o\nt, \s\nt|b\t, a\t)} [{G}^{(\tilde{\pi})}(b\nt)] \geq \mathbb{E}_{ {P}(b\nt,o\nt, \s\nt|b\t, a\t)} [\min_{\tilde{\pi}} {G}^{(\tilde{\pi})}(b\nt)] , \nonumber
 \end{eqnarray}
and
\begin{eqnarray}
  \mathbb{E}_{ {P}(b\nt,o\nt, \s\nt|b\t, a\t)} [\min_{\tilde{\pi}} {G}^{(\tilde{\pi})}(b\nt) ] &=& \mathbb{E}_{ {P}(b\nt,o\nt, \s\nt|b\t, a\t)} [{G}^{\tilde{\pi}^*}(b\nt) ] \nonumber  \\
  &\leq& \min_{\tilde{\pi}} \mathbb{E}_{ p(\s\nt,o\nt| a\t)} [{G}^{(\tilde{\pi})}(b\nt)] .
 \end{eqnarray}
Finally, (c) follows from the definition of the optimal belief state EFE.
\end{proof}
\subsubsection{Proof of Corollary.~\ref{corr:optimal_principle}} \label{APP:optimal_principle}
\begin{proof}
 Proof of this corollary is straightforward by taking the argument of Eq.~\eqref{Eq:unified_optimality_Bellman}.
\end{proof}
\subsubsection{Proof of Theorem~\ref{the:instant_optimal}} \label{APP:instant_optimal}
\begin{proof}
From Corollary~\ref{corr:optimal_principle}, we have:
\begin{eqnarray}
 \pi^*(a\t|b\t) &\in&    \arg\min_{\pi(a\t|b\t)} \mathbb{E}_{ \pi(a\t|b\t)} \bigg[ \mathbb{E}_{  P(r\t|b\t,a\t)} [\beta \, \text{log} \pi(a\t|b\t) - \alpha \, r\t  ]  \label{Eq:APP_Bellman_optimal}    \\
  & +& \mathbb{E}_{ P(b\nt,o\nt, \s\nt| s\t,a\t)} \big[ \zeta \,   \text{log} \frac{P(\s\nt|  b\t, a\t)}{ q({\s}_{t+1}|{o}_{t+1}, b\t)} + \gamma \min_{\tilde{\pi}} {G}^{(\tilde{\pi})}(b\nt) \big]  \bigg].\nonumber
\end{eqnarray}
The minimization in the equation above subject to the constraint $\int_{\mA} \pi(a'|b\t) \, da'= 1$  yields minimizing the Lagrangian function $L(\lambda, \pi(a\t|b\t))$ with the Lagrange multiplier $\lambda$:
\begin{eqnarray}
\!\! \!\! L(\lambda, \pi(a\t|b\t)) &\!\!\! \!\!=\!\!\! \!\!& \mathbb{E}_{ \pi(a\t|b\t)} \bigg[ \mathbb{E}_{  P(r\t|b\t,a\t)} [\beta \, \text{log} \pi(a\t|b\t) - \alpha \, r\t  ]  
  + \mathbb{E}_{ P(b\nt,o\nt, \s\nt| s\t,a\t)} \big[ \zeta \,   \text{log} \frac{P(\s\nt|  b\t, a\t)}{ q({\s}_{t+1}|{o}_{t+1}, b\t)}     \nonumber \\
\!\! \!\!  &\!\!\! \!\!+\!\!\! \!\!&   \gamma {G}^{*}(b\nt) \big]  \bigg] - \lambda \left( \int_{\mA} \pi(a'|b\t) \, da'- 1 \right).
\end{eqnarray}
The derivative of $L(\lambda, \pi(a\t|b\t)) $ with respect to $\pi(a\t|b\t)$ is given by:
\begin{eqnarray}
 \frac{\partial L(\lambda, \pi(a\t|b\t))}{\partial \pi(a\t|b\t))} &\!\!\! \!\!=\!\!\! \!\!& \mathbb{E}_{  P(r\t|b\t,a\t)} [ \beta \, \text{log} \pi(a\t|b\t) + \beta - \alpha \, r\t ]\\
  &\!\!\! \!\!+\!\!\! \!\!& \mathbb{E}_{ P(b\nt,o\nt, \s\nt| s\t,a\t)} \big[ \zeta \,   \text{log} \frac{P(\s\nt|  b\t, a\t)}{ q({\s}_{t+1}|{o}_{t+1},b\t)} + \gamma {G}^{*}(b\nt) \big] -\lambda.  \nonumber
\end{eqnarray}
 Setting $ \frac{\partial L(\lambda, \pi(a\t|b\t))}{\partial \pi(a\t|b\t))}=0$ gives the optimal ${\pi^*}(a\t|b\t)$ for the specific action $a\t$ as:
\begin{align}
\!\!\! \!\!\! {\pi^*}(a\t|b\t) \!\! = \!\! \text{exp} \!\! \left( \! \frac{ \lambda \!- \! \beta - \! \mathbb{E}_{  P(r\t|b\t,a\t)} \big[ \alpha r\t  \!  - \!  \mathbb{E}_{ P(b\nt,o\nt, \s\nt| b\t,a\t)} [ \zeta \,   \text{log} \frac{P(\s\nt|  b\t, a\t)}{ q({\s}_{t+1}|{o}_{t+1},b\t)} \!  + \!  \gamma  {G}^{*}(b\nt) ] \big] }{\beta}\! \right). \nonumber \\ \label{Eq:La1}
\end{align}
By accounting the constraint $\int_{\mA} \pi(a'|b\t) \, da'= 1$:
\begin{eqnarray}
\lefteqn{\!\!\!\! \int_{\mA} \pi(a'|b\t) \, da' = \text{exp} (\frac{\lambda - \beta}{\beta})} \label{Eq:La2}  \\
 &&\!\!\!\!\!\!\!\!\!\!\!  \int_{\mA} \text{exp} \left( \frac{ \mathbb{E}_{  P(r\t|b\t,a\t)} \big[ \alpha r\t  \!  - \!  \mathbb{E}_{ P(b\nt,o\nt, \s\nt| b\t,a\t)} [ \zeta \,   \text{log} \frac{P(\s\nt|  b\t, a\t)}{ q({\s}_{t+1}|{o}_{t+1},b\t)} \!  + \!  \gamma  {G}^{*}(b\nt) ] \big] }{\beta} \right) =1. \nonumber
\end{eqnarray}
Putting together Eqs.~\eqref{Eq:La1} and \eqref{Eq:La2} results in:
\begin{eqnarray}
\!\!\!\!\!\!\!\!\!\!\!\!\!\! \pi^*(a\t|b\t) &\!\!\!\!\!\!=\!\!\!\!\!\!& \frac{ \text{exp} \left(  \frac {\mathbb{E}_{  P(r\t|b\t,a\t)} \big[ \alpha r\t  \!  - \!  \mathbb{E}_{ P(b\nt,o\nt, \s\nt| b\t,a\t)} [ \zeta \,   \text{log} \frac{P(\s\nt|  b\t, a\t)}{ q({\s}_{t+1}|{o}_{t+1},b\t)} \!  + \!  \gamma  {G}^{*}(b\nt) ] \big]} {\beta} \right) }   { \int_{\mA} \text{exp} \left(  \frac{\mathbb{E}_{  P(r\t|b\t,a')} \big[ \alpha r\t  \!  - \!  \mathbb{E}_{ P(b\nt,o\nt, \s\nt| b\t,a')} [ \zeta \,   \text{log} \frac{P(\s\nt|  b\t, a')}{ q({\s}_{t+1}|{o}_{t+1},b\t)} \!  + \!  \gamma  {G}^{*}(b\nt) ] \big]}{\beta} \right)\, da' } \nonumber \\
\!\!\!\!\!\! &\!\!\!\!\!\!=\!\!\!\!\!\!& \sigma  \!\! \left( \!\frac{ \mathbb{E}_{  P(r\t|b\t,a\t)} \big[ \alpha r\t  \!  - \!  \mathbb{E}_{ P(b\nt,o\nt, \s\nt| b\t,a\t)} [ \zeta \,   \text{log} \frac{P(\s\nt|  b\t, a\t)}{ q({\s}_{t+1}|{o}_{t+1},b\t)} \!  + \!  \gamma  {G}^{*}(b\nt) ] \big] }{\beta} \!\right). \nonumber 
\\ 
\label{Eq:APP_optimal}
\end{eqnarray}
Taking the second derivative of $L(\lambda, \pi(a\t|b\t))$ with respect to $\pi(a\t|b\t)$ yields:
\begin{eqnarray}
\frac{\partial^2 L(\lambda, \pi(a\t|b\t))}{\partial \pi^2(a\t|b\t)} = \frac{\beta}{\pi(a\t|b\t)},   \label{Eq:La3}
\end{eqnarray}
which is a positive value for the compact action space $\mA$. Therefore, the expression in Eq.~\eqref{Eq:APP_optimal} is the unique minimizer of $L(\lambda, \pi(a\t|b\t))$ and, consequently, the unique minimizer of the unified Bellman equation.
\end{proof}
\subsubsection{Proof of Lemma~\ref{lem:constr_action_EFE}} \label{APP:constr_action_EFE}
\begin{proof}
From the definition of $\mathcal{G}^{(\bar{\pi}^*)}(b\t,a\t)$ in Eq.~\eqref{Eq:opt_constr_action_EFE}:
\begin{eqnarray}
 \mathcal{G}^{(\bar{\pi}^*)}(b\t,a\t) &\!\!\!\!=\!\!\!\!& \mathbb{E}_{  P(r\t|b\t,a\t)} \left[  -  \alpha r\t    + \!  \mathbb{E}_{ P(b\nt,o\nt, \s\nt| b\t,a\t)} \big[ \zeta   \text{log} \frac{P(\s\nt| b\t, a\t)}{ q({\s}_{t+1}|{o}_{t+1}, b\t)} \!  + \!  \gamma  {G}^{(\bar{\pi}^*)}(b\nt) \big] \right] \nonumber  \\
\!\!\!\!\!\!  &\!\!\!\! \overset{\mathrm{(a)}}{=} \!\!\!\!&  \mathbb{E}_{  P(r\t|b\t,a\t)} \left[ -  \alpha r\t    + \!  \mathbb{E}_{ P(b\nt,o\nt, \s\nt| b\t,a\t)} \big[ \zeta   \text{log} \frac{P(\s\nt| b\t, a\t)}{ q({\s}_{t+1}|{o}_{t+1}, b\t)} \!  + \!  \gamma  \min_{\bar{\pi} \in \Pi} {G}^{(\bar{\pi})}(b\nt) \big] \right]  \nonumber \\
 &\!\!\!\! \overset{\mathrm{(b)}}{=}\!\!\!\!&  \min_{\bar{\pi} \in \Pi} \underbrace{\mathbb{E}_{  P(r\t|b\t,a\t)} \left[  -  \alpha r\t   + \!  \mathbb{E}_{ P(b\nt,o\nt, \s\nt| b\t,a\t)} \big[ \zeta   \text{log} \frac{P(\s\nt| b\t, a\t)}{ q({\s}_{t+1}|{o}_{t+1}, b\t)} \!  + \!   \gamma  {G}^{(\bar{\pi})}(b\nt) \big] \right]}_{\mathcal{G}^{(\bar{\pi})}(b\t,a\t)} \nonumber \\
&\!\!\!\!=\!\!\!\!& \min_{\bar{\pi} \in \Pi} \mathcal{G}^{(\bar{\pi})}(b\t,a\t),
\end{eqnarray}
where (a) follows from the definition of ${G}^{(\bar{\pi}^*)}(b\nt)$, and (b) follows since $P(r\t|b\t,a\t)$, $r\t$, $ P(b\nt,o\nt, \s\nt| b\t,a\t)$, and $  \text{log} \frac{P(\s\nt|  b\t, a\t)}{ q({\s}_{t+1}|{o}_{t+1}, b\t)} $ are independent of $\bar{\pi}$.
\end{proof}
\subsubsection{Proof of Proposition~\ref{pro:unified_Bellman_action}} \label{APP:unified_Bellman_action}
\begin{proof}
The proof follows two steps:\\
First, by replacing  $\mathcal{G}^{(\bar{\pi})}(b\t,a\t)$  with $\mathbb{E}_{P(r\t|b\t,a\t)} \big[-\alpha r\t + \mathbb{E}_{P(b\nt,o\nt, \s\nt| b\t,a\t)} [\zeta \text{log} \frac{P(\s\nt| b\t, a\t)}{q({\s}_{t+1}|{o}_{t+1}, b\t)} + \gamma {G}^{(\bar{\pi})}(b\nt)] \big]$ in Eq.~\eqref{Eq:unified_Bellman_mod}, we can express the belief state EFE ${G}^{(\bar{\pi})}(b\t)$ as a function of the belief state-action EFE ${G}^{(\bar{\pi})}(b\t,a\t)$ as follows:
\begin{eqnarray}
\!\!\!\!\!\!\!\!\!\!\!\!\!\!\!\!   {G}^{(\bar{\pi})}(b\t)  &\!\!\! =\!\!\! & \mathbb{E}_{\bar{\pi}(a\t|b\t)} \left[\beta \, \text{log} {\bar{\pi}}(a\t|b\t) + \mathcal{G}^{(\bar{\pi})}(b\t, a\t)     \right].  \label{Eq:G_g} 
\end{eqnarray}
Now, by plugging ${G}^{(\bar{\pi})}(b\nt) = \mathbb{E}_{\bar{\pi}(a\nt|b\nt)} \left[\text{log} \bar{\pi}(a\nt|b\nt) + \mathcal{G}^{(\bar{\pi})}(b\nt, a\nt) \right] $ obtained from Eq.~\eqref{Eq:G_g} into the result of Lemma~\ref{lem:constr_action_EFE}, the proof is completed:
\begin{eqnarray}
\!\!\!\!\!\!\!\! \!\!\! \!\!\!\!\! \mathcal{G}^{(\bar{\pi})}(b\t,a\t)  &\!\!=\!\!&  \mathbb{E}_{  P(r\t|b\t,a\t)} \bigg[  -\alpha \, r\t+\mathbb{E}_{   P(b\nt,o\nt, \s\nt| b\t,a\t) }  \big[ \zeta \,   \text{log} \frac{P(\s\nt|  b\t, a\t)}{ q({\s}_{t+1}|{o}_{t+1}, b\t)}   \nonumber \\
&\!\!+\!\!&   \gamma \mathbb{E}_{\bar{\pi}(a\nt|b\nt)}[ \beta \, \text{log} {\pi}(a\nt|b\nt) + \mathcal{G}^{(\bar{\pi})}(b\nt, a\nt)]  \big] \bigg].
\end{eqnarray}
\end{proof}
%
\subsubsection{Proof of Theorem~\ref{lem:policy_evaluation}} \label{APP:policy_evaluation}
\begin{proof}
Following~\cite{sutton2018reinforcement, NEURIPS2019_13384ffc}, let's define $P^{\pi}(b\nt, o\nt, \s\nt, b\t, a\t)$ and $R^{unified}(b\t, a\t)$ as
\begin{eqnarray}
\!\!\!\!\! P_{\bar{\pi}}( b\t,a\t, b\nt, a\nt) :=\mathbb{E}_{P(o\nt, s\nt| b\t,a\t)} [P(b\nt|s\nt, o\nt, b\t, a\t) \bar{\pi}(a\nt|b\nt)  ],
 \end{eqnarray}
\begin{eqnarray}
\!\!\!\!\!\!\!\!\!\!\!\!\!\!\!\! R^{unified}(b\t,a\t) &\!\!\! := \!\!\! & \mathbb{E}_{  P(r\t|b\t,a\t)} \bigg[  -\alpha \, r\t+\mathbb{E}_{   P(b\nt,o\nt, \s\nt| b\t,a\t) }  \big[ \zeta \,   \text{log} \frac{P(\s\nt|  b\t, a\t)}{ q({\s}_{t+1}|{o}_{t+1}, b\t)}   \nonumber \\
&\!\!+\!\!&   \gamma \mathbb{E}_{\bar{\pi}(a\nt|b\nt)}[ \beta \, \text{log} {\pi}(a\nt|b\nt) ] \big] \bigg]. \label{Eq:r_unified}
 \end{eqnarray}
The compactness assumptions on the state space $\mathcal{S}$, action space $\mathcal{A}$, and observation space $\mathcal{O}$ guarantee that the space of all belief states $b\ta \in \mathcal{B}$ is a compact subset of a separable Hilbert space. As a result, both $P_{\bar{\pi}}( b\t,a\t, b\nt, a\nt)$ and $R^{unified}(b\t, a\t)$ are bounded. With these assumptions in place, we can express $T^{unified}_{\bar{\pi}} \mathcal{G}$ in a compact form as follows:
\begin{eqnarray}
\!\!\!\!\! T^{unified}_{\bar{\pi}} \mathcal{G} =  R^{unified} + \gamma P^{\pi}  \mathcal{G}, \nonumber
 \end{eqnarray}
By starting from an initial bounded mapping $\mathcal{G}_0: \mathcal{B} \times \mathcal{A} \rightarrow \mathbb{R}$ and recursively applying $\mathcal{G}\nk = T^{unified}_{\bar{\pi}} \mathcal{G}\k$ for $k = 0, 1, 2, \ldots$, we have
\begin{eqnarray}
\!\!\!\!\! \mathcal{G}^{\bar{\pi}} := \lim_{k \rightarrow \infty} T^{unified}_{\bar{\pi}} \mathcal{G}\k &\!\!\! = \!\!\! & \text{lim}_{k \rightarrow \infty} \sum_{i=0}^{k-1} \gamma^i P_{\bar{\pi}}^i  R^{unified} + \gamma^k P_{\bar{\pi}}^k \mathcal{G}_0 \nonumber \\
&\!\!\! \overset{\mathrm{(a)}}{=}\!\!\! & \text{lim}_{k \rightarrow \infty} \sum_{i=0}^{k-1} \gamma^i P_{\bar{\pi}}^i  R^{unified},
 \end{eqnarray}
where (a)  is because  $\mathcal{G}_0$ and $P_{\bar{\pi}} $ are bounded and thus the term $\gamma^k P_{\bar{\pi}}^k \mathcal{G}_0$ will converge to zero. Therefore the convergence of $T^{unified}_{\bar{\pi}} \mathcal{G}^{\bar{\pi}}$ does not depend on the initial value $\mathcal{G}_0$. Now, by substituting $\mathcal{G}^{\bar{\pi}}$ with $\text{lim}_{k \rightarrow \infty} \sum_{i=0}^{k-1} \gamma^i P_{\bar{\pi}}^i  R^{unified}$, we have:
\begin{eqnarray}
\!\!\!\!\! T^{unified}_{\bar{\pi}} \mathcal{G}^{\bar{\pi}} &\!\!\! = \!\!\! & R^{unified} + \gamma P^{\pi}  \lim_{k \rightarrow \infty} \sum_{i=0}^{k-1} \gamma^i P_{\bar{\pi}}^i  R^{unified} = \lim_{k \rightarrow \infty} \sum_{i=0}^{k} \gamma^i P_{\bar{\pi}}^i  R^{unified} \nonumber  \\
&\!\!\! \overset{\mathrm{(a)}}{=} \!\!\! &\lim_{k \rightarrow \infty} \sum_{i=0}^{k-1} \gamma^i P_{\bar{\pi}}^i  R^{unified} + \underbrace{\gamma^k P_{\bar{\pi}}^k R^{unified}}_{0} = \mathcal{G}^{\bar{\pi}}, \label{Eq:fixed_point}
 \end{eqnarray}
where (a)  follows since  $R^{unified}$ is bounded. Eq.~\eqref{Eq:fixed_point} demonstrates that $\mathcal{G}^{\bar{\pi}}$ is a fixed point of $T^{unified}_{\bar{\pi}} \mathcal{G}^{\bar{\pi}}$. \\
To establish the uniqueness of this fixed point, let's assume there exists another fixed point $\mathcal{G}'$ of $T^{unified}_{\bar{\pi}}$ such
that $T^{unified}_{\bar{\pi}}\mathcal{G}'=\mathcal{G}'$. Then, $\mathcal{G}'=\lim_{k \rightarrow \infty} T^{unified}_{\bar{\pi}} \mathcal{G}\k$ with $\mathcal{G}_0 = \mathcal{G}'$ converges to $\mathcal{G}^{\bar{\pi}}$ since the convergence behavior of $T^{unified}_{\bar{\pi}}$ is independent of the initial value $\mathcal{G}'$. Consequently, we can conclude that $\mathcal{G}' = \mathcal{G}^{\bar{\pi}}$. 
\end{proof}
 %
\subsubsection{Proof of Corollary~\ref{theo:action_optimal_policy}} \label{APP:action_optimal_policy}
\begin{proof}
This result can be easily obtained by replacing $\bar{\pi}$ and $\mathcal{G}^{(\bar{\pi})}(b\t,a\t)$ in the result of Theorem~\ref{lem:policy_evaluation} with the optimal policy $\bar{\pi}^*$ and its corresponding belief state-action value function $\mathcal{G}^{(\bar{\pi}^*)}(b\t,a\t)$, i.e., the optimal belief state-action EFE.
\end{proof}
\subsubsection{Proof of Lemma~\ref{lem:policy_impr}} \label{APP:policy_impr}
\begin{proof}
Based on the definition of $\bar{\pi}^{(new)}$, for any $\bar{\pi} \in \bar{\Pi}$:
\begin{eqnarray}
   \! \!  \! \! \! \!  \! \! \! \!  D_{KL} \!\left[{\bar{\pi}}^{(new)}(.|b\t), \sigma \left( \frac{- \mathcal{G}^{(\bar{\pi}^{(old)})}(b\t,.)}{\beta} \right) \right] 
\! \leq D_{KL} \! \! \left[\bar{\pi}(.|b\t), \sigma \left( \frac{-\mathcal{G}^{(\bar{\pi}^{(old)})}(b\t,.)}{\beta} \right) \right] \! \! . \label{Eq:KL_in_APP}
\end{eqnarray}
By choosing $\bar{\pi}=\bar{\pi}^{(old)}$ in Eq.~\eqref{Eq:KL_in_APP} and using definition of KL-divergence:
\begin{eqnarray}
\lefteqn{  \!\!\!\!\!\!\!\!\!\!\!\!\!\!\!\! \!\!\!\!  - \beta\, \mathcal{H} ({\bar{\pi}}^{(new)}(.|b\t)) + \mathbb{E}_{{\bar{\pi}}^{(new)}(a\t|b\t)} \left[ \mathcal{G}^{(\bar{\pi}^{(old)})}(b\t, a\t)     \right]} \label{Eq:compare}     \\
&& \,\,\,\,\,\,\,\,\,\,\,\,\,\,\,\,\,\,\,\,   \leq - \beta\, \mathcal{H} (\bar{\pi}^{(old)}(.|b\t)) + \mathbb{E}_{\bar{\pi}^{(old)}(a\t|b\t)} \left[ \mathcal{G}^{(\bar{\pi}^{(old)})}(b\t, a\t)     \right] =  {G}^{(\bar{\pi}^{(old)})}(b\t). \nonumber
\end{eqnarray}
Moreover, from Eq.~\eqref{Eq:g_G}, we have:
\begin{eqnarray}
\!\!\!\!\!\!\!\!\!\!\!\!\!\!\!\!\! \mathcal{G}^{(\bar{\pi}^{(old)})}(b\t, a\t)  &\!\!=\!\!&  \mathbb{E}_{  P(r\t|b\t,a\t)} \bigg[  -\alpha \, r\t+\mathbb{E}_{   P(b\nt,o\nt, \s\nt| b\t,a\t) }  \big[ \zeta \,   \text{log} \frac{P(\s\nt|  b\t, a\t)}{ q({\s}_{t+1}|{o}_{t+1}, b\t)}   \nonumber \\
&\!\!+\!\!&   \gamma  {G}^{(\bar{\pi}^{(old)})}(b\nt)  \big] \bigg] \label{Eq:extend} 
\end{eqnarray}
Now,  by repeatedly applying Eq.~\eqref{Eq:extend} and the bound in Eq.~\eqref{Eq:compare}:
\begin{eqnarray}
\!\!\!\!\!\!\!\!\!\!\!\!\!\!\!\!\! \mathcal{G}^{(\bar{\pi}^{(old)})}(b\t, a\t)  &\!\!=\!\!&  \mathbb{E}_{  P(r\t|b\t,a\t)} \bigg[  -\alpha \, r\t+\mathbb{E}_{   P(b\nt,o\nt, \s\nt| b\t,a\t) }  \big[ \zeta \,   \text{log} \frac{P(\s\nt|  b\t, a\t)}{ q({\s}_{t+1}|{o}_{t+1}, b\t)}   \nonumber \\
&\!\!+\!\!&   \gamma  {G}^{(\bar{\pi}^{(old)})}(b\nt)  \big] \bigg] \nonumber \\
\!\!\!\!\!\!\!\!\!\!\!\!\!\!\!\!\!\!\!\!\!\! &\!\! {\geq} \!\!&  \mathbb{E}_{  P(r\t|b\t,a\t)} \bigg[  -\alpha \, r\t+\mathbb{E}_{   P(b\nt,o\nt, \s\nt| b\t,a\t) }  \big[ \zeta \,   \text{log} \frac{P(\s\nt|  b\t, a\t)}{ q({\s}_{t+1}|{o}_{t+1}, b\t)}   \nonumber \\
\!\!\!\!\!\!\!\!\!\!\!\!\!\!\!\!\!\!\!\!   &\!\!-\!\!& \beta\, \mathcal{H} ({\bar{\pi}}^{(new)}(.|b\t)) + \mathbb{E}_{{\bar{\pi}}^{(new)}(a\t|b\t)} \left[ \mathcal{G}^{(\bar{\pi}^{(old)})}(b\t, a\t)     \right] \\
&& \vdots \nonumber \\
\!\! \!\!\!\!\!\!\!\!\!\!\!\!\!\!\!\! &\!\! \overset{\mathrm{(a)}} {\geq} \!\!&    \mathcal{G}^{(\bar{\pi}^{(new)})}(b\t, a\t),
\end{eqnarray}
where (a) results from Theorem~\ref{lem:policy_evaluation}. 
\end{proof}

\subsubsection{Proof of Theorem~\ref{the:policy_iteraion}} \label{APP:policy_iteraion}
\begin{proof}
Let $\bar{\pi}\k$ be the policy at iteration $k$ of the policy iteration process. According to Lemma~\ref{lem:policy_impr}, the sequence $\mathcal{G}^{(\bar{\pi}\k)}$ is monotonically decreasing. Since $R^{unified}$ defined in Eq.~\eqref{Eq:r_unified} is bounded, $\mathcal{G}^{(\bar{\pi})}=\mathbb{E}[\sum_{\tau=t}^{\infty} \gamma^{\tau-t} R^{unified}(b\ta, a\ta))]$  is also bounded. The sequence $\mathcal{G}^{(\bar{\pi}\k)}$ and consequently sequence $\bar{\pi}\k$ converge to some $\mathcal{G}^{(\bar{\pi}^*)}$ and policy $\bar{\pi}^*$. We still need to show that $\bar{\pi}^*$ is the optimal policy. To do so, we need to show that the belief state-action of the converged policy is lower than any other policy in $\bar{\Pi}$, i.e., $\mathcal{G}^{(\bar{\pi}^*)}(b\t,a\t) < \mathcal{G}^{(\bar{\pi})}(b\t,a\t)$ for all $\bar{\pi} \in \Pi$ and all $(b\t,a\t) \in \mB \times \mA $. This can be achieved by following the same iterative steps outlined in proof of Lemma~\ref{lem:policy_impr}. 
\end{proof}
\subsection{Unified inference model}
In this section, we elaborate on the proposed unified inference model, encompassing perceptual inference and learning, as well as the unified actor-critic framework. We delve into three training methods for the unified actor-critic: model-based, model-free, and hybrid approaches. Furthermore, we present the pseudocode for the unified inference model.
\subsubsection{{Perceptual inference and learning model}} \label{sub-sec:condionined_perc}
{ As stated in Sub-section~\ref{sub:AIF_inference}, AIF approximates generative model $P(o\t,\s\t|a_{t-1}, \s_{t-1})=P(o\t|\s\t)P(\s\t|s\pt, a\pt)$ and the variational distribution $q(s\t|o\t)$ by minimizing the VFE (negative of ELBO in VAE). However, in our proposed unified inference, we need  to approximate the belief state-conditioned variational posterior $q({\s_{t}}|o_{t}, b_{t-1})$ and the belief state generative model defined in Eq.~\eqref{Eq:beleif_GM}, i.e.,
\begin{eqnarray}
 \!\!\!\!\!\!\!\!\!\!\!\! P(o\t,\s\t, b\t|a_{t-1}, b_{t-1})=P(o\t|\s\t,b\pt)P(\s\t|b\pt, a\pt) P(b\t|s\t, o\t, b\pt,a\pt ). \label{Eq:instant_beleif_GM}
 \end{eqnarray}
To achieve this, we use the Variational Recurrent Neural Networks (VRNN) model~\cite{chung2015recurrent}. VRNNs combine the variational inference of VAEs with the sequential modeling capabilities of RNNs. They minimize the VFE conditioned on a variable of an RNN, enabling the modeling of temporal dependencies and generating sequential data. However, in our case, where the POMDP assumes a continuous state space (i.e., an infinite number of latent states), the corresponding belief space $b_{t-1}$ is a continuous probability distribution function with an infinite number of dimensions. Therefore, we need to approximate $b\pt$ to use it as an input to the VRNN model for approximating  $P({o}_t|{s}_t, {b}_{t-1})$, $P({s}_t|{b}_{t-1},{a}_{t-1})$, $P(b\t|s\t, o\t, b\pt,a\pt )$, and $q({s}_t|{o}_t, {b}_{t-1})$. 
\\
Various approaches can be used to approximate a continuous belief state for input to a neural network model, including vector representation, kernel density estimation (KDE), discretization, and particle filtering~\cite{igl2018deep}. These methods aim to transform the continuous belief state into a fixed-length representation that can be fed into the neural network. One common approach is to learn a belief representation through the belief state update function in Eq.~\eqref{Eq:beleif-state}~\cite{gregor2019shaping}. Following established practices, we learn an explicit belief state representation $h_t \in \mathbb{R}^{D_{\mathcal{H}}}$, where $D_{\mathcal{H}}$ represents the dimension of the belief state representation $h_t$.\footnote{While our unified inference framework can be combined with various algorithms for approximating the belief state, we employ a commonly used approximation method that has demonstrated effectiveness in many recent approaches.} We learn this belief state representation  through a deterministic non-linear function $f$ that updates the belief state representation based on the previous belief state representation, the previous action, the current observation, and the current state:
\begin{eqnarray}
 \!\!\!\!\!\!\!\!\!\!\!\! h\t &\!\! = \!\!&   f(h\pt, a\pt, o\t, s\t).  \label{Eq:RNN}
 \end{eqnarray}
 By iteratively applying the update function $f$, the belief state representation can adapt and evolve over time as new observations and actions are encountered.  This approach allows us to approximate the belief state representation in a continuous space and utilize it as an input to a VRNN model for approximating $P({o}_t|{s}_t, {h}_{t-1})$, $P({s}_t|{h}_{t-1},{a}_{t-1})$, $P(h\t|s\t, o\t, h\pt,a\pt )$, and $q({s}_t|{o}_t, {h}_{t-1})$. Given the transition model $P(s\t|b\pt, a\pt)$ and the update function $f(h\pt, a\pt, o\t, s\t)$, we can intuitively say that the states are split into a stochastic part $\s\t$ and a deterministic part $h\t$. This aligns with the modifications made in previous RL and AIF works that heuristically incorporate this separation~\cite{Hafner2020Dream, ogishima2021reinforced, lee2020stochastic, han2020variational}. 
 \\
 We parameterize the posterior distribution with $\nu$, denoted as $q_{\nu}(\s\t|o\t, h\pt)$, and the agent's belief state generative model with $\theta$, i.e., $P_{\theta}(o\t|\s\t, h\pt)$, $P_{\theta}(\s\t|h\pt,a\pt)$, and $P_{\theta}(h\t|s\t, o\t, h\pt,a\pt )=\delta(h\t-f_{\theta}(h\pt, a\pt, o\t, s\t))$.  
  To model $q_{\nu}(\s\t|o\t, b\pt)$, $P_{\theta}(o\t|\s\t. b\pt)$, and $P_{\theta}(\s\t|b\pt,a\pt)$ we use DNNs that output the mean and standard deviation of the random variables according to the Gaussian distribution. For example, we model $q_{\nu}({s}_t|{o}_t, {b}_{t-1})$ as a network that takes ${o}_t$ and ${b}_{t-1}$ as inputs, calculates through several hidden layers, and outputs the Gaussian distribution of ${s}_t$. The update function $f_{\theta}$ can be implemented using gated activation functions such as Long Short-Term Memory (LSTM) or Gated Recurrent Unit (GRU). We use LSTM~\cite{hochreiter1997long} as it has shown good performance in general cases.
\\
The parameters $\nu$ and $\theta$ are trained by minimizing the objective function of VRNN, given by
\begin{eqnarray}
 \!\!\!\!\!\!\!\!\!\!\!\! F^{\text{VRNN}}_{\theta, \nu} \!\! &\!\! = \!\!& \!\!  -\mathbb{E}_{q(\s\t|o\t, h\pt)} \left[\text{log} P(o\t|\s\t, h\pt) \right] +  D_{KL} [q_{\nu}(.|o\t, h\pt), P_{\theta}(.|b\pt,a\pt)].  \label{Eq:VFE_VRNN} 
 \end{eqnarray}
%
The VRNN objective function $F^{\text{VRNN}}_{\theta, \nu}$ is minimized using gradient descent on batches of data sampled from a dataset called a replay buffer, denoted as $\mathcal{D}$. The replay buffer contains quadruples consisting of the agent's action $a_k$, the received extrinsic reward $r\nk$ (caused by the action), and the subsequent observation $o_{k+1}$. The VFE can be expressed as an expectation over a batch of size $B$, consisting of $M$ sequential data points, denoted as $\{( a\pk, o_k)_{k=1}^{M} \}_{i=1}^B$, which are sampled from the replay buffer $\mathcal{D}$. 
\begin{eqnarray}
 \!\!\!\!\!\!\!\!\!\!\!\! F^{\text{VRNN}}_{\theta, \nu} &\!\! = \!\!&  \mathbb{E}_{ \mathcal{D}(a\pk,o\k)}\bigg[ -\mathbb{E}_{q_{\nu}(\s\t|o\k, h\pt)} \left[\text{log} P_{\theta}(o\k|\s\t, h\pt) \right] \nonumber \\
&\!\! + \!\!&  D_{KL} [q_{\nu}(.|o\t, h\pt), P_{\theta}(.|h\pt,a\pk)] \bigg] \label{Eq:VFE_replay} 
\end{eqnarray}
The expectation with respect to $q_{\nu}(s_t | o_k, h\pt)$ in Eq.~\eqref{Eq:VFE_replay} involves integrating over the continuous state space $\mathcal{S}$, which is not analytically tractable. To estimate $F^{\text{VRNN}}_{\theta, \nu}$, we use Monte Carlo estimation, which provides an unbiased estimation by sampling $s_t\l$ for $l\in \{1,2, ..., L \}$ from $q_{\nu}(s_t | o_k, h\pt)$ and evaluating $F^{\text{VRNN}}_{\theta, \nu}$ for each sampled state. The estimate of $F^{\text{VRNN}}_{\theta, \nu}$ is obtained by taking the average over the samples.
However, computing the gradient of $F^{\text{VRNN}}_{\theta, \nu}$ with respect to $\nu$ requires differentiating with respect to $q_{\nu}(s_t | o_k, h\pt)$, which is the distribution we sampled from. To enable gradient-based optimization, it is necessary to have a differentiable sampling process that generates samples from a probability distribution. The reparameterization trick~\cite{kingma2013auto} provides a solution to this problem by reparameterizing the original distribution and introducing an auxiliary variable that follows a fixed distribution, such as a standard Gaussian. This reparameterization allows us to obtain differentiable samples by applying a deterministic transformation to the auxiliary variable. As a result, we can backpropagate gradients through the sampling operation and compute gradients with respect to the distribution parameters.
In our approach, we draw inspiration from established practices in the field~\cite{lee2020stochastic, han2020variational, Hafner2020Dream} and leverage the reparameterization trick~\cite{kingma2013auto} to sample $s_t^{l}$ from the distribution $q_{\nu}(s_t | o_k, h\pt)$. Given samples $s_t^{(l)}$, belief state representation $h_t$ is then computed as $h_t = \frac{1}{L} \sum_{l=1}^L f_{\theta}(h_{t-1}, a_{k-1}, o_k,s\t\l )$.
}
\subsubsection{Reward model learning} \label{APP:reward_model}
We approximate $P(r_t|h_t,a_t)$ using a DNN with parameter $\xi$, denoted as $P_{\xi}(r_t|h_t,a_t)$. The reward model approximation is a supervised learning problem that can be addressed through likelihood maximization, using samples from the replay buffer $\mathcal{D}$. This results in minimizing the following loss function:
\begin{eqnarray}
 L_{r}(\xi) = -\mathbb{E}_{\mathcal{D}( a\k, r\k)} [\text{log} P_{\xi}(r\k|h\t,a\k)].  \label{Eq:r_loss}
\end{eqnarray}
\subsubsection{Pseudocode for the unified inference } \label{App:algorithm}
Algorithm~\ref{Algorithm:unified iteration} outlines the pseudocode for the proposed unified inference framework. We use a variable $mode$ to represent the approach (model-based, model-free or hybrid) used for training the proposed unified actor-critic model.
\begin{algorithm}[tp]
\caption{\textproc{Unified inference model}}
\label{Algorithm:unified iteration}
\begin{algorithmic}[1]
\State \textbf{Model components:} Likelihood $P_{\theta}(o\t|\s\t, h\pt)$, variational posterior $q_{\nu}(\s\t|o\t, h\pt)$, transition model $P_{\theta}(\s\nt|h\t,a\t)$,  belief state representation update function $f_{\theta}(h\pt, a\pt, o\t,s\t)$, reward model $P_\xi(r_t|h_t,a_t)$, belief state-action EFE $\mathcal{G}_{\psi}(h\t,a\t)$, belief state-action policy $\bar{\pi}_{\phi}(a\t|h\t)$.
\State \textbf{Hyperparameters:} Scaling parameters $\alpha$, $\beta$, $\zeta$, and $c$, discount factor $\gamma$, learning rate $\lambda$, actor-critic learning mode indicator $I_{mode}$, batch size $B$, sequence  length $M$, imagination horizon $N$.
\State \textbf{Initialize:}  Neural network parameters $\theta$,  $\nu$, $\xi$, $\psi$, and $\phi$, global step $t \leftarrow 0$.
\State  \textbf{while} not converged  \textbf{do}:
\State \quad Reset the environment and initialize $h_0$ randomly.
\State \quad \textbf{for} each environment step $t$ \textbf{do}:
\State \quad  \quad Receive observation $o\t$
  \State \quad \quad   Sample $\s\t \sim q_{\nu}(\s\t|o\t, h\pt)$   
\State \quad \quad Sample action ${a}\t\sim \bar{\pi}_{\phi}(a\t|h\t)$ and execute $a\t$ in the environment.
\State \quad  \quad Receive observation $o\nt$ and extrinsic reward $ r\t$.
\State \quad \quad  Record $( a\t, r\t, o\nt)$ into $\mathcal{D}$.
\State \quad \textbf{end for}
\State \quad  \textbf{for} each gradient step \textbf{do}:
\State \quad  \quad Sample a minibatch $\{( a\pk, o\k)_{k=1}^{M} \}_{i=1}^B$ from $\mathcal{D}$.
\State \quad \quad Compute $F^{\text{VRNN}}_{\theta, \nu}$ from Eq.~\eqref{Eq:VFE_replay}. 
\State \quad \quad  \multiline{ Update $\theta \leftarrow \theta - \lambda  \nabla_{\theta}{F^{\text{VRNN}}_{\theta, \nu}  } .$ } 
\State \quad \quad  \multiline{ Update $\nu \leftarrow \nu - \lambda  \nabla_{\nu}{F^{\text{VRNN}}_{\theta, \nu}  } .$ }
\State \quad \quad  \textbf{if} $I_{mode}==model-free$ \textbf{then}: \algorithmiccomment{Model-free unified actor-critic.}
\State \quad \quad  \quad Sample a minibatch $\{ (a\k, r\k, o\nk)_{k=1}^{M} \}_{i=1}^B$ from $\mathcal{D}$.
\State \quad \quad  \quad Compute $L_{\mathcal{G}}^{\text{MF}}(\psi)$ from Eq.~\eqref{Eq:loss_psi_MF} and set $L_{\mathcal{G}}(\psi) = L_{\mathcal{G}}^{\text{MF}}(\psi)$.
\State \quad \quad  \textbf{end if}
\State \quad \quad  \textbf{if} $I_{mode}==model-based$ \textbf{then}:  \algorithmiccomment{Model-based unified actor-critic.}
\State \quad \quad  \quad Compute $ L_{r}(\xi) $ from Eq.~\eqref{Eq:r_loss}. 
\State \quad \quad \quad  \multiline{ Update ${\xi} \leftarrow \xi - \lambda_{r}  \nabla_{\xi}{ L_{r}(\xi) } .$ }
\State \quad \quad \quad  Imagine trajectories $\{ (a\ta, r\ta, o_{\tau+1})_{\tau=t}^{t+N} \}_{i=1}^B$ starting from $h\t$.
\State \quad \quad \quad  Compute  $L_{\mathcal{G}}^{\text{MB}}(\psi)$ from Eq.~\eqref{Eq:loss_psi_MB} and set $L_{\mathcal{G}}(\psi) = L_{\mathcal{G}}^{\text{MB}}(\psi)$. 
\State \quad \quad  \textbf{end if}
\State \quad \quad  \textbf{if} $I_{mode}==hybrid$ \textbf{then}: \algorithmiccomment{Hybrid unified actor-critic.}
\State \quad \quad  \quad   Sample a minibatch $\{( a\pk, o\k)_{k=1}^{M} \}_{i=1}^B$ from $\mathcal{D}$.
\State \quad \quad \quad Compute $L_{\mathcal{G}}^{\text{MF}}(\psi)$ from Eq.~\eqref{Eq:loss_psi_MF}. 
\State \quad \quad \quad   Imagine trajectories $\{ (a\ta, r\ta, o_{\tau+1})_{\tau=t}^{t+N} \}_{i=1}^B$ starting from $h\t$.
\State \quad \quad \quad  Compute $L_{\mathcal{G}}^{\text{MB}}(\psi)$ from Eq.~\eqref{Eq:loss_psi_MB}. 
\State \quad \quad \quad Set $L_{\mathcal{G}}(\psi) =c L_{\mathcal{G}}^{\text{MF}}(\psi)+L_{\mathcal{G}}^{\text{MB}}(\psi) $.
\State \quad \quad  \textbf{end if}
\State \quad \quad   \multiline{  $\psi \leftarrow \psi - \lambda \nabla_{\psi}{L_{\mathcal{G}}(\psi)} .$ }
\State \quad \quad Compute $L_{\mathcal{\bar{\pi}}}(\phi)$ from Eq.~\eqref{Eq:pi_loss}. 
\State \quad \quad   \multiline{  $\phi \leftarrow \phi - \lambda \nabla_{\phi}{L_{\mathcal{\bar{\pi}}}(\phi)} .$ }
\State \quad \quad  $t \leftarrow t+1$.
\State \quad  \textbf{end for}
\State \textbf{end while}
\end{algorithmic}
\end{algorithm}

\subsection{Environments} \label{sec:APP_POMDP_modification}
We utilize the PyBullet~\cite{coumans2016pybullet} benchmarks, replacing the deprecated Roboschool as recommended by the official GitHub repository\footnote{\url{https://github.com/openai/roboschool}}. To convert these PyBullet benchmarks into tasks with partial observations, we made modifications by retaining the velocities while removing all position/angle-related entries from the observation vector. Table~\ref{Table:POMDP_modification} provides a summary of key information for each environment.
\begin{table}[htb]
\caption{Information of the Roboschool PyBullet environments we used.}\label{Table:POMDP_modification}
\centering
\begin{tabular}{|P{5.cm}|P{.7cm}|P{3.7 cm}|}
\hline
\textbf{\!\!\!\ {Task} \!\!\!} & { \textbf{$D_{\mS}$}} & { \textbf{$D_{\mO}$} (velocities only)} \\
\hline
{HalfCheetahPyBulletEnv} & {26} & {9}   \\
\hline
{HopperPyBulletEnv} & {15} &  {6 } \\
\hline
{AntPyBulletEnv} & {28} &  { 9 } 
\\ \hline 
Walker2DPyBulletEnv  & {22} &  { 9 } 
\\ \hline 
\end{tabular}
\end{table}
\subsection{Implementation details} \label{sec:APP_Implementation Details}
In this section, we describe the implementation details for our algorithm and the alternative approaches.
\\
For  Recurrent Model-Free~\cite{ni2022recurrent}, we followed its original implementation, including the hyperparameters. The official implementation of VRM~\cite{han2020variational} was also used for the inference, encoding, decoding, and actor-critic networks. We ensured a fair comparison for G-Dreamer and G-SAC by adopting the same network structure as VRM, including VRNN and actor-critic networks.

 While the original Dreamer framework~\cite{Hafner2020Dream} used pixel observations and employed convolutional encoder and decoder networks, we made adaptations to suit our partially observable case. Specifically, we replaced CNNs and transposed CNNs with two-layer MLPs, each containing $256$ units. Additionally, we adjusted Dreamer's actor and critic network parameters to match those of VRM, G-Dreamer, and G-SAC, ensuring a consistent configuration for fair comparison.
\subsubsection{VRNN} 
In accordance with~\cite{han2020variational}, we set the dimensionality of the belief state representation $h$ to $256$. To model the variational posterior $q_{\nu}(\s\t|o\t, h\pt)$, we employ a one-hidden-layer fully-connected network with $128$ hidden neurons. The likelihood function $P_{\theta}(o\t|\s\t, h\pt)$, the transition function $P_{\theta}(\s\nt|h\t,a\t)$, and the reward function $P_{\xi}(r_t|h_t,a_t)$ are modelled using two-layer MLPs with $128$ neurons in each layer.  All of these networks are designed as Gaussian layers, wherein the output is represented by a multivariate normal distribution with diagonal variance. The output functions for the mean are linear, while the output functions for the variance employ a non-linear softplus activation. 
For the belief state representation update model $f_{\theta}(h\pt, a\pt, o\t,s\t)$, we utilize an LSTM architecture.
\subsubsection{Actor-Critic} 
We represent the actor $\bar{\pi}_{\phi}(a\t|h\t)$ as a diagonal multivariate Gaussian distribution. Both the actor network $\bar{\pi}_{\phi}(a\t|h\t)$ and the critic network $ \mathcal{G}_{\psi}(h\t,a\t)$ consist of $4$ fully connected layers with an intermediate hidden dimension of $256$.
The actor model generates a linear mean for the Gaussian distribution, while the standard deviation is computed using the softplus activation function. The resulting standard deviation is then transformed using the tanh function.
On the other hand, the critic model utilizes a linear output layer, which provides the estimated value of the given belief state-action pair.
\subsubsection{Model Learning}
The VRNN parameters are trained with a learning rate of $0.0008$. The training is conducted using batches of $4$ sequences, each with a length of $64$.
The critic and actor parameters are trained with a learning rate of $0.0003$ for G-SAC. Batches of $4$ sequences, each with a length of $64$, are used during training.
For G-Dreamer, the actor and critic are trained with a learning rate of $0.00008$, and batches of $50$ sequences, each with a length of $50$, are employed.
All of the model parameters are optimized using the Adam optimizer. During training, a single gradient step is taken per environment step.
The imagination horizon, $N$, is set to $15$, and the discount factor $\gamma$ is set to $0.99$ for both G-Dreamer and G-SAC. In addition, the scaling parameter $\beta$  and $\zeta$ is set to $1$ for both algorithms.
For a comprehensive summary of the hyperparameters, please refer to Tables~\ref{Table:hyperparams}.
\begin{table*}[htb]
\caption{ Hyperparameters and network setup used to implement the proposed G-SAC and G-Dreamer algorithms.  }\label{Table:hyperparams}
\centering
\begin{tabular}{l |c|c}
\hline
\textbf{\!\!\!\ Hyperparameter \!\!\!} & \textbf{G-SAC} & \textbf{G-Dreamer}\\
\hline
Discount factor $\gamma$ & $0.99$ & $0.99$ \\
Scaling factor $\beta$  and $\zeta$ & $1$ & $1$ \\
Optimizer for all the networks & Adam  & Adam\\
Learning rate for VRNN parameters $\nu$ and $\theta$ & $0.0008$ & $0.0008$\\
Learning rate for the reward parameter $\xi$ & $-$ & $0.0008$\\
Batch size for $\nu$, $\theta$, and $\xi$ & $4$  & $4$\\
Sequence length for $\nu$, $\theta$, and $\xi$ & $64$  & $64$\\
Learning rate for $\mathcal{G}_{\psi}$ and $\bar{\pi}_{\phi}$ & $0.0003$ & $0.00008$ \\
Batch size for $\psi$ and $\phi$ & $4$  & $50$\\
Sequence length for for $\psi$ and $\phi$ & $64$  & $50$\\
The imagination horizon $N$ & $-$  & $15$\\
\end{tabular}
\end{table*}
\subsection{{Comparison of computational and memory complexity}}\label{app: computation}
{Table~\ref{Table:computation} compares the computational speeds and memory usage of VRM, Recurrent Model-Free, G-SAC, and G-Dreamer when implemented over $1$ million steps in the Hopper-P environment, utilizing an Nvidia V100 GPU and 10 CPU cores for each training run. This comparison allows us to assess the computational time required for each method as follows:}
\begin{table}[!b]
\caption{{ Time and memory cost comparison in Hopper-P for $1$ million steps. }}\label{Table:computation}
\centering
\begin{tabular}{|P{6.9cm}|P{3.cm}| P{3.cm}|}
\hline
\textbf{\!\!\!\ {Method} \!\!\!}  & { \textbf{Time}} & { \textbf{Memory}}  \\
\hline
{G-Dreamer (ours)} &  {8.7 hours} &  {840 MB}   \\
\hline
{VRM~\cite{han2020variational}} &  {3 hours} &  {600 MB}  \\
\hline
{G-SAC (ours)} & {4.8 hours} &  {720 MB}  \\
\hline
{Recurrent Model-Free~\cite{ni2022recurrent}} &  {17.5 hours}  &  {1.2 GB} 
\\ \hline 
\end{tabular}
\end{table}
\\
{ VRM has the lowest computing time and memory. VRM involves perceptual learning and inference by learning a generative model and inferring belief states via feed-forward neural networks. It also learns belief state representations through an LSTM. Our proposed model-free G-SAC algorithm exhibits similar computational efficiency, with the addition of KL divergence calculation in the information gain term of its unified objective function (referenced in Eq.~\eqref{Eq:Objective}), leading to slightly increased computational and memory demands compared to VRM. Our proposed model-based G-Dreamer requires more computation and memory than both VRM and G-SAC, as it not only encompasses the tasks performed by G-SAC but also learns the reward function, further elevating its computational load.
\\
Recurrent Model-Free, on the other hand, showcases the highest computational time, primarily attributed to its RNN layer (LSTM), which processes the complete history of observations and actions as inputs. At each time step $t$, an additional $7$-dimensional input (comprising a $6$-dimensional observation $o_t$ and a scalar action $a_t$ in Hopper-P) is incorporated into the LSTM used in Recurrent Model-Free, posing challenges in RNN training.
\\
It should be noted that while VRM, G-SAC, and G-Dreamer also entail significant computational demands arising from generative model learning, belief state inference, and representation via $h_t = f(h_{t-1}, a_t, o_t, s_t)$, with $h_t$ being $256$-dimensional, their complexity in our experiments is comparatively lower than that of Recurrent Model-Free. This disparity can be attributed to two key factors:
\\
\textit{i)} VRM, G-SAC, and G-Dreamer learn the generative model and belief state via feed-forward neural networks, which are less computationally demanding than RNNs.
\\
\textit{ii)} Although belief state representation learning $h_t = f(h_{t-1}, a_t, o_t, s_t)$ in VRM, G-SAC, and G-Dreamer is conducted via an LSTM with a high-dimensional input (a $278$-dimensional combination, including a $256$-dimensional $h_{t-1}$, a scalar $a_t$, a $6$-dimensional $o_t$, and a $15$-dimensional $s_t$ in Hopper-P), the input length remains fixed over all decision-making steps into the future.  In contrast, the LSTM in Recurrent Model-Free starts with an initial $6$-dimensional observation $o_0$, with an additional $7$ dimensions added after each step. Consequently, after about $39$ time steps, the input dimension of the LSTM in Recurrent Model-Free ($6 + 39 \times 7 = 279$) surpasses the $278$-dimensional input in VRM, G-SAC, and G-Dreamer. Table~\ref{Table:computation} shows that this increase in input dimension leads to an increase in computation that exceeds that of VRM, G-SAC, and G-Dreamer. The subsequent increase in input dimensionality leads to higher computational demands, surpassing those of VRM, G-SAC, and G-Dreamer. This is because inputs with higher dimensions in a network require larger matrix multiplications per step, leading to increased computational costs.
This highlights the significant impact of increased input dimensions in RNNs on computational demand, especially relevant in our infinite horizon POMDP setting, where the agent may operate over a vast number of steps. Furthermore, more memory is required to store inputs, outputs, intermediate states, and gradients.}

{In conclusion, within an infinite horizon POMDP setting where the agent consistently interacts with its environment, the escalating memory and computation demands per step, driven by increased input dimensions, become increasingly significant. Employing belief state inference-based methods, rather than history-based approaches, not only substantially reduces the memory footprint by abstracting detailed histories into compact probabilistic representations but also improves computational efficiency, despite the initial complexities associated with learning the underlying generative model and belief state representation.}
%
\subsection{Visualizations of results in Sub-section~\ref{sub-sec:POMDP_results}} \label{APP:visu}
In this sub-section, we provide the learning curves for all the compared methods in sub-section~\ref{sub-sec:POMDP_results}. The average returns are shown in Fig.~\ref{fig:results}, with the shaded areas indicating the standard deviation.
\begin{figure}[bp]
    \centering 
  \begin{subfigure}{0.49\textwidth}
  \includegraphics[keepaspectratio=true, scale = 0.365]{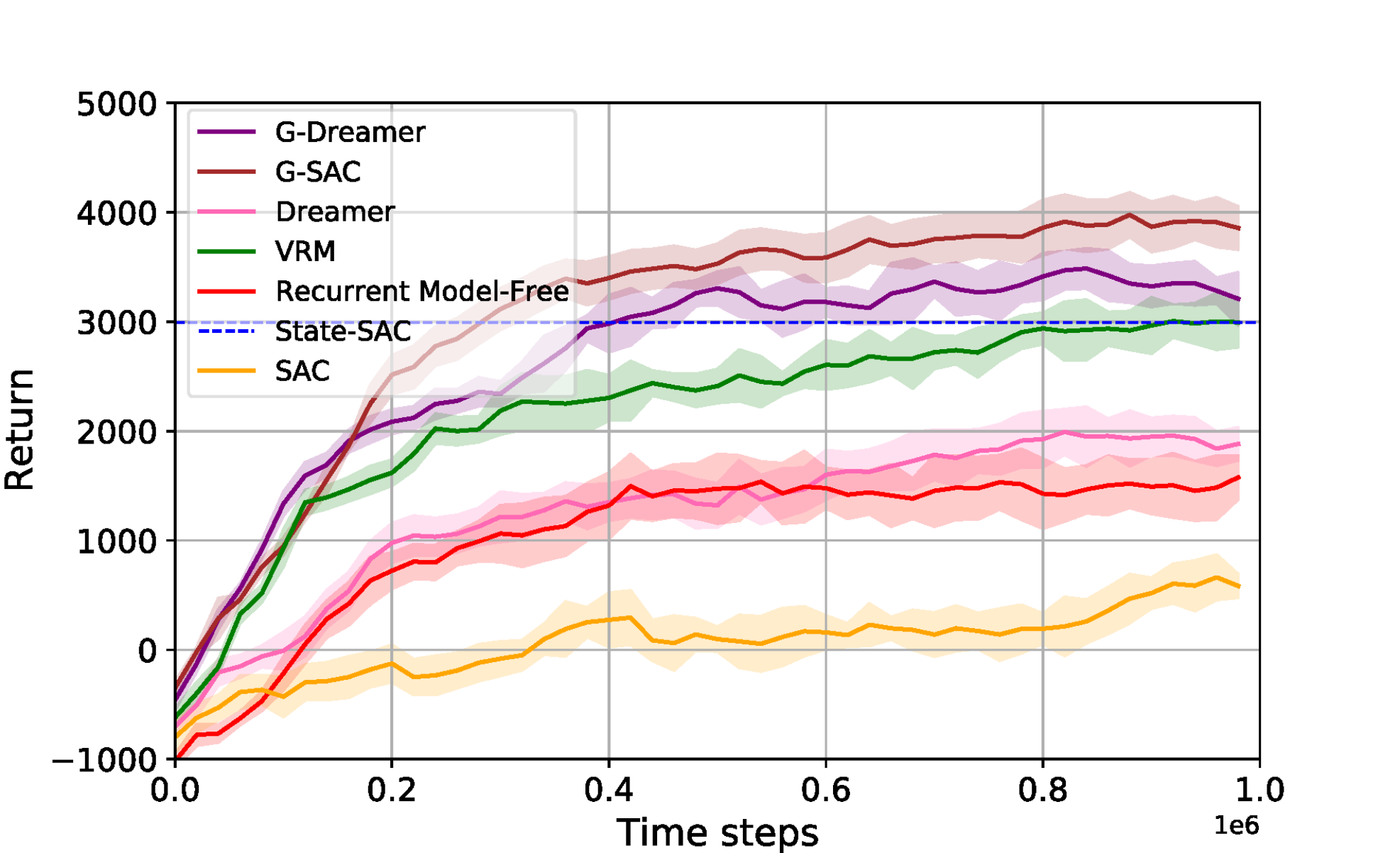}
  \caption{  HalfCheetah-P}
\end{subfigure} 
\begin{subfigure}{0.49\textwidth}
  \includegraphics[keepaspectratio=true, scale = 0.36]{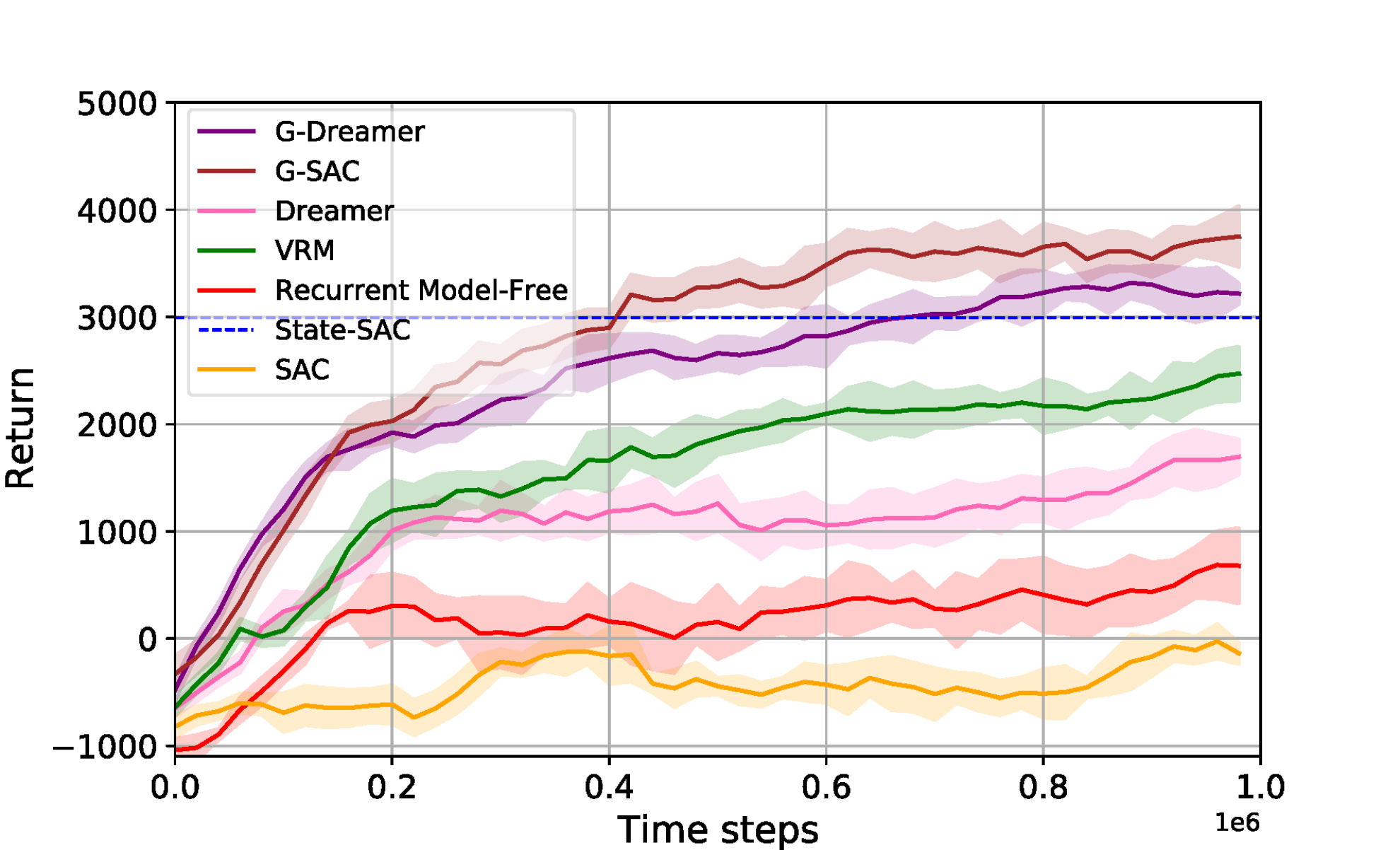}
  \caption {  HalfCheetah-N }
\end{subfigure}\hfil 
\vspace{-.05in }\\
\begin{subfigure}{0.49\textwidth}
  \includegraphics[keepaspectratio=true, scale = 0.42]{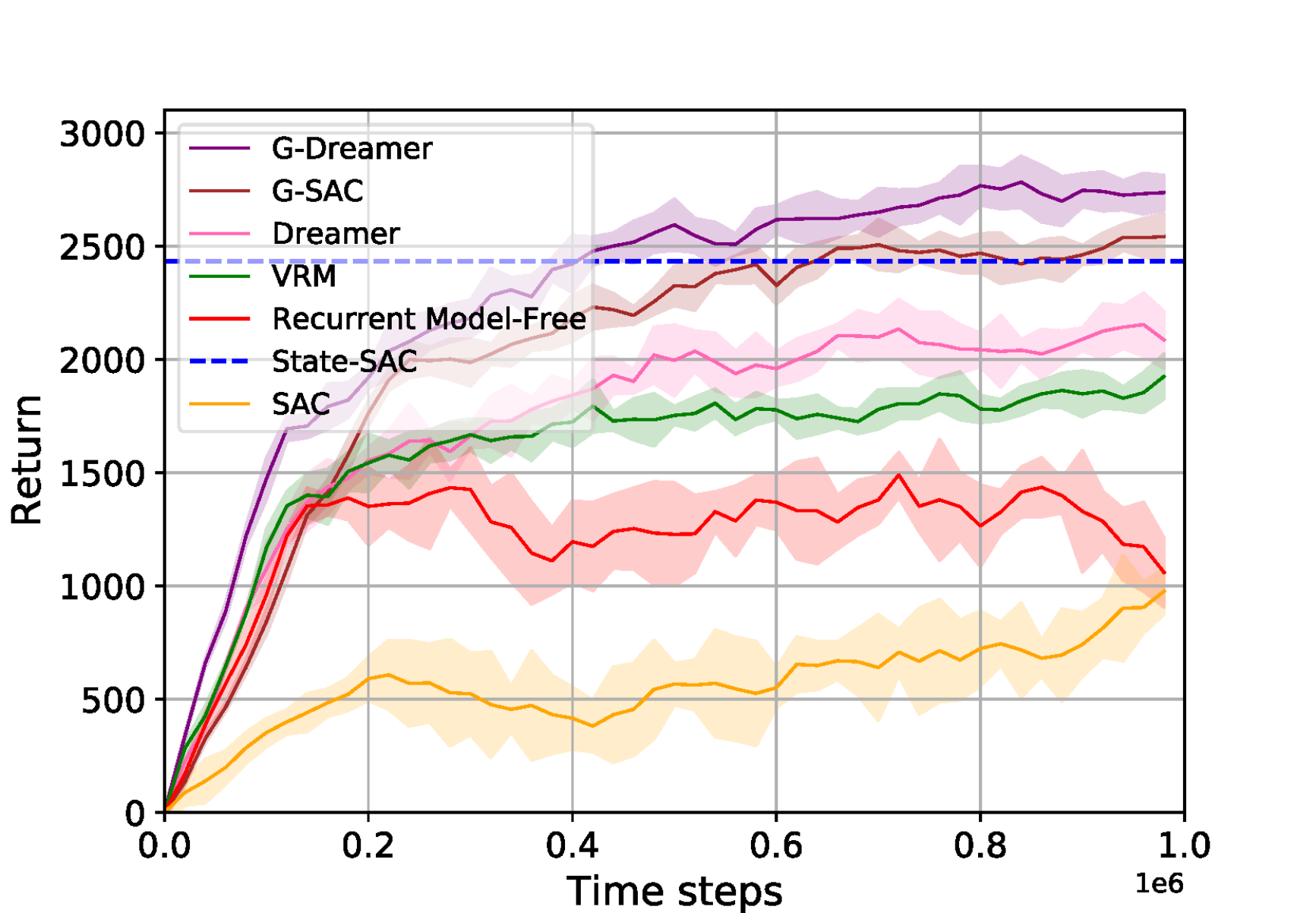}
  \caption{ Hopper-P}
\end{subfigure} 
\begin{subfigure}{0.49\textwidth}
  \includegraphics[keepaspectratio=true, scale = 0.42]{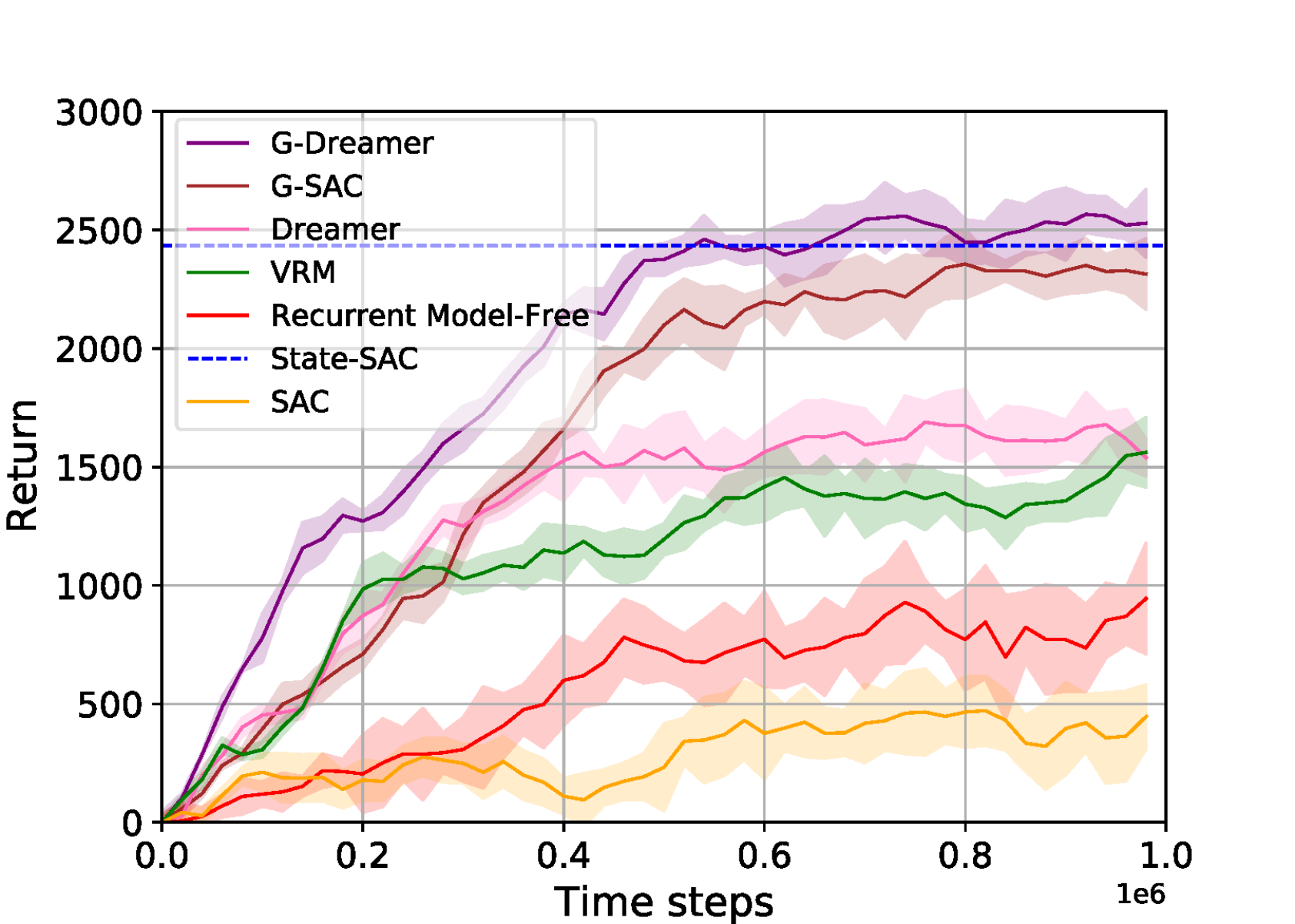}
  \caption{ Hopper-N}
\end{subfigure}\hfil 
\vspace{-.04in }\\
\begin{subfigure}{0.49\textwidth}
  \includegraphics[keepaspectratio=true, scale = 0.42]{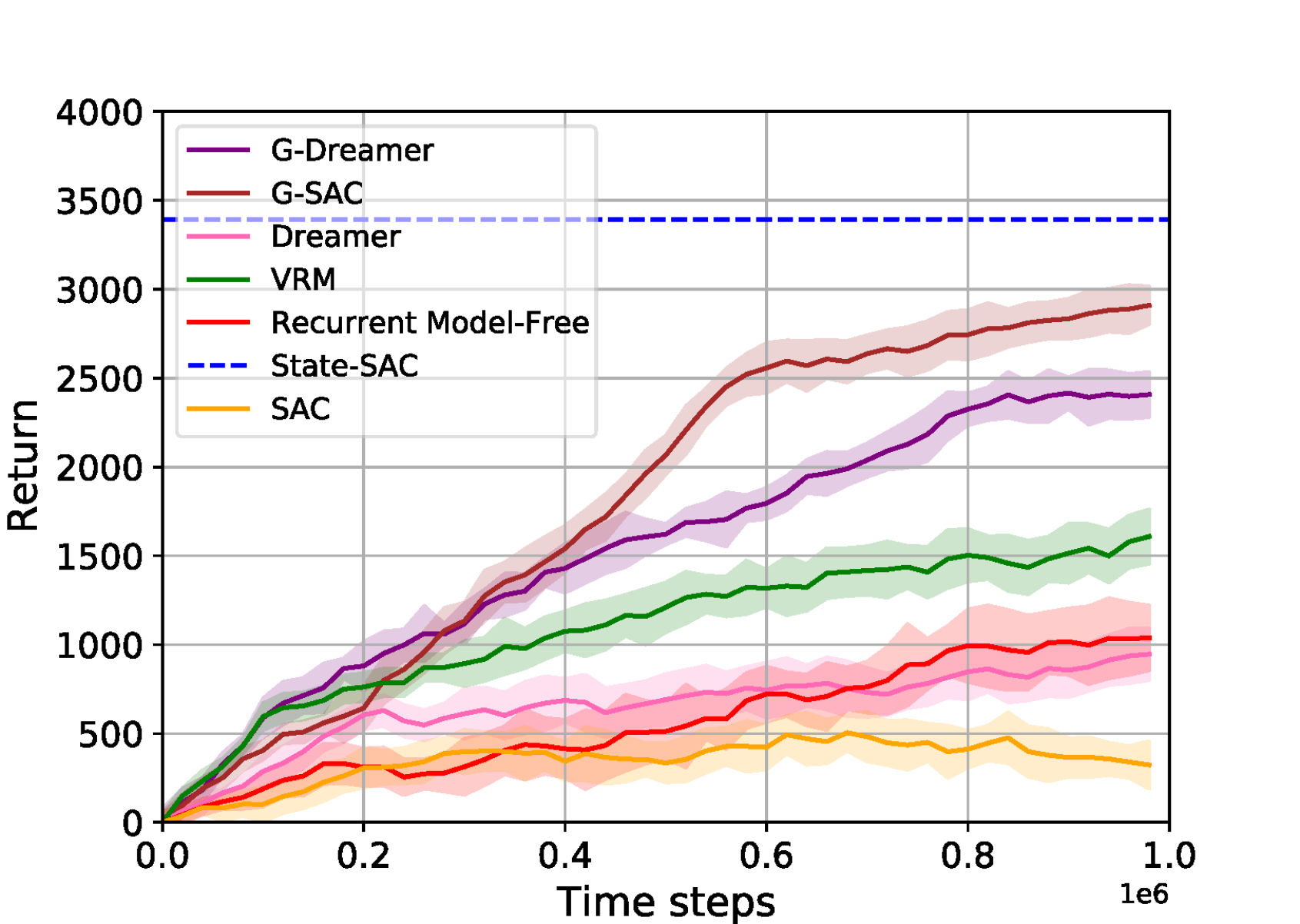}
  \caption{Ant-P}
\end{subfigure}
\begin{subfigure}{0.49\textwidth}
  \includegraphics[keepaspectratio=true, scale = 0.42]{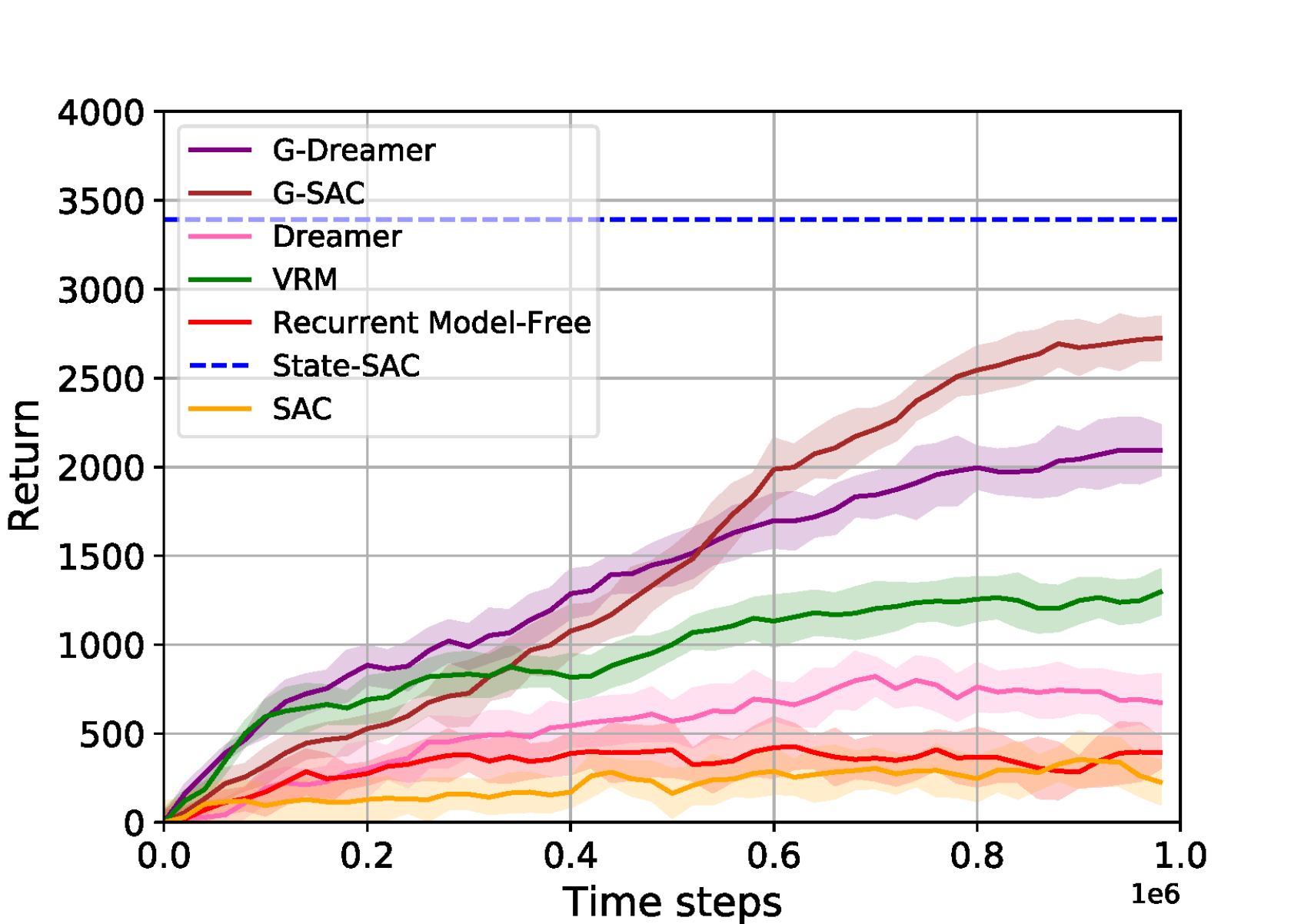}
  \caption{ Ant-N	}
\end{subfigure}
\vspace{-.07in }\\
\begin{subfigure}{0.49\textwidth}
 \includegraphics[keepaspectratio=true, scale = 0.42]{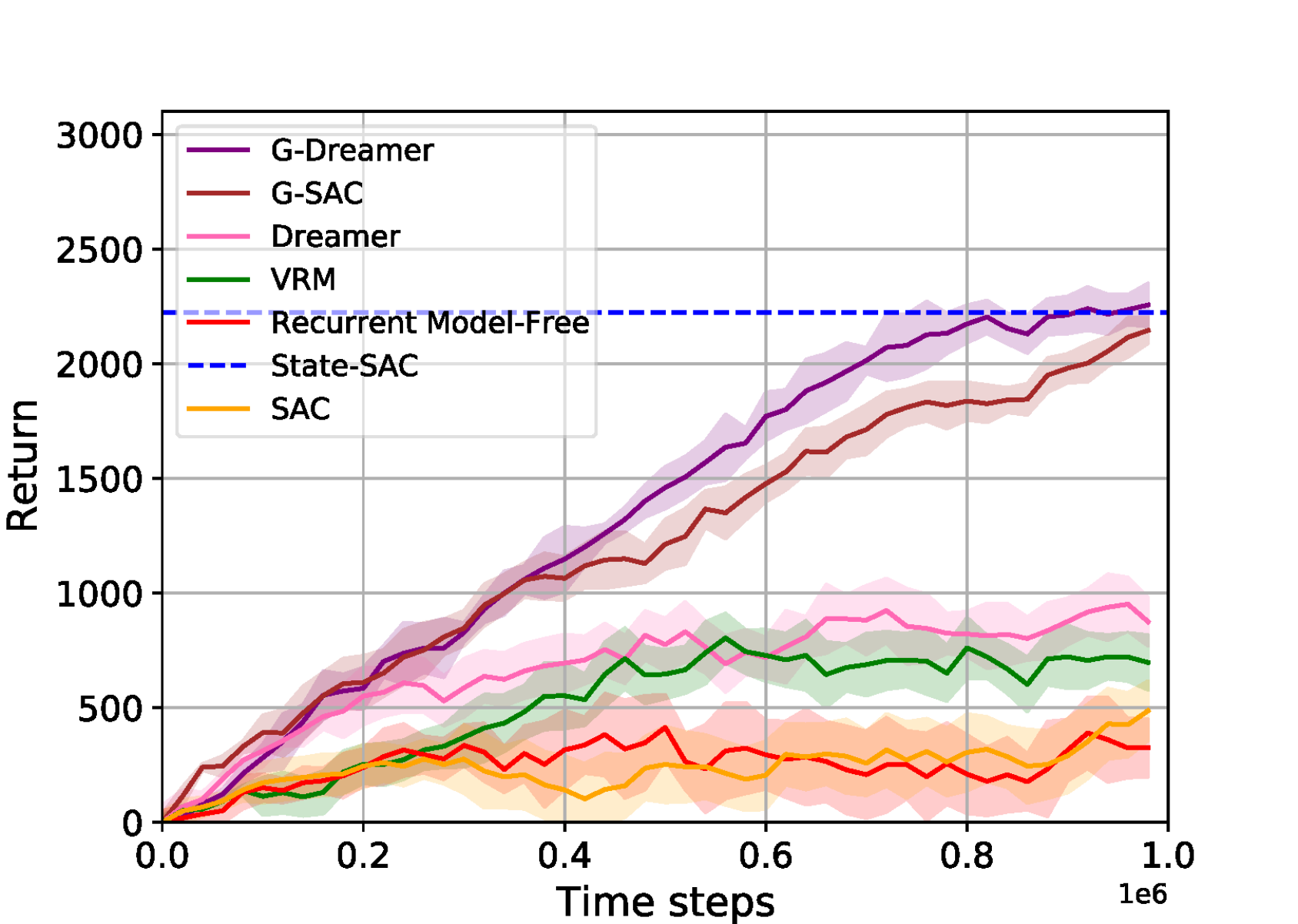}
 \caption{ Walker-2d-P}
\end{subfigure}
\begin{subfigure}{0.495\textwidth}
 \includegraphics[keepaspectratio=true, scale = 0.436]{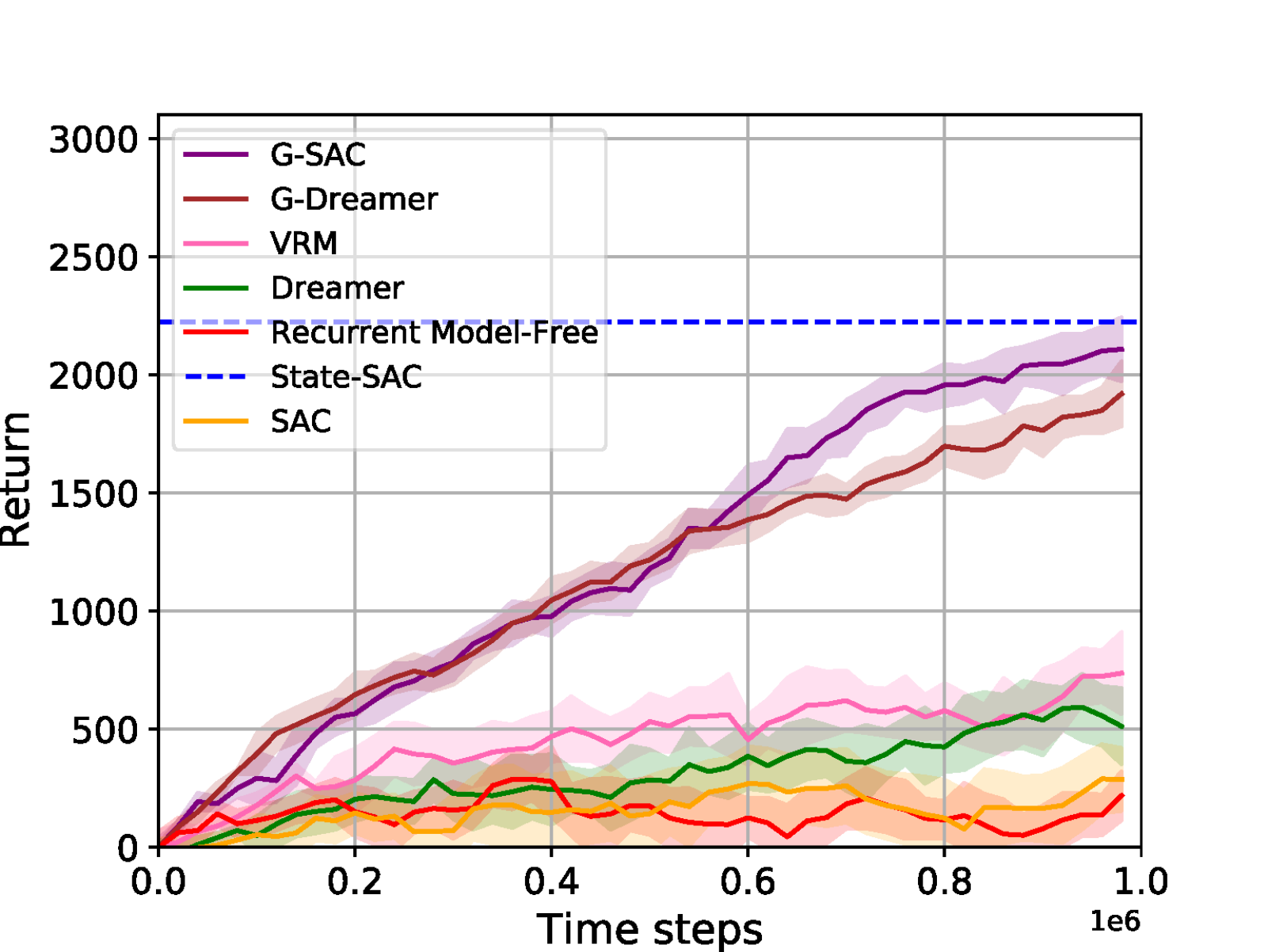}
 \caption{Walker-2d-N}
 \label{fig:6}
\end{subfigure}
\caption{Mean return for four Roboschool benchmarks with partial observations (left), and noisy observations (right). Shaded areas indicate standard deviation.}
\label{fig:results}
\end{figure}
\end{appendices}

\end{document}